\RequirePackage[log]{snapshot}
\documentclass[11pt, a4paper]{article}
\usepackage[margin=1in]{geometry}

\usepackage[utf8]{inputenc} % allow utf-8 inpu
\usepackage[T1]{fontenc}    % use 8-bit T1 fonts
\usepackage{url}            % simple URL typesetting
\usepackage{booktabs}       % professional-quality tables
\usepackage{amsfonts}       % blackboard math symbols
\usepackage{nicefrac}       % compact symbols for 1/2, etc.
\usepackage{microtype}      % microtypography
\usepackage{multicol}

\usepackage{amsmath}
\usepackage{amssymb}
\usepackage{amsthm}
\usepackage{fancybox}
\usepackage{afterpage}

\newcommand{\rot}{\intercal}
\usepackage{tikz}
\usepackage[active]{srcltx}
\usepackage{xcolor}

\usepackage{wrapfig}

\usepackage[colorinlistoftodos]{todonotes}

\newtheorem{thm}{Theorem} 

\newtheorem{lem}[thm]{Lemma}

\newtheorem{pro}[thm]{Proposition}
\newtheorem{cor}[thm]{Corollary}

\newcommand{\R}{\mathbb{R}}

\renewcommand\Im{\operatorname{Im}}

\newcommand{\thmref}[1]{Theorem~\ref{thm:#1}}

\newcommand{\lemref}[1]{Lemma~\ref{lem:#1}}

\newcommand{\figref}[1]{Figure~\ref{fig:#1}}

\newcommand{\eq}[1]{equation~\eqref{eq:#1}}

\DeclareMathOperator{\tr}{tr}
\newcommand{\eps}{\epsilon}

\usepackage[procnumbered,ruled,vlined]{algorithm2e}
\usepackage{fancybox}
\newenvironment{fminipage}%
  {\begin{Sbox}\begin{minipage}}%
  {\end{minipage}\end{Sbox}\fbox{\TheSbox}}

\DontPrintSemicolon
\SetKw{KwAnd}{and}
\SetKw{KwOr}{or}
\SetProcNameSty{textsc}
\SetFuncSty{textsc}
\SetKwInOut{Input}{Input}
\SetKwInOut{Output}{Output}

\usepackage{subcaption}
\usepackage{caption}
\usepackage[skip=-5pt, font={small}]{subcaption}
\usepackage{xcolor}

\newcommand{\noteMC}[1]{\textbf{\textcolor{cyan}{{\normalsize{[Mihai:}}#1]}}}

\newcommand{\cip}{\text{CI}}
\newcommand{\cis}{\text{CI}^{\text{size} } }
\newcommand{\civ}{\text{CI}^{\text{vol} } }

\newcommand{\cistop}{\text{Top}\cis}
\newcommand{\civtop}{\text{Top}\civ}

\usepackage[english]{babel}
\usepackage{blindtext}

\usepackage{multirow,bigdelim}
\usepackage{array}     % needed for the nice table of figures with rownames and column names 

\usepackage{xcolor}% http://ctan.org/pkg/xcolor
\usepackage{colortbl}

\usepackage[rflt]{floatflt}

\usepackage[noend]{algpseudocode}

\title{Hermitian matrices for clustering directed graphs: \\
insights and applications}

\author{%
  Mihai Cucuringu 
  \thanks{University of Oxford and The Alan Turing Institute, \texttt{mihai.cucuringu@stats.ox.ac.uk}}
  \and
  Huan Li
  \thanks{Fudan University, \texttt{huanli.me@gmail.com}}
  \and
  He Sun
  \thanks{University of Edinburgh, \texttt{h.sun@ed.ac.uk}}
  \and
  Luca Zanetti
  \thanks{University of Cambridge, \texttt{luca.zanetti@cl.cam.ac.uk}}
}

\begin{document}

\maketitle

\begin{abstract}
Graph clustering is a basic technique in machine learning, and has widespread applications in different domains. While spectral techniques have been successfully applied for clustering undirected graphs, the performance of spectral clustering algorithms for directed graphs~(digraphs) is not in general satisfactory: these algorithms usually require  symmetrising the matrix representing a digraph, and typical objective functions for undirected graph clustering   do not capture  cluster-structures in which the information given by the direction of the edges is crucial. 
To overcome these downsides,  we propose  a spectral clustering algorithm based on a complex-valued matrix representation of digraphs. We analyse its theoretical performance on a Stochastic Block Model for digraphs in which the cluster-structure is given not only by variations in edge densities, but also by the direction of the edges.
 The significance of our work is   highlighted  on a data set pertaining to internal migration in the United States: while previous spectral clustering algorithms for digraphs can only reveal that people are more likely to move between counties that are geographically close, our approach is able to cluster together counties with  a similar socio-economical profile even when they are geographically distant, and illustrates how people tend to move from rural to more urbanised areas.

\end{abstract}

\section{Introduction}

% \begin{figure*}[t]
% \captionsetup[subfigure]{skip=0pt}
% \begin{centering}
% \subcaptionbox{  \textsc{Naive} \label{fig:intronaive} }[0.65\columnwidth]{\includegraphics[width=0.66\columnwidth]{Figures/instances_mig1_ClustObj/mig1_k10_Naive__clust.jpg}}
% \hspace{0.3cm}
% \subcaptionbox{  \textsc{Our method} \label{fig:introourcl} }[0.65\columnwidth]{\includegraphics[width=0.66\columnwidth]{Figures/instances_mig1_ClustObj/mig1_k10_HermRW__clust.jpg}}
% \hspace{0.3cm}
% \subcaptionbox{  \textsc{Our method: top pair} \label{fig:introourtop} }[0.65\columnwidth]{\includegraphics[width=0.66\columnwidth]{Figures/mig1Top5/mig1_k10_HermRW__TopIFMpair1.jpg}}
% \end{centering}
% \captionsetup{width=0.99\linewidth}
% \caption{Visualisation of the clustering obtained on a US migration data set}
% \label{fig:intro}
% \end{figure*}

% 0.32 

Clustering is one of the most important techniques in analysing massive data sets, and has numerous applications ranging from machine learning to computer vision, from network analysis to social sciences. When the underlying  graph to cluster is undirected, the objective is  to partition the vertices of the graph into  clusters such that vertices within the same cluster are on average better connected to one another than vertices belonging to different clusters. This notion can be formalised by introducing an objective function to minimise, such as the conductance or the normalised cut value~\cite{hoc,ShiM00}. For example, the widely used spectral clustering algorithm~\cite{Ng,luxburg07}, which uses the top eigenvectors of the adjacency   matrix of a graph as input features for $k$-means, essentially exploits a convex relaxation of the normalised cut to obtain a good partitioning of the graph. %\todo{This is a comment that will appear in the margin}
%Understanding the quality of such partitioning has been one of the main goals in spectral graph theory~\cite{sgt,hoc,PSZ17,ST07} and has many applications in designing other combinatorial optimisation problems \textcolor{blue}{(refs)}.

However, when the underlying graph  is directed, the normalised cut value and other clustering metrics based on edge-density often fail to uncover many of the significant patterns in a graph. For instance, 
let us consider a graph  representing the number of people moving between different counties in the (mainland) United States during 1995-2000~\cite{census,census_rep}. % This graph is inherently directed. In fact, i
If one tries to symmetrise its (asymmetric) adjacency matrix $M$ in a naive way by considering the symmetric matrix $M + M^{\rot}$,  % and then apply spectral clustering, 
migration flows between counties in different states will be lost in the process.
% it becomes invisible to see the  migration blows between different counties, which is crucial in analysing such kind of graphs.
% 
Indeed, when considering the outcome of spectral clustering on   $M + M^{\rot}$ of this  migration data set as input, the visualisation in Figure~\ref{fig:intronaive} shows that clusters 
% \noteMC{do  you mean strongly correlate between their} strong correlate tween their state lines 
% \noteMC{or 
align particularly well with the political and administrative boundaries of the US states,    
% }, 
as observed in \cite{belgium2010}. This is,  somehow counterintuitively, an unsatisfactory outcome: it is quite obvious that people are more likely to move to neighbouring counties than to far away ones, and it does not provide us with much information about higher-order migration patterns across the country. 

\begin{figure}
\captionsetup[subfigure]{skip=0pt}
\begin{centering}
\subcaptionbox{  \textsc{Naive} \label{fig:intronaive} }[0.33\columnwidth]{\includegraphics[width=0.33\columnwidth]{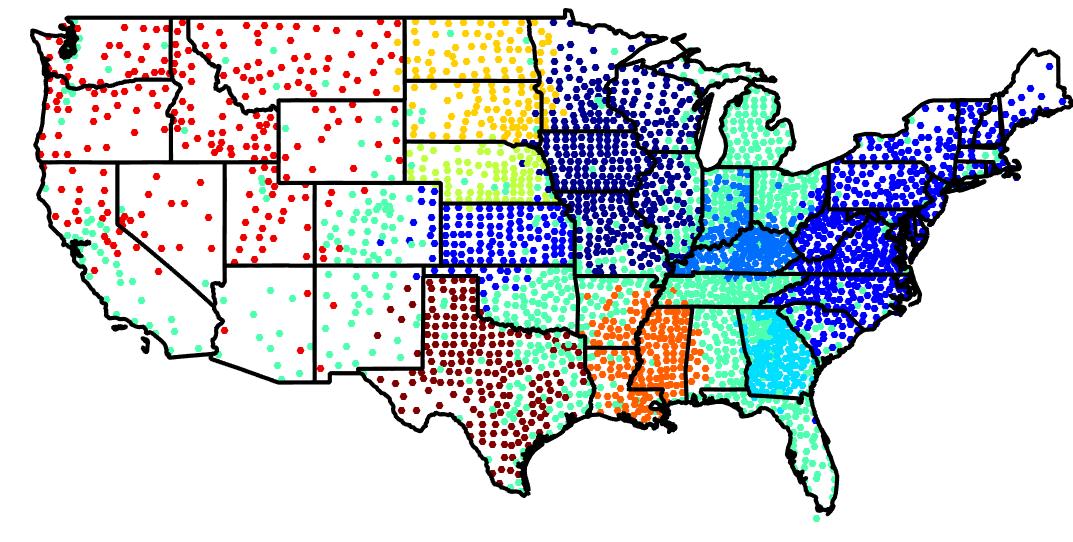}}
%\hspace{0.3cm}
\subcaptionbox{  \textsc{Our method} \label{fig:introourcl} }[0.33\columnwidth]{\includegraphics[width=0.33\columnwidth]{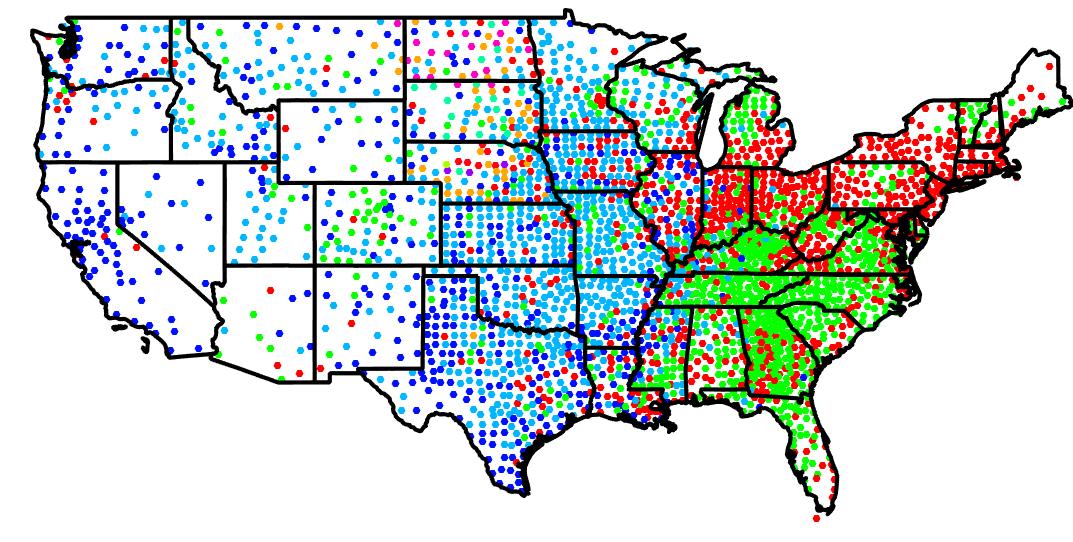}}
%\hspace{0.3cm}
\subcaptionbox{  \textsc{Our method: top pair} \label{fig:introourtop} }[0.33\columnwidth]{\includegraphics[width=0.33\columnwidth]{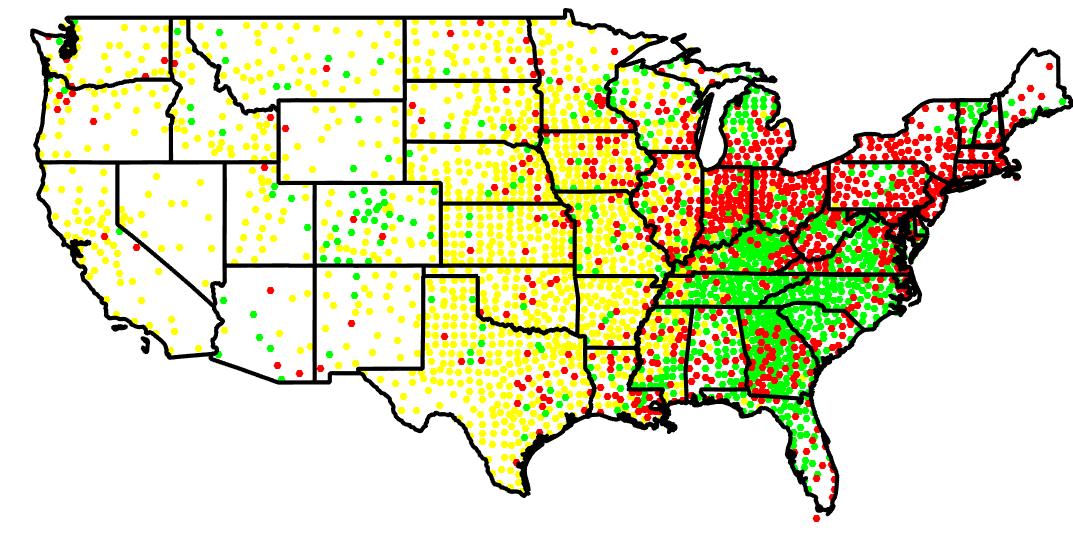}}
\end{centering}
\vspace{0mm}
\captionsetup{width=0.99\linewidth}
\caption{ \small Visualisation of the clustering obtained on a US migration data set: (a) spectral clustering on the symmetrised matrix $M+M^{\rot}$, and (b) our procedure. The red and green clusters highlighted in (c) are such that 68\% of the total weight of the edges between the two clusters is oriented from the green to the red one. %\textcolor{red}{The explanation here is a bit confusing.}
} 
\label{fig:intro}
\end{figure}

Motivated by this example, we study spectral clustering algorithms for digraphs based on a \emph{complex-valued} Hermitian adjacency matrix representations considered in \cite{GM17,SingerVDM} and defined as follows:
 for any $N$-vertex digraph $G$, the Hermitian adjacency matrix $A\in\mathbb{C}^{N\times N}$ of $G$ is the matrix where $A_{u,v}=\overline{A_{v,u}}=i$ if there is a directed edge $u\leadsto v$, and $A_{u,v}=0$ otherwise, where $i$ is the imaginary unity. Because of the use of $i$ and its conjugate $\overline{i}$  in expressing a directed edge,
all the eigenvalues of 
  $A$  are real-valued. We  show that, when the 
  edge directions  impart a cluster-structure  on $G$, this structure is approximately encoded in the eigenvectors associated with the top eigenvalues of $A$. 
    % \noteMC{New paragraph here maybe?} 
To  demonstrate the significance  of our Hermitian adjacency matrix,  Figure~\ref{fig:introourcl} visualises the outcome of spectral clustering when $A$ is used to encode the migration data set.  It is clear such clustering is much less correlated with state boundaries than the one from Figure~\ref{fig:intronaive}. Furthermore, in Figure~\ref{fig:introourcl}  we can observe  several interesting migration patterns emerging, especially  when considering pairs of clusters exhibiting a large ``imbalance'' in the  direction of the edges between them. The pair with the largest such imbalance (which we formalise in a later section) is shown in Figure \ref{fig:introourtop}, showcasing that people tend to move from counties in green towards counties in red. In particular, Figure \ref{fig:introourtop} highlights a migration pattern around the East Coast, where people tend to move from, for example, Virginia and North and South Carolina to geographically distant areas such as the New York metropolitan area, Chicago, and the East side of Florida. From this perspective, while previous algorithms identify different clusters based on the relations between vertices in a cluster and vertices outside  a cluster, our algorithm uncovers ``higher-order'' structures between  clusters. 
%a migration pattern where people tend to move from, for example, Virginia and North and South Carolina to geographically distant areas such as the New York metropolitan area, Chicago, but also Texas and California. From this perspective, while previous algorithms identify different clusters based on the relations between vertices in a cluster and vertices outside of a cluster, our algorithm uncovers the ``higher-order'' structure between the clusters.
We highlight that, in contrast to all the previous spectral algorithms for digraphs we experimented with, only our approach is able to uncover such patterns in this data set.

Our contributions and the organisation of this paper are as follows.    In Section~\ref{sec:model} we generalise the classical stochastic block model~(SBM) to the setting of digraphs, and  propose a  directed stochastic block model~(DSBM) with a latent structure defined with respect to imbalanced cuts between the clusters. In contrast to the classical SBM, the additional parameters of our model  are used to assign different probabilities to the directions of the edges across different clusters. As graphs from the DSBM possess a ground truth clustering, this model will be used to analyse the theoretical and experimental performances of our algorithm. 
In Section~\ref{sec:algo}  we present a spectral clustering algorithm for digraphs, and compare our algorithm with previous approaches.   To convince the reader of the  effectiveness of our algorithm, in Section~\ref{sec:analysis} we provide theoretical guarantees for our algorithm when applied to a broad class of DSBMs. Complementing the theoretical analysis of our proposed algorithm, in Section~\ref{sec:Experiments} we empirically demonstrate its practicality, and compare its performance against several competing approaches on synthetic and real-world data sets. Proofs of the theoretical results of Section~\ref{sec:analysis} and additional experimental results can be found in the appendix.

\vspace{-1mm}
\paragraph{Notation.}
For any unweighted and  directed graph $G$ with $N$ vertices, the Hermitian adjacency matrix of $G$ is the  matrix $A\in\mathbb{C}^{N\times N}$, where $A_{u,v} =\overline{A_{v,u}}=i$ if there is a directed edge from $u$ to $v$, expressed by $u\leadsto v$, and $A_{u,v}=0$ otherwise. When $G$ is a weighted digraph with weight $w_{u,v}$ on any edge $u \leadsto v$, we define  $A_{u,v} =(w_{u,v} - w_{v,u})i$. Notice that  $A$ is a Hermitian matrix, and therefore has $N$ real-valued eigenvalues $\{\lambda_j\}_{j=1}^N$. We order these eigenvalues 
% such that 
$|\lambda_1|\geq\ldots\geq|\lambda_N|$, and the eigenvector associated with $\lambda_j$ is denoted by $g_j\in\mathbb{C}^N$ with  $\| g_j\|=1$, for $1\leq j\leq N$. 
For any $y\in\mathbb{C}^N$, the complex conjugate of $y$ is expressed by $y^*$.   
For any Hermitian matrix $A$, the image of $A$ is denoted by $\mathrm{Im}(A)$ and the spectral norm of $A$ is denoted by $\|A\|$. We   use $\mathbf{1}_{k\times k}$ to express the $k\times k$ matrix where all the entries are $1$. For ease of discussion, we always label the clusters, as well as the rows and columns of the matrix $F\in\R^{k\times k}$ introduced later, from $0$ to $k-1$. %Finally, we write $a \gg b$ to indicate that it holds $a > C \cdot b$ for a large absolute constant $C > 0$.

%For the pattern matrix $F\in\mathbb{R}^{k\times k}$ that will be introduced for the DSBM model, we express the eigenvalues of $F$ as $\rho_1\geq\ldots\geq\rho_k$. \textcolor{red}{Actually we use $\rho_i$ to express the eigenvalues of $\wildetilde{F}$, which is difficult to define here.}

%\subsection{Preliminaries and notation}

%For any directed graph $G$ with $N$ vertices and $j$-root of unit $\omega$ for some $j$, the Hermitian adjacency matrix  of $G$ is the matrix $A\in\mathbb{C}^{N\times N}$, where   
%\begin{align}
 %   (A)_{uv} \defeq
  %  \begin{cases}
        %\deg(u) & u = v \\
 %       \omega & u\leadsto v, \\
  %      \conj{\omega}\,& v\leadsto u, \\
  %      0 & \mathrm{otherwise}.
  %  \end{cases}
%\end{align}
%The eigenvalues of $A$ is expressed by $\lambda_1(A)\geq\ldots \lambda_N(A)$ with the corresponding eigenvectors $f_1,\ldots, f_N$, where $f_i\in\mathbb{C}^N$ for any $1\leq i\leq N$.

\definecolor{burlywood}{rgb}{0.87, 0.72, 0.53}
\definecolor{cadetblue}{rgb}{0.37, 0.62, 0.63}
		\definecolor{coolblack}{rgb}{0.0, 0.18, 0.39}
	\definecolor{darkcerulean}{rgb}{0.03, 0.27, 0.49}
	
\section{Directed stochastic block model\label{sec:model}}

%\textcolor{red}{He: I would suggest to smiply the model: instead of introducing $\{n_j\}$, we assume that the size of each cluster is $n$.} \textcolor{blue}{Luca: I agree. I would be fine even removing the parameter q.}

We study graphs generated from the directed stochastic block model~(DSBM) defined by  $k, n ,p,q$, and matrix $F\in[0,1]^{k\times k}$, where $k\geq 2$ represents the number of clusters, $n$ the number of vertices in each cluster,  $p\in [0,1]$ the probability there is an edge between two vertices within the same cluster, $q\in [0,1]$ the probability there is an edge between two vertices belonging to two different clusters, while $F\in [0,1]^{k\times k}$  controls the edge orientations 
 among  clusters and satisfies  
$F_{\ell, j }+ F_{j, \ell}=1$ for any $0\leq \ell, j\leq k-1$. This  implies that $F_{\ell,\ell}=1/2$ for any $0\leq \ell\leq k-1$.
The set $\mathcal{G}\left(k,n,p,q,F\right)$ consists of graphs $G$ generated as follows: every $G\in \mathcal{G}$ is a directed graph defined on vertex set $V=\{1,\ldots, N\}$, where $N=k\cdot n$. These vertices belong to $k$ clusters $C_0,\ldots, C_{k-1}$, where $|C_j|=n$ for $0\leq j\leq k-1$. For any pair of vertices $\{u,v\}$, if they belong to the same cluster,  they are connected by an edge with probability $p$; otherwise,   they are connected with probability $q$. Moreover, if $u\in C_{\ell}$ and $v\in C_j$ are connected, the direction of this edge is determined  by $F$: the direction is set to be $u\leadsto v$ with probability $F_{\ell, j}$, and $v\leadsto u$ with probability $F_{j,\ell} = 1 - F_{\ell, j}$. By   definition,  the direction of an edge inside a cluster is chosen uniformly at random.
The matrix $F$ can be viewed as the adjacency matrix of a weighted directed graph which represents the \emph{meta-graph} describing the relations between the clusters.
The   example  below explains the roles of these parameters.

\emph{Example.} Let $k=3$,   $p=q$, and 
\[
F= \left(\begin{array}{ccc} 1/2 & 2/3 &  1/3 \\1/3 & 1/2 & 2/3 \\
2/3 & 1/3 & 1/2 \end{array}\right)
\]
 \begin{wrapfigure}[2]{r}{0.16\textwidth}
 \vspace{-3cm}
\begin{tikzpicture}[scale=0.27] 
\draw [cadetblue,  thick] (0,0) circle [radius=1.8];
\draw [red,  thick] (3,5.2) circle [radius=1.8];
\draw [darkcerulean,  thick] (6,0) circle [radius=1.8];

%% Cluster 1
%\draw [thick, ->]  (-0.6,0) -- (0.6,0);

\draw [thick, ->]  (-0.6,0.4) -- (0.6,0.4);

\draw [thick, <-]  (-0.6,-0.4) -- (0.6,-0.4);

\draw [thick, <-]  (-0.6,1.2) -- (0.6,1.2);

\draw [thick, ->]  (-0.6,-1.2) -- (0.6,-1.2);

%% Cluster 2
%\draw [thick, ->]  (-0.6,0) -- (0.6,0);

\draw [thick, <-]  (5.4,0.4) -- (6.6,0.4);

\draw [thick, <-]  (5.4,-0.4) -- (6.6,-0.4);

\draw [thick, ->]  (5.4,1.2) -- (6.6,1.2);

\draw [thick, ->]  (5.4,-1.2) -- (6.6,-1.2);

 3
%\draw [thick, ->]  (-0.6,0) -- (0.6,0);

\draw [thick, <-]  (2.4,5.6) -- (3.6,5.6);

\draw [thick, <-]  (2.4,4.8) -- (3.6,4.8);

\draw [thick, ->]  (2.4,6.4) -- (3.6,6.4);

\draw [thick, ->]  (2.4,4.2) -- (3.6,4.2);

%% edges from C1 to C2

\draw [thick, ->]  (2.1,0.2) -- (3.9,0.2);
\draw [thick, <-]  (2.1,-0.4) -- (3.9,-0.4);
\draw [thick, <-]  (2.1,-1) -- (3.9,-1);

%% edges from C2 to C3

\draw [thick,  rotate around={60:(0,0)}, ->]  (2.1,0.8) -- (3.9,0.8);
\draw [thick, rotate around={60:(0,0)}, ->]  (2.1,0.2) -- (3.9,0.2);
\draw [thick, rotate around={60:(0,0)}, <-]  (2.1,-0.4) -- (3.9,-0.4);

\begin{scope}[shift={(-3,5.2)}]
%% edges from C3 to C1

\draw [thick,  rotate around={300:(0,0)},->]  (5.1,6) -- (6.9,6);
\draw [thick, rotate around={300:(0,0)}, ->]  (5.1,5.4) -- (6.9,5.4);
\draw [thick, rotate around={300:(0,0)}, <-]  (5.1,4.8) -- (6.9,4.8); 
\end{scope}

\end{tikzpicture} 
\caption{\label{pattern}}   
\end{wrapfigure}
In this case, $G$ consists of $3$ clusters $C_0, C_1$ and $C_2$ of equal size, and any pair of vertices is connected by an edge with the same probability $p$. The directions of the edges inside a cluster are chosen uniformly at random, but directions of the edges crossing different clusters are chosen non-uniformly according to $F$.
In particular, in expectation two thirds of the edges  between  $u\in C_{j}$ and $v\in C_{j+1 \bmod 3}$ are set to be $u\leadsto v$, and the remaining one third is set to be $v\leadsto u$, as shown in Figure~\ref{pattern}. We notice that this ``cyclic flow structure'' of the edges across different clusters is particularly interesting, since in expectation all the vertices in $G$ have the same in- and out-degrees, and the cluster-structure of $G$ cannot be easily identified by the vertices' degree distribution.

% Note that 
Our model can be viewed as a  generalisation of the classical SBM~\cite{sbm} into the setting of directed graphs. As a special case of our model, when  $F_{\ell, j}=1/2$  for $0\leq \ell, j\leq k-1$, the edge directions play no role in defining a cluster-structure, and the  clusters 
are completely determined by $p$ and $q$, which is exactly the case for the SBM. On the other hand, the DSBM   captures the setting where $p=q$ and the cluster structure is determined exclusively by the directions of the edges. 
%While the cluster structure is indistinguishable by ignoring the edge directions, we will show that 
%Hermitian adjacency matrices  can be used to recover the cluster structure of graphs generated from the DSBM.
We remark that   our proposed  DSBM  is a special case of the  co-SBM~\cite{co-clustering}, which also includes bipartite structures. We think, however, that  what is lost by our model in generality is gained in clarity and simplicity.

% in the sense that, when $\Delta_{\ell, j}=1/2$ for any $0\leq j, \ell\leq k-1$, the cluster structure is essentially determined by the values of $p$ and $q$. However,  comparing with the stochastic block model where a gap between $p$ and $q$ is needed to approximately recover the clusters~\textcolor{red}{[add reference]},  we will prove that  the clusters of a graph generated from $\mathcal{G}$  can be approximated recovered as long as the clusters are well-defined based on the directions of the edges from $\Delta$, even if $p=q$.

%We remark that, comparing with the classical stochastic block model~\textcolor{red}{[add reference]}, we did not introduce the parameter $q$ defining the probability of connecting vertices from different clusters for the case of analysis. Notice that, while the studies for stochastic block model focuses on the possibility of approximately/perfectly recovering the clusters based on the relations of $p$ versus $q$, our  primary focus is to study the cluster structure of the graphs with the same edge densities, and approximation algorithms for finding these clusters based on the directed edges' information. We'll leave the work on the impact of such parameter $q$ for our directed stochastic block model for further studies. 

\section{Algorithm \label{sec:algo}}

Now we describe a spectral clustering algorithm 
for graphs generated from the DSBM.  
Given a graph $G= (V,E)$ generated from the DSBM  $\mathcal{G}\left(k,n,p,q,F\right)$, our algorithm first computes the eigenvectors $g_1,\dots,g_{\ell}$ corresponding to the eigenvalues $\lambda_j$ satisfying $|\lambda_j |\geq \eps$ for some parameter $\eps$.
Secondly, the algorithm  constructs a matrix $P$ which is the projection matrix on the subspace spanned by $g_1,\dots,g_{\ell}$, and applies  $k$-means  with the rows of $P$ as input features.\footnote{We remark that using the $nk$-dimensional embedding given by the rows of $P$ is analogous to using the $\ell$-dimensional embedding given by the rows of $U$, where $U$ is the  eigendecomposition of $P=U U^{\rot}$.}
Finally, the algorithm partitions the vertex set of $G$ based on the output of  $k$-means. See Algorithm~\ref{alg:spectral}.% for a formal description. 

      \begin{algorithm}
       \caption{Spectral clustering for digraphs\label{alg:spectral}}
       \begin{algorithmic}[1]
         \Require   directed graph $G = (V,E)$ with Hermitian adjacency matrix $A$; $k \ge 2$; $\eps > 0$
         
         \State    Compute  the eigenpairs $\{(\lambda_i,g_i) \}_{i=1}^{\ell}$
				 of $A$ with $|\lambda_i| > \eps$. 
				 
		 \State   $P \gets \sum_{j=1}^{\ell} g_jg_j^*$
		 
		 \State  Apply a $k$-means algorithm with input the rows of $P$.
		 
		 \State Return a partition of $V$ based on the output of $k$-means.
          \end{algorithmic}
      \end{algorithm}

% \begin{wrapfigure}{L}{0.7\textwidth}
% \vspace{-0.2cm}
%     \begin{minipage}{0.7\textwidth}
%       \begin{algorithm}[H]
%       \caption{Spectral clustering for digraphs\label{alg:spectral}}
%       \begin{algorithmic}[1]
%          \Require   directed graph $G = (V,E)$ with Hermitian adjacency matrix $A$; $k \ge 2$; $\eps > 0$
         
%          \State    Compute  the eigenpairs $\{(\lambda_i,g_i) \}_{i=1}^{\ell}$
% 				 of $A$ with $|\lambda_i| > \eps$. 
				 
% 		 \State   $P \gets \sum_{j=1}^{\ell} g_jg_j^*$
		 
% 		 \State  Apply a $k$-means algorithm with input the rows of $P$.
		 
% 		 \State Return a partition of $V$ based on the output of $k$-means.
%           \end{algorithmic}
%       \end{algorithm}
%     \end{minipage}
%   \end{wrapfigure}
We remark  that the number $\ell$ of eigenvectors used by the algorithm   depends on the parameters of the model, and in particular on the rank of  $F$ which defines the direction of the edges among different clusters.  In general, $\ell \le k$,
but for practical purposes one can simply set  $\ell = k$.\footnote{More precisely, we recommend setting $\ell = k-1$ when $k$ is odd, since in this case $F$ is always rank-deficient.} However, to obtain the optimal theoretical guarantees, at least for the case of $p=q$, we  set $\eps = 10 \sqrt{pn \log(pn)}$, whose value can be easily estimated with high probability since the average degree in the graph concentrates around $pkn$ when $p\gg 1/n$. As it will become clear from our following analysis, in this way $\ell$ is   set as the rank of $F$, without the need to actually know $F$. We also notice that including all the eigenvectors corresponding to the same eigenvalue in absolute value ensures that $P$ is a \emph{real} matrix. This follows from the fact that $A$ is not only Hermitian, but also skew-symmetric.

\vspace{-1mm}
\paragraph{Comparison with other spectral methods.}
We compare our algorithm with other spectral methods for digraph clustering that are based on  the classical real-valued adjacency matrix $M$ of an unweighted digraph $G=(V,E)$, defined as follows:
for any pair of vertices $u,v$, $M_{u,v}=1$ if $u \leadsto v$ and $M_{u,v}=0$ otherwise. While 
Algorithm~\ref{alg:spectral}
  exploits the top eigenvectors of the Hermitian adjacency matrix $A = (M - M^{\rot}) \cdot i$, previous spectral clustering algorithms for directed graphs~\cite{MV13,co-clustering,SP11} typically use eigenvectors of $M^{\rot} M$, $M M^{\rot}$, or $M^{\rot} M + M M^{\rot}$ (or a  regularised version of these matrices). To compare our algorithm with previous ones,  notice that  for any $u,v\in V$ these matrices' corresponding entries can be written as  
\begin{align}
    (M^{\rot} M)_{uv} &= |\{ w \colon w \leadsto u \text{ and } w \leadsto v \}|, \label{eq:mleft} \\
    (M M^{\rot})_{uv} &= |\{ w \colon u \leadsto w \text{ and } v \leadsto w \}|, \label{eq:mright} \\
    (M^{\rot} M + M M^{\rot})_{uv} &=|\{ w \colon w \leadsto u \text{ and } w \leadsto v \}|  
         + |\{ w \colon u \leadsto w \text{ and } v \leadsto w \}|. \label{eq:mlr}
\end{align}
By definition, $M^{\rot} M$ keeps track of the common ``parents'' between two vertices, $M M^{\rot}$ of the common ``offspring'', while their sum of both. To draw a direct comparison, we study the matrix $A^2$, since 
$A$ and $A^2$  share the same eigenvectors and  $A^2$ is easier to analyse.
By definition, we have that 
\begin{align*}
    A^2_{uv} &= |\{ w \colon (w \leadsto u  \text{ and } w \leadsto v) \text{ or } (u \leadsto w \text{ and } v \leadsto w) \}| \\
        &\qquad\qquad - |\{ w \colon (u \leadsto w \text{ and } w \leadsto v) \text{ or } (w \leadsto u \text{ and } v \leadsto w) \}|,
\end{align*}
which implies that  $A$ keeps track of both common parents and offspring of two vertices $u,v$, while assigning a penalty for every node $w$ that is simultaneously a parent of $u$ and an offspring of $v$, or vice versa. Hence, $A$ implicitly assigns a positive weight between a pair of vertices who have more common parents and offspring than ``mismatched'' relations with  a third vertex, and a negative weight otherwise. This peculiar behaviour is  at the heart of the better performances of our algorithm on some real-world data sets compared to the state-of-the-art.
Moreover, it is worth mentioning that $A$ can implicitly keep track of both common parents and offspring without the need to perform expensive matrix multiplications as in the case of   $M^{\rot} M + M M^{\rot}$.

\vspace{-1mm}
\paragraph{Normalisation of $A$.}
When dealing with real-world data sets, a proper normalisation of the graph adjacency matrix is usually required. 
% Hence, we introduce 
For a diagonal matrix $D$, with $D_{jj} = \sum_{\ell=1}^{N} | A_{j\ell} |$, we define  
\begin{equation}
    A_{\mathrm{rw}} = D^{-1} A,
\label{def:Lrw}
\end{equation} 
which is similar to the Hermitian matrix 
$
    A_{\mathrm{sym}} = D^{-1/2} A  D^{-1/2}
$ and has   $N$ real eigenvalues. The operator \eqref{def:Lrw} was studied  in the context of  angular synchronisation and the graph realisation problem~\cite{asap2d}, and  in  
\cite{SingerVDM}, which introduced Vector Diffusion Maps for nonlinear dimensionality reduction. %and explored the interplay with the connection-Laplacian operator for vector fields over manifolds.  
We also notice that these Hermitian operators have 
  been successfully used in the ranking literature. %, where the net outcome of pairwise matches  between players can be encoded in a digraph with a skew-symmetric adjacency matrix. 
  In particular, \cite{syncRank} formulated the ranking problem as an instance of the group synchronisation problem, considered an angular embedding of $M-M^{\rot}$ and relied on the top eigenvector of $A_{\mathrm{rw}}$ to recover a   one-dimensional
  ordering of the players. 

%Recently, \cite{FANUEL2017JACHA} proposed  a certain deformation of the combinatorial Laplacian,  % in particular, the \textit{Dilation Laplacian}, which is shown to perform well for ranking in directed networks of pairwise comparisons. 

% \textcolor{red}{I'm not sure if it's very related. I sugest to drop the last paragraph completely}

\section{Analysis}
\label{sec:analysis}
We now analyse the performance of
Algorithm~\ref{alg:spectral} on the DSBM. Let $G \sim \mathcal{G}\left(k,n,p,q,F\right)$ with Hermitian adjacency matrix $A$. For simplicity, we assume that  $p=q$. We remark that this condition does not simplify the problem, since in this case edge densities do not give us any information on the cluster-structure of the graph, which is entirely determined by the edge orientations.
We first study the expected adjacency matrix $\mathbb{E}A$. For any $u\in C_{j}$ and $v\in C_{\ell}$, we have  that $(\mathbb{E}A)_{u,v} = p\left( F_{j,\ell} - F_{\ell,j} \right)\cdot i = p \left( 2 F_{j,\ell}-1 \right)\cdot i$.
This implies that $\mathbb{E}A$ is a Hermitian matrix and can be decomposed into $k\times k$ blocks. Moreover, the rank of $\mathbb{E}A$ is at most $k$. To analyse the spectral property of $\mathbb{E}A$, we define the matrix 
$ \widetilde{F}= \left(2F - \mathbf{1}_{k\times k} \right)\cdot i.$
Observe that, if $\widetilde{\lambda}\in\mathbb{R}$ is an eigenvalue of $\widetilde{F}$ with the corresponding eigenvector $\widetilde{f}\in\mathbb{C}^k$, then $\widetilde{\lambda} pn$ is an eigenvalue of $\mathbb{E}A$ with eigenvector $f\in\mathbb{C}^{kn}$ where $f(u)=\widetilde{f}(j)$ for any $u\in C_{j}$. %\textcolor{red}{To Luca: I changed $\lambda ipn$ to $\lambda pn$} \textcolor{red}{To Luca: The definition of $\widetilde{F}$ is different from the one below from proposition 2. They differ by a factor of $i$.} \textcolor{blue}{We can reconcile them: Prop. 2 still holds if you multiply $\widetilde{F}$ by $\i$.}

%We first study the expected adjacency matrix $\E{A}$. Let $u \in C_i, v \in C_j$. Then, $\E{A}_{u,v} = p(\Delta_{i,j} - \Delta_{j,i})\i = p(2\Delta_{i,j} - 1)\i$. $\E{A}$ is clearly Hermitian  and can be decomposed into $k \times k$ blocks. Therefore, it has rank at most $k$. To understand the spectral property of $\E{A}$ we define a new $k \times k$ matrix $\widetilde{\Delta} = \i\left(2\Delta - \mathbf{1}_{k \times k}\right)$. We can see that if $\lambda \in \mathbb{R}$ is an eigenvalue of $\Delta$ with corresponding eigenvector $f \in \mathbb{C}^k$, then $\lambda \i p n$ is an eigenvalue for $\E{A}$ with corresponding eigenvector $\widetilde{f} \in  \mathbb{C}^{kn}$ such that $\widetilde{f}(u) = f(i)$ for any $u \in C_i$.

Now we explain why Algorithm~\ref{alg:spectral} works for graphs generated from the DSBM. Note that, if $A$ is close to  $\mathbb{E}A$, which is the case for most instances, then the projection on the top eigenspaces of $A$ will be close to   $P_{\Im(\widetilde{F})}\otimes \mathbf{1}_{n\times n}$, where $P_{\Im(\widetilde{F})}$ is the projection on $\mathrm{Im}(\widetilde{F})$. %Because of this and the fact that vectors in $\mathrm{Im}(\widetilde{F}\otimes \mathbf{1}_{n\times 1})$ always assign the same value to vertices within the same cluster, 
  Therefore, it suffices to ensure that $P_{\Im(\widetilde{F})}$ is actually able to distinguish different clusters. 
  %Note that this is the case when $\widetilde{F}$ is nonsingular. However, being a skew-symmetric matrix, $\widetilde{F}$ is singular whenever $k$ is odd.
Because of this,
%\textcolor{orange}{$\theta$-distinguishing is to measure how close we are to the case where two clusters are indistingishable right? It's not immediate to me its relation to singualrity} 
we introduce the notion of $\theta$-distinguishing image to ensure that the rows of   $P_{\Im(\widetilde{F})}$ are not similar to each other. Formally, for any $\theta\in[0,1]$, we say that  $\widetilde{F}$ has a \emph{$\theta$-distinguishing image},
if it holds for any $0 \le j \neq \ell \le k-1$ that $\big\|P_{\Im(\widetilde{F})}(j,\cdot) - P_{\Im(\widetilde{F})}(\ell,\cdot)\big\| \ge \theta$. Moreover, we say that $\widetilde{F}$ has a \emph{nondistinguishing} image if the previous equation holds only for $\theta = 0$. Proposition~\ref{prop2} below   shows that $ \widetilde{F} $ has a nondistinguishing image if and only if $F$ has two identical rows. When $p=q$, this condition implies every graph generated from the DSBM  has two statistically indistinguishable clusters.

%To clarify the importance of this definition, consider the following lemma which shows that the image of $\widetilde{\Delta}$ is $0$-distinguishing if and only if $\Delta$ has two identical rows. If $p=q$, this means the model has two statistically undistinguishable communities.
 
\begin{pro}\label{prop2}
Let $G \sim \mathcal{G}\left(k,n,p,q,F\right)$. Then, the matrix $\widetilde{F}$ defined by $\widetilde{F} = \left( 2F - \mathbf{1}_{k \times k}\right)\cdot i$ has a nondistinguishing image if and only if there exist $0 \le j \ne \ell \le k-1$ such that  $F(j,\cdot) = F(\ell,\cdot)$.
\end{pro}

%\textcolor{orange}{Should $\theta$ be the largest value satisfying that conditon? Otherwise every $\widetilde{F}$ has a 0-distinguishing image}

Our   analysis  is based on matrix perturbation theory, and requires that the nonzero eigenvalues of $\widetilde{F}$ are far from $0$ in order to ensure that projection on the the top eigenspaces of $A$ is close to $P_{\mathrm{Im}(\widetilde{F})} \otimes \mathbf{1}_{n\times n}$. Hence, we define   the \emph{spectral gap} of   $\widetilde{F}$  by 
$ \widetilde{\rho} \triangleq \min_{1\le j \le k} \{ |\rho_j| \colon \rho_j \neq 0 \}$,
where $\rho_1,\ldots, \rho_k$ are the eigenvalues of  $\widetilde{F}$.  
Note that in the standard SBM a similar definition of spectral gap governs the performance of spectral clustering algorithms~(see, e.g., \cite[Corollary 3.2]{leiRinaldo}).  Theorem~\ref{thm:main} bounds the number of misclassified vertices by  Algorithm~\ref{alg:spectral} for graphs generated from the DSBM.

%\textcolor{red}{From the NIPS submission checklist, we need to include time and space analysis of the presented algorithm. We should restate Theorem 2.}\noteLZ{I don't think we should add this to the main theorem. After all it depends on the k-means algorithm used. If you want we can add a line explaining that we assume we are using a polynomial time algorithm for k-means with constant-approximation factor. BTW, we don't need to tick all the boxes in the reproducibility checklist. }
 
\begin{thm}[Main Theorem]
\label{thm:main}
Let $G \sim \mathcal{G}\left(k,n,p,q, F \right)$, where $p=q$. Assume that
\begin{equation}\label{eq:assumption}
\widetilde{\rho} \ge C \ (k/  \theta) \sqrt{ (1/pn)\ \log{n}}
\end{equation}
  holds for a large absolute constant $C$ and  $\widetilde{F}$ has a $\theta$-distinguishing image with $\theta>0$. Then,  with high probability, the number of misclassified vertices by Algorithm~\ref{alg:spectral} is $O\left( {k^2}/(\tilde{\rho}^{2}\ \theta^{2}\ p) \ \log{n}\right)$.
\end{thm}

% Note that 
For a family of graphs with $k$ fixed and $n$  growing, as long as $p$ is not too small, assumption~(\ref{eq:assumption}) is always met. 
% satisfied. 
It also implies that, for most cluster-structure matrices $F$, $p$ needs to be greater than $k^2 \log{n}/n$, which is comparable to the connectivity threshold $p\geq\log(kn)/(kn)$.  % that is 
%required  to ensure $G$ is connected.

Next we  evaluate the theoretical guarantee by 
\thmref{main} when $G \sim \mathcal{G}\left(k,n,p,q,F\right)$,    $p=q$, and there exists a noise parameter $\eta \in [0,1/2)$ such that  $F_{j,\ell} = 1-\eta$ if $j \equiv \ell - 1 \mod k$, $F_{j,\ell} = \eta$ if $j \equiv \ell + 1 \mod k$, and $F_{j,\ell} = 1/2$ otherwise. By definition,   the connections among the $k$  clusters can be represented by a directed cycle where each edge has weight $1-2\eta$, and hence we call this particular DSBM \emph{the cyclic block model}. We believe this cyclic block model is particularly suitable to evaluate the performance of a clustering algorithm for digraphs due to the following reasons: (1) since every vertex of the graph has the same in-degree and out-degree in expectation, the vertices' degrees provide no information for clustering; (2) even for the case of $\eta=1$, i.e., all the  edges between two clusters $C_j$ and $C_{j+1 \mod k}$ are oriented in the same direction, the clustering task could be still  very challenging because the directions of most edges are  randomly oriented. We summarise the performance of Algorithm~\ref{alg:spectral} on the cyclic block model as follows.

\begin{cor}
\label{cor:cyclic}
Let $G$ be a graph sampled from a cyclic block model with parameters $k,n,p=q=\omega\left(k^3 / ((1-2\eta)^{2}\ n)\ \log{n}\right)$, and  $\eta \in [0,1/2)$.  Then, with high probability, the number of misclassified vertices by Algorithm~\ref{alg:spectral} is   $O\left(k^4/ ((1-2\eta)^{2} p) \  \log{n} \right)$.
\end{cor}

\section{Experiments}
\label{sec:Experiments}
\vspace{-1mm}

We compare the performance of our algorithm with other spectral clustering algorithms for digraphs on  synthetic and real-world data sets. Since ground truth clustering is available for graphs generated from the DSBM, we measure the recovery accuracy by the Adjusted Rand Index (ARI)~\cite{ARI_JMLR_Gates_Ahn}, which is closely related to and alleviates some of the issues of the popular  Rand Index~\cite{rand1971}. Both measures indicate how well a recovered clustering matches the ground truth, with a value close to 1~(resp. 0) indicating an almost perfect recovery~(resp. an almost random assignment of the vertices into clusters). For real-world data sets, due to the lack of a ground truth clustering, we will introduce   appropriately defined new objective functions to measure the quality of a clustering, while taking the edge directions into account and aiming to uncover imbalanced cuts in the partition. 
% A few new objective functions will be introduced and their significance will be discussed when we present the experimental results on the real-world data sets. 

% \emph{a Linux machine which has a 4.2GHz Intel i7-7700 CPU and 32G memory using Matlab R2018b.}
 
% \vspace{-2mm} 
\textbf{Experimental setup.} 
We compare against the three variants of the \textsc{DI-SIM} algorithm~\cite{co-clustering}, and spectral clustering  for digraphs when bibliometric and degree-discounted symmetrisations are applied~\cite{SP11}. Note that all these algorithms follow the standard framework of spectral clustering, but employ different eigenvectors to construct the feature vectors for $k$-means++. More specifically, \textsc{DI-SIM (left)}~(denoted by \textsc{DISG-L}) and \textsc{DI-SIM (right)} (\textsc{DISG-R})  use, respectively, the top $k$ eigenvectors of a regularised and normalised version of the matrix defined in  (\ref{eq:mleft}) and (\ref{eq:mright}) as input features for $k$-means; \textsc{DI-SIM (left+right)} (\textsc{DISG-LR}) uses the top $k$ eigenvectors of a regularised and normalised version of both matrices (\ref{eq:mleft}) and (\ref{eq:mright});  \textsc{Bi-Sym} and \textsc{DD-Sym} use the top $k$ eigenvectors of the matrix  in (\ref{eq:mlr}), with  an additional normalisation for \textsc{DD-Sym}. 

We also consider an additional variant of our Algorithm~\ref{alg:spectral} based on a different normalisation of our Hermitian adjacency matrix. Specifically, we use  \textsc{Herm} and \textsc{Herm-RW} to represent  Algorithm~\ref{alg:spectral} when the top eigenvectors of $A$  and  $A_{\mathrm{rw}}$ defined in \eqref{def:Lrw}  are applied as the  input matrix, respectively. We remark that Algorithm~\ref{alg:spectral} is described with respect to the non-normalised Hermitian adjacency matrix, since all the vertices of a graph generated from the DSBM have the same expected degree and normalising   $A$ with respect to degrees is not needed. On the other hand, in real-world data sets, the degree distribution is typically very skewed with large outlier degrees and,
as our experiments suggest,  \textsc{Herm-RW} usually performs the best among the  tested algorithms.

\textbf{Experimental results for the DSBM.}
We perform experiments on graphs randomly generated from the DSBM with   different values of $n,p=q$, and matrix $F$. Since spectral techniques perform better in the SBM for large $p$, our focus is to compare the performance of different algorithms when $p$  is close to the connectivity threshold $\log(N)/N$ of a random $\mathcal{G}(N,p)$ graph. 
Our reported results  are averaged over 10 independently generated graphs for every fixed parameter set. For  % the purpose of better 
ease of visualisation, we assume that the entries of   $F$  have only three different values: $1/2$ (which corresponds to uniformly random edge-directions), $\eta$, and $1-\eta$, and the experimental results are reported with respect to $\eta$.

\vspace{-2mm} 
\begin{figure*}[!htp]
\captionsetup[subfigure]{skip=0pt}
\begin{centering}
\hspace{-2mm}
\subcaptionbox{$p=0.45\%$}[0.245\columnwidth]{\includegraphics[width=0.270\columnwidth]{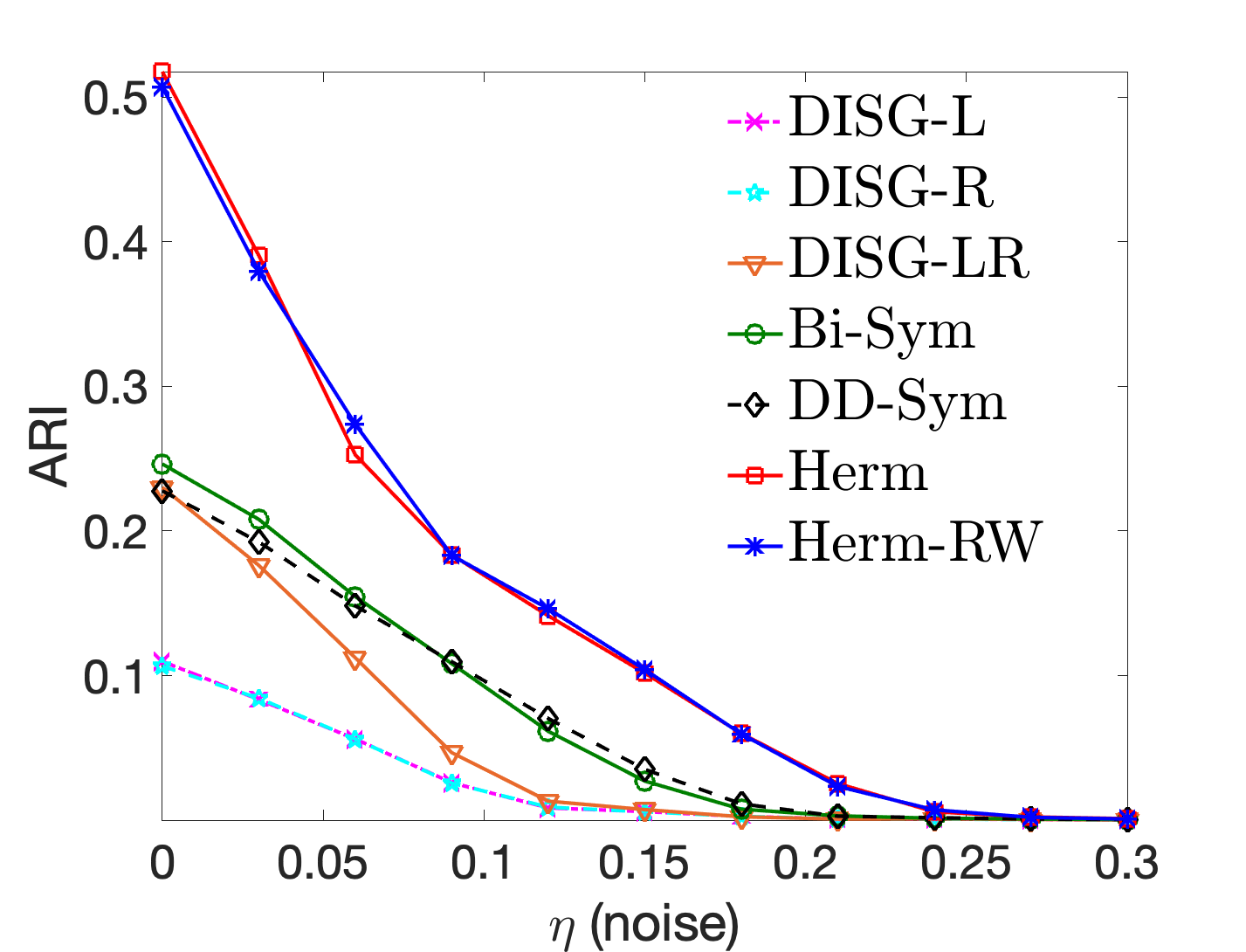}}
%\hspace{-4mm} 
\subcaptionbox{$p=0.5\%$}[0.245\columnwidth]{\includegraphics[width=0.270\columnwidth]{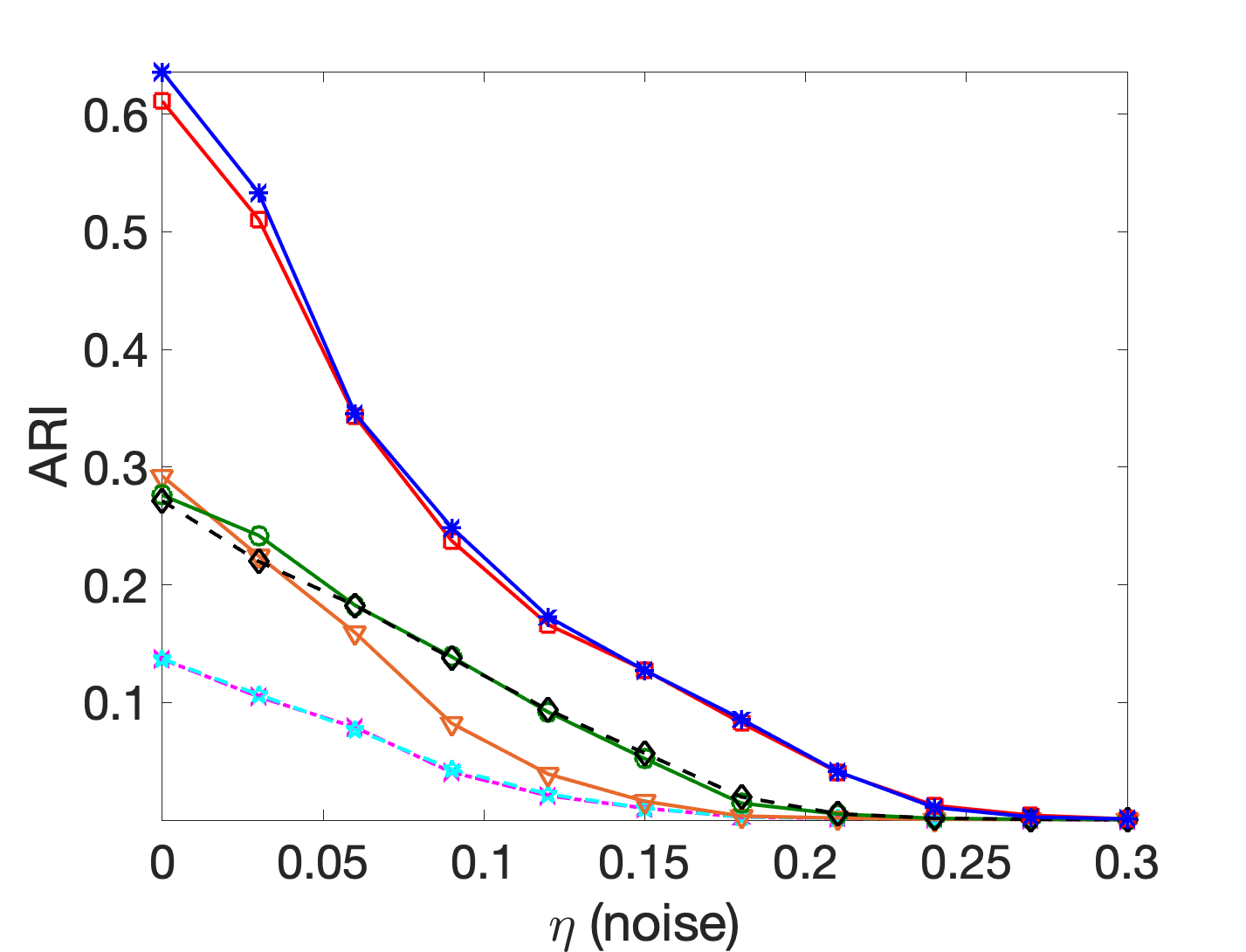}}
%\hspace{-4mm} 
\subcaptionbox{$p=0.6\%$}[0.2450\columnwidth]{\includegraphics[width=0.270\columnwidth]{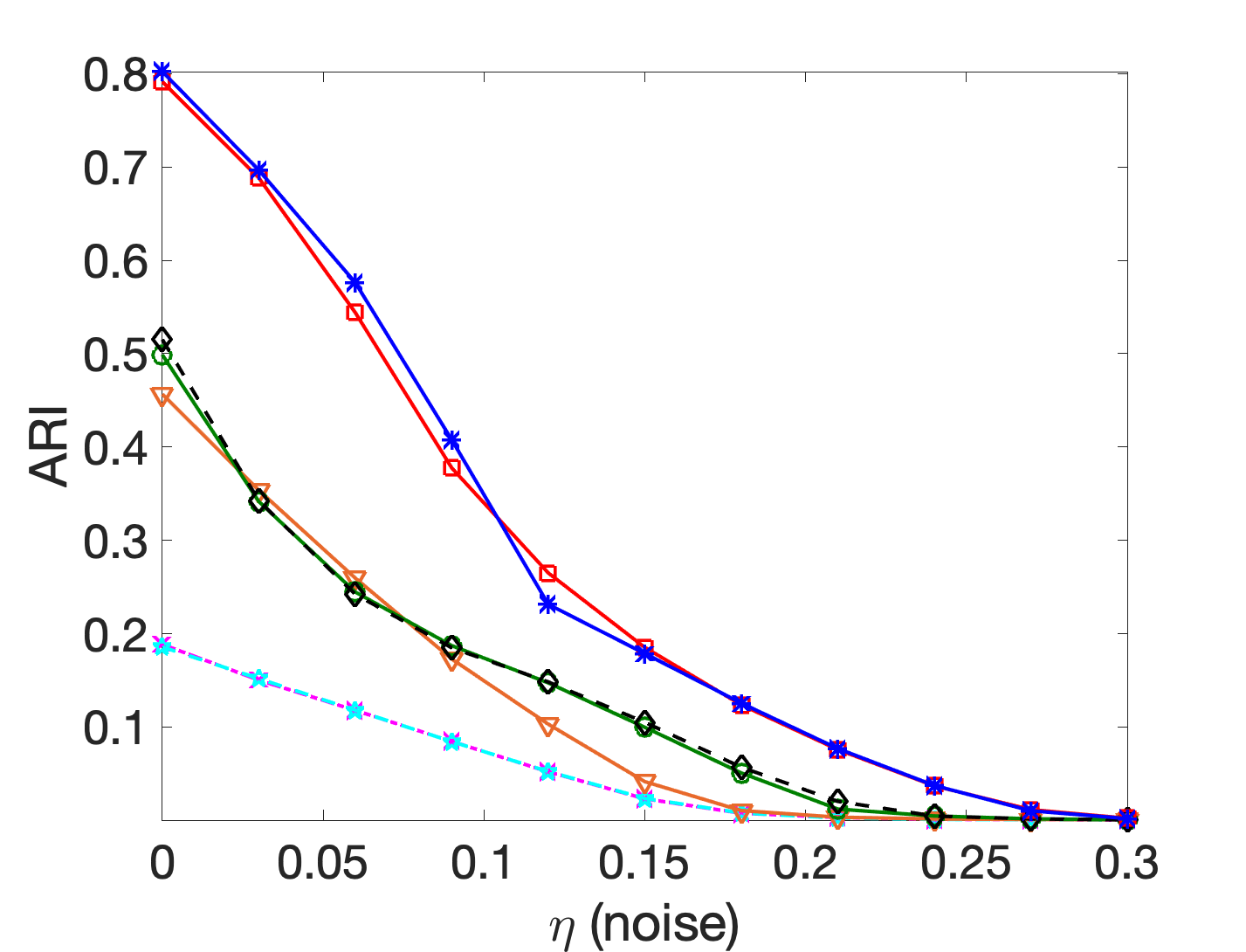}}
%\hspace{-4mm} 
\subcaptionbox{$p=0.8\%$}[0.2450\columnwidth]{\includegraphics[width=0.270\columnwidth]{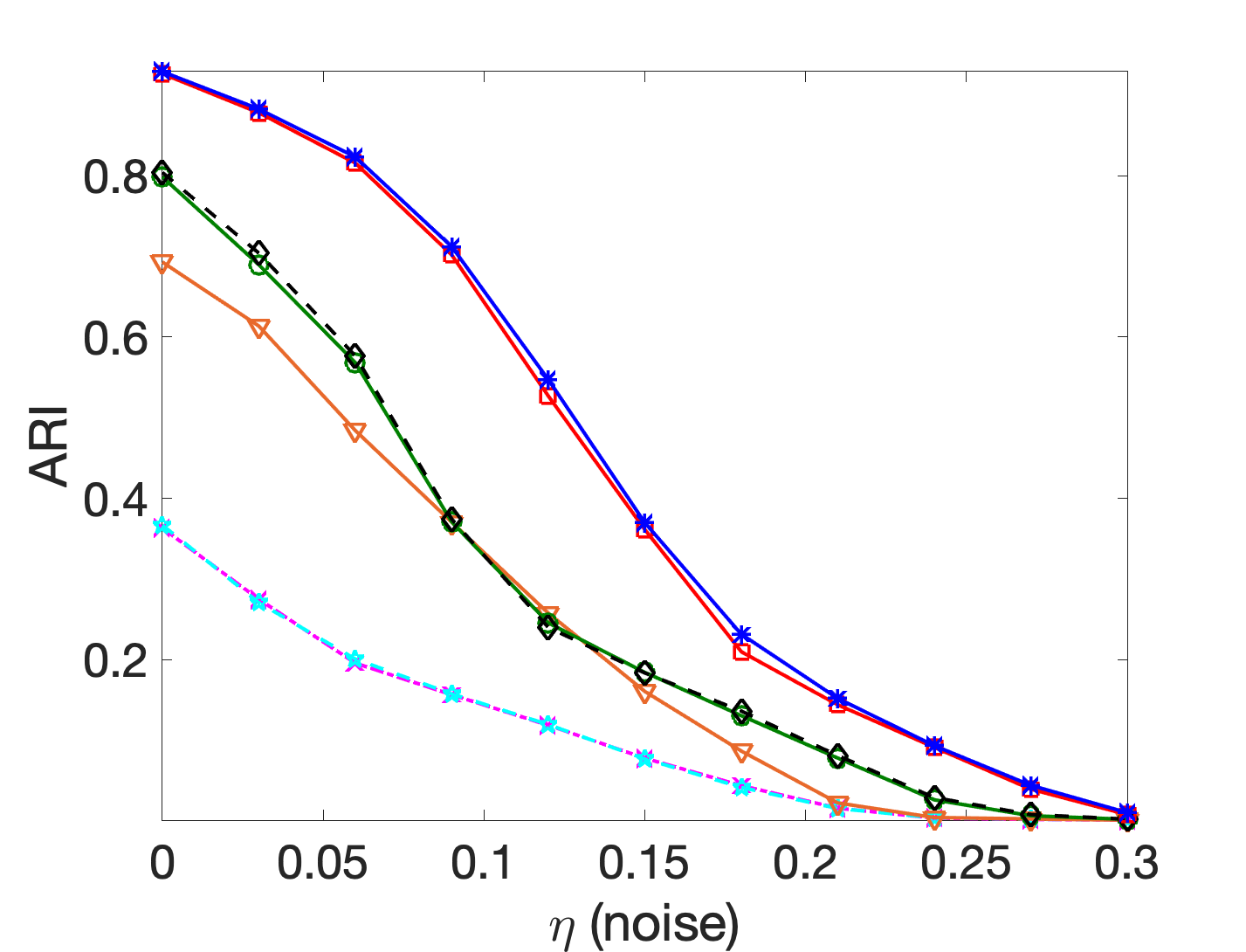}}
% 
\iffalse 
\end{centering}
\captionsetup{width=0.99\linewidth}
%\vspace{-1mm}
\caption{\small  Recovery rates for the circular pattern  with $k=5$, $N=5000$ with various levels of sparsity.
}
\label{fig:scanID_4b_circular}	
\end{figure*}
\vspace{-1mm}

\begin{figure*}[!htp]
\captionsetup[subfigure]{skip=0pt}
\begin{centering}
\fi 
\subcaptionbox{$p=0.45\%$}[0.2450\columnwidth]{\includegraphics[width=0.270\columnwidth]{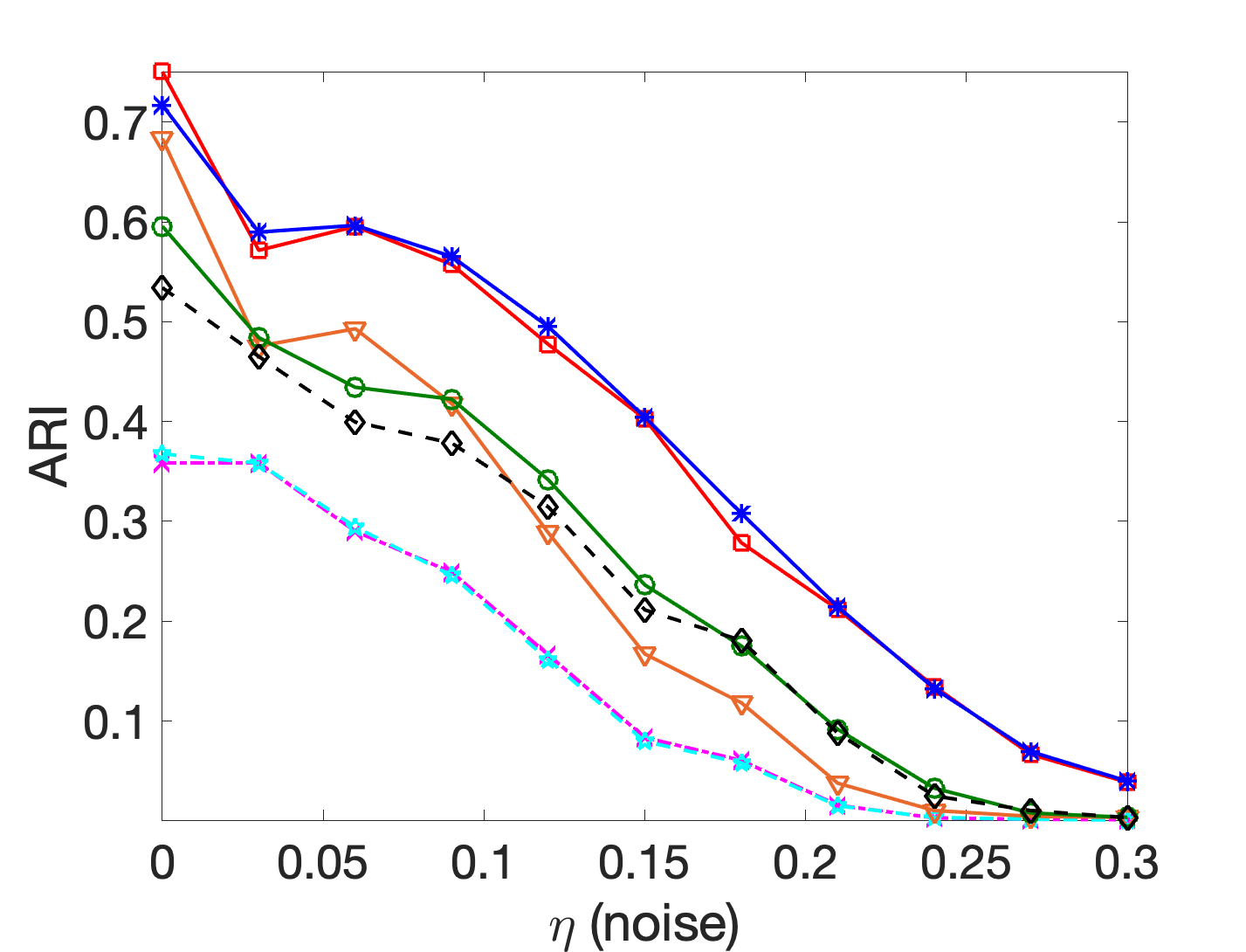}}
\hspace{-1mm} 
\subcaptionbox{$p=0.5\%$}[0.2450\columnwidth]{\includegraphics[width=0.270\columnwidth]{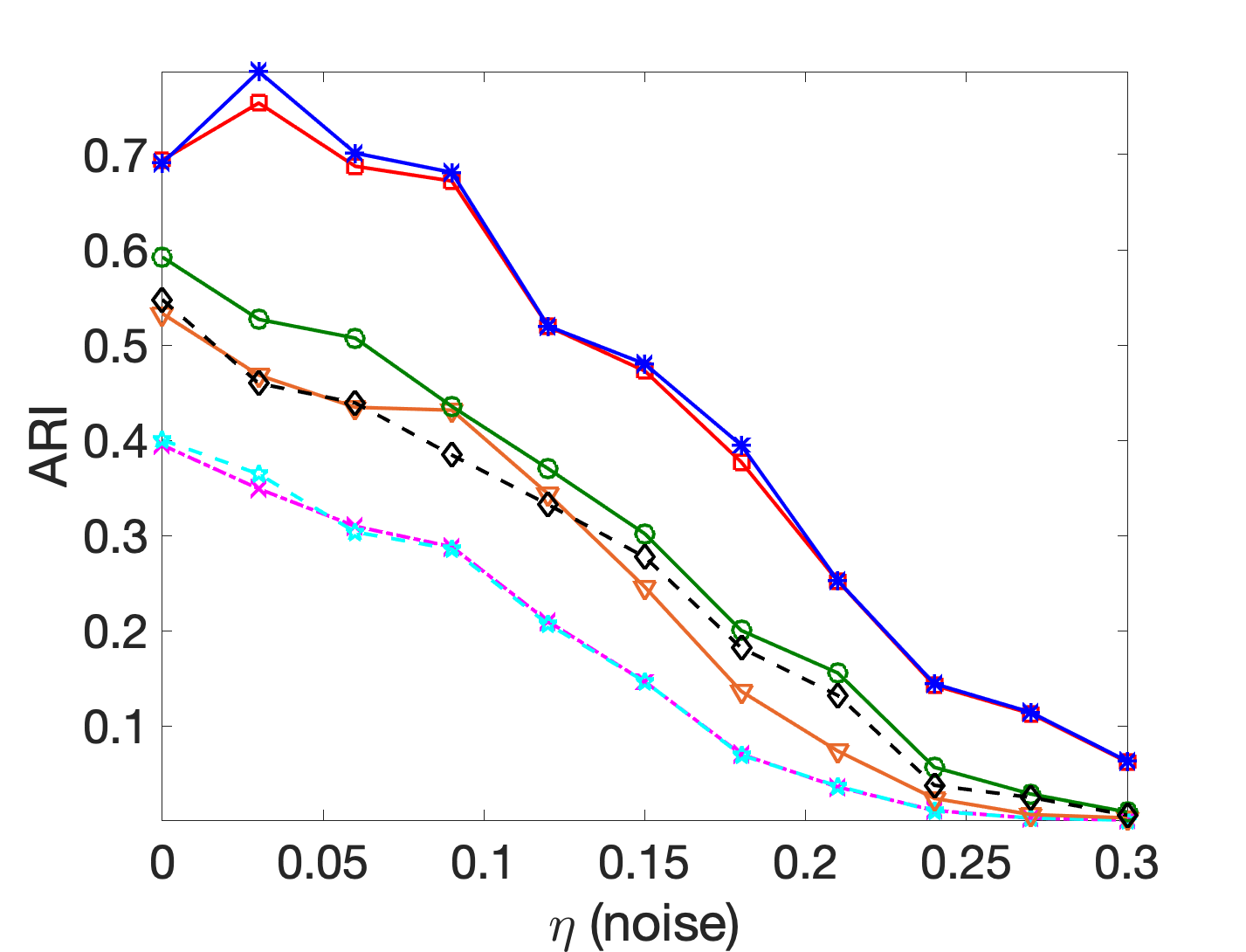}}
\hspace{-1mm} 
\subcaptionbox{$p=0.6\%$}[0.2450\columnwidth]{\includegraphics[width=0.270\columnwidth]{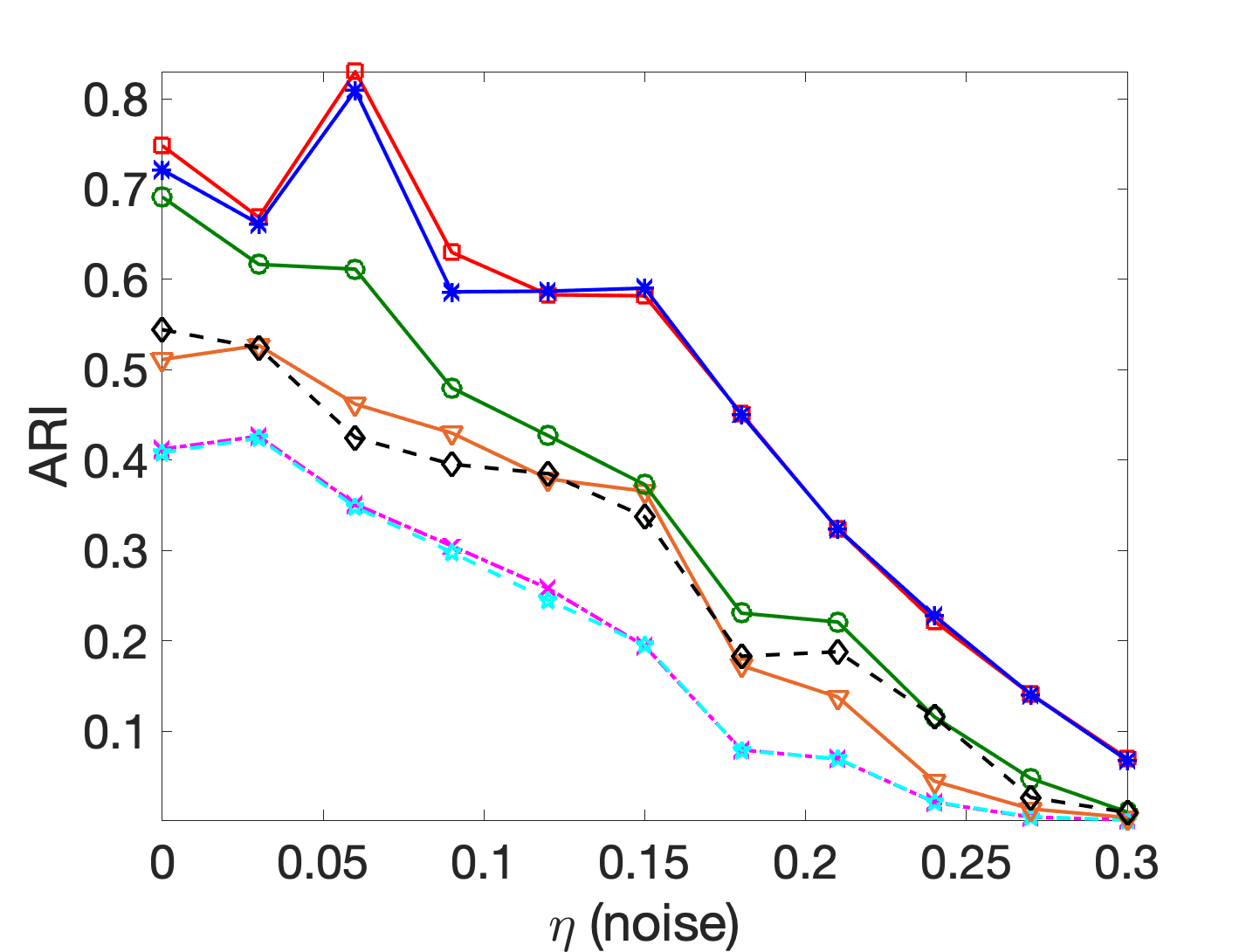}}
\hspace{-1mm} 
\subcaptionbox{$p=0.8\%$}[0.2450\columnwidth]{\includegraphics[width=0.270\columnwidth]{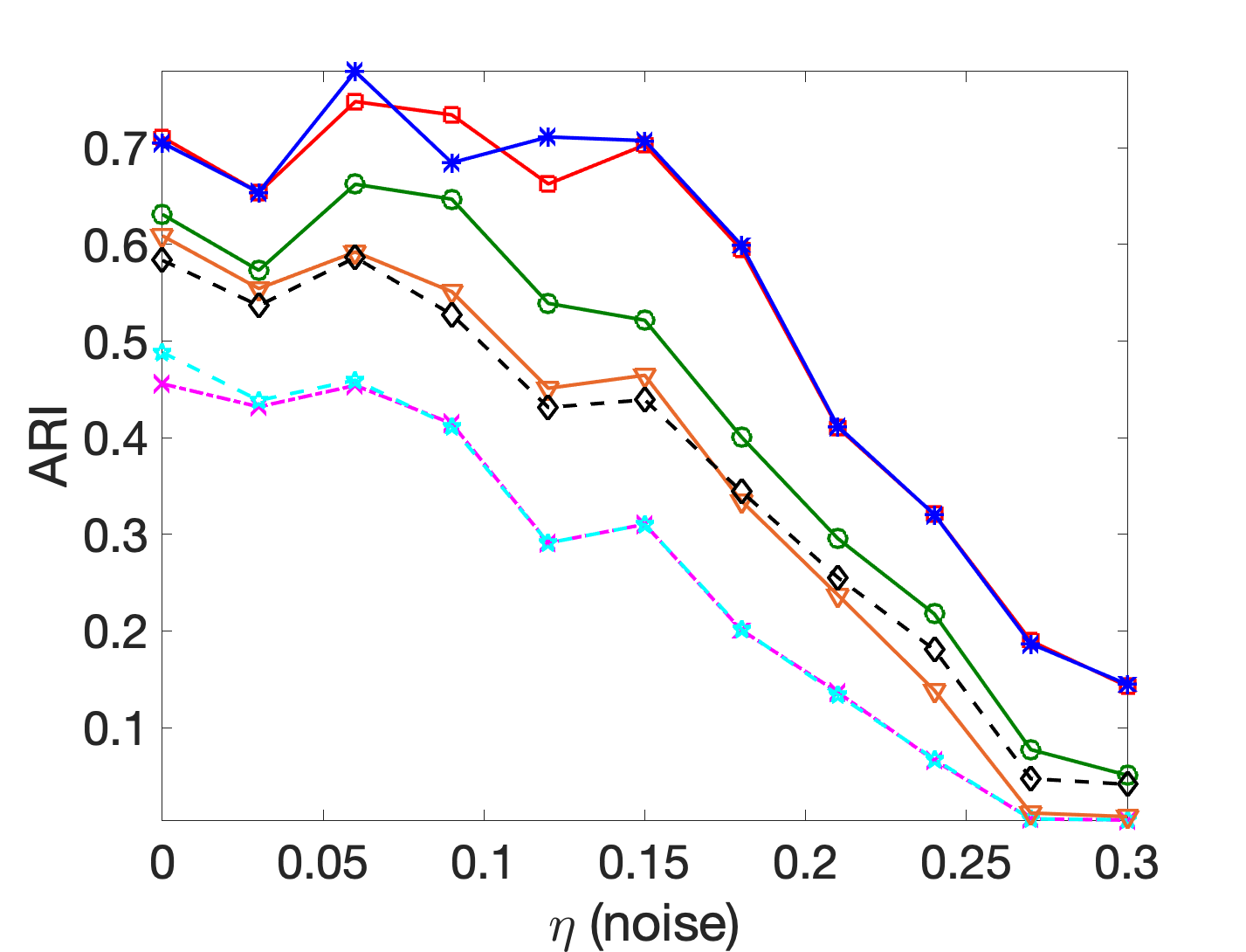}}
\end{centering}
\captionsetup{width=0.99\linewidth}
\vspace{-1mm}
\caption{\small Recovery rates for the circular pattern (top) and complete meta-graph (bottom) ($N=5,000, k=5$).  % , and various levels of sparsity.
}
% \label{fig:scanID_4a_Kk}
\label{fig:scanID_4ab}	
\end{figure*}
\vspace{-1mm}

% \figref{scanID_4b_circular} and \figref{scanID_4a_Kk} 
\figref{scanID_4ab} reports the performance of all the tested algorithms for input graphs from the DSBM with $N=5,000$, $k=5$, and the meta-graph is    a directed cycle, %(as in Section \ref{sec:cyclic})
or a complete graph with random orientations of the edges.
The two variants of our algorithm give similar results due to the fact that all the vertices have the same expected degree, and they % We also remark that, 
perform significantly better than all other algorithms. While all methods are unable to find a meaningful cluster structure when $\eta$ is close to $0.3$, our algorithm performs significantly better, especially for smaller values of $\eta$.

\iffalse 
 \vspace{-1mm}
\begin{figure}[h]  % [t]
\captionsetup[subfigure]{skip=0pt}
\begin{centering}
\subcaptionbox{$p = 0.01$ }[0.50\columnwidth]{\includegraphics[width=0.25\columnwidth]{Figures/scanID_1a_Kk/scanID_1a_n100_k50_p0p01_Kk_nrReps10_ARI.png}}
\hspace{-4mm} 
\subcaptionbox{  $p = 0.02$ }[0.50\columnwidth]{\includegraphics[width=0.250\columnwidth]{Figures/scanID_1b_Kk/scanID_1b_n100_k50_p0p02_Kk_nrReps10_ARI.png}}
\end{centering}
\captionsetup{width=0.99\linewidth}
\vspace{-1mm}
\caption{\small Recovery rates for the complete meta-graph in the \textsc{DSBM} with  $k=50$, % clusters,
$N=5000$, two sparsity values $p$.
Averaged over 10 runs. }
\label{fig:largeclusters}
\end{figure}
\fi 

% \vspace{-2mm}

We further investigate the performance of all algorithms for a large value of $k$. \figref{largeclusters_main} reports the ARI values of a randomly generated graph with respect to different values of  $\eta$, with  $N=5,000$, $k=50$, $p=1\%$, and the underlying meta-graph is a complete graph. This regime of parameters, i.e., large $k$ and relatively small $p$, is of particular interest due to its prevalence in most real-world data sets,
%\todo{Luca: I don't like this sentence. Difficult to argue it's true.}, 
and clearly illustrates that our algorithm has overwhelmingly superior performance compared to other algorithms in the literature.

\begin{wrapfigure}{r}{.26\textwidth}
    \begin{minipage}{\linewidth}
    \vspace{-0.1cm}
    \centering 
    \includegraphics[width=1.07\columnwidth]{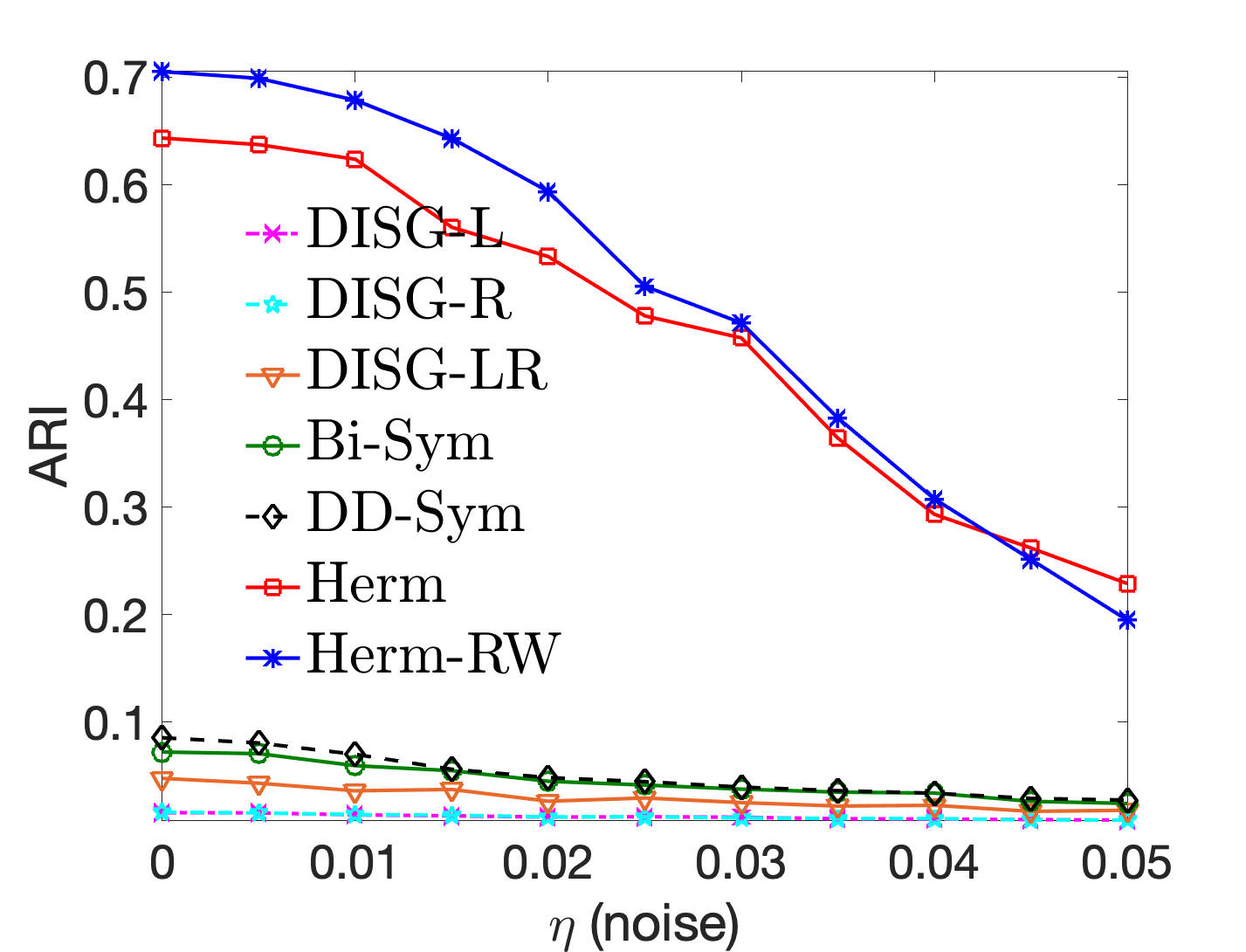}  % \subcaption{$p = 0.01$}  % \label{fig:5a}  % \par\vfill 
    %  \hspace{-10mm} 
   % \includegraphics[width=0.49\columnwidth]{Figures/scanID_1b_Kk/scanID_1b_n100_k50_p0p02_Kk_nrReps10_ARI.png}  % \subcaption{$p = 0.02$} % \label{fig:5b}
% Note: if we turn on the \subcaption{} - the two figures end up one under the other...
\end{minipage}
\caption{\small  % Recovery rates 
% Comparison on the 
Complete meta-graph (\textsc{DSBM}, $k=50$).}
\vspace{-3mm}
\label{fig:largeclusters_main}
\end{wrapfigure}

\vspace{-2mm} 
\paragraph{Experimental results for real-world data.}
We also detail results on real-world data sets, showcasing the efficiency and robustness of our algorithm for identifying structures in % directed graphs. 
digraphs.  
Since no ground truth clustering is available, we compare % the 
performance as measured by three related objective functions (also referred to as scores), 
% also 
showing   % at the same time 
that our approach favours balanced cluster sizes. We consider a  \textsc{US-Migration} network, and a \textsc{BLOG} network during the 2004 US presidential election; additional experimental results on a  \textsc{UK-Migration} network and % the
\textsc{$c$-Elegans} neural network  
% can be found
are shown in the appendix.

% \noteMC{Define objective function: IF: Imbalance Flow matrix, or IC: Imbalance Cut, or CI: "Cut Imbalance" I think this is best, as the other way around kind of already implies it is an Imbalanced Cut, but if the entry is close to 0.5, then it is not. So I think CUT imbalance makes more sense.. what do you think?
%  define  so we can late change globally if we want }
%CI
%CI_s
%CI_v

\iffalse 
\begin{minipage}{0.5\linewidth}  
\begin{equation}
\cis(X,Y) = \cip(X,Y) \cdot \min \{ |X|, |Y| \} 
\label{def:cis}
\end{equation}
\end{minipage}  \hspace{0.5cm}  
\begin{minipage}{0.49\linewidth}  
  \begin{equation}
\civ(X,Y) = \cis(X,Y) \cdot \min \{ \text{vol}(A), \text{vol}(B) \},
\label{def:civ}
\end{equation}
\end{minipage}
\fi

For any two disjoint vertex sets $X$ and $Y$, we define the Cut Imbalance ratio between $X$ and $Y$ by  
\vspace{-1mm}
\begin{equation}       \label{eq:def_cip}
\cip(X,Y) =  \frac{1}{2}\cdot \left|\frac{ w(X,Y) -w(Y,X) }{ w(X,Y) + w(Y,X) }\right|
= \left|  \frac{w(X,Y)}{w(X,Y) + w(Y,X)} - \frac{1}{2} \right|,
\end{equation}
where $w(X,Y) = \sum_{u\in X, v\in Y} w(u,v)$,
and  define  the size and volume normalised versions by 
% Cut Imbalance ratio with respect to their volumes  by 
\vspace{-2mm}
\begin{equation}      \label{eq:def_cis}
\cis(X,Y) =\cip(X,Y)  \cdot \min \{ |X|, |Y| \},
\end{equation}
\vspace{-4mm}
\begin{equation}          \label{eq:def_civ}
\civ(X,Y) =\cip(X,Y)  \cdot \min \{ \text{vol}(X), \text{vol}(Y) \},
\end{equation}
where $\mathrm{vol}(X)$ is the sum of in-degrees and out-degrees of the vertices in $X$. To explain Equations~(\ref{eq:def_cis}) and (\ref{eq:def_civ}), notice that 
$\cip(X,Y)\in[0,1/2]$ quantifies the imbalance of the edge directions between $X$ and $Y$, with $\cip(X,Y)=0$~(resp. $\cip(X,Y)=1/2$) indicating that the directions of the edges between $X$ and $Y$ are completely balanced~(resp. imbalanced). Furthermore, since our objective is to identify pairs of clusters with a large $\cip$-value, we scale $\cip(X,Y)$  by the minimum of their sizes or volumes   to penalise small clusters, in the same spirit as the normalised cut value~\cite{ShiM00}.

% Since our goal is to capture intrinsic structures in terms of flow imbalance between pairs of subsets of vertices, we quantify this measure for the entire graph by considering all ${k \choose 2} $ pairs of clusters. For instance, given a clustering $C_1, \ldots, C_k$ of our graph $G$, the total volume-normalised cluster imbalance of $G$ is given by

% \vspace{-3mm}
% \begin{equation}
%     \civ = \sum_{0\leq j < \ell \leq k-1} \civ(C_j,C_{\ell})
% \end{equation}
% \vspace{-2mm}

% \noindent and similarly for $\cip$ and $\cip^{\mathrm{size}}$. Furthermore, since in most practical applications one does not expect for cut imbalances to manifest between all pairs of clusters (in other words, the meta-graph $F$ is typically also sparse) we only consider the top $2k$ largest pairs

% \vspace{-3mm}
% \begin{equation}
%     \text{Top}\civ = \sum_{t=1}^{2k} \civ(\widetilde{C}_{j_t},C_{{\ell}_t})
% \end{equation}
% \vspace{-2mm}

% \noindent  where $ \widetilde{C}_{j_t},C_{\ell_t} $ denotes the $t$-th largest $ \civ$ cut imbalance pair. In other words, we only consider the top $2k$ largest pairs  in terms of the values of their cut imbalance normalised by volume for $\civtop$, and their cut imbalance normalised by vertex size for $\cistop$ respectively.

\underline{\emph{\textsc{US-Migration} Network.}} We consider the 2000 US Census data, which reports the number of people that migrated between pairs of counties in the US during 1995-2000 \cite{census,census_rep}. This data can be expressed as a matrix $M\in\mathbb{Z}_{\geq 0}^{N\times N}$,  where $N=3107$ denotes the number of counties in mainland US, and 
$ M_{j\ell} $ denotes the total number of people that migrated from county $j$ to county $\ell$. %  during the five-year period. 
% Data is also available on the population of each county, but we have not taken that into account in our study. 
% mig1:
We consider the % following 
transformation  % of the data 
$ \widetilde{M}_{j\ell} = M_{j\ell} /( M_{j\ell} + M_{\ell j} )$, which leads to a matrix often encountered in various  applications. For example, in ranking, this could capture the fraction of games won by player $j$ in the match against $\ell$ \cite{RankCentrality}. The input matrix to our pipeline is given by the skew symmetric matrix $G =  \widetilde{M} -  \widetilde{M}^{\rot}$. 
Figure \ref{fig:scanID_8a_mig1} shows the $\civ$ values  for the top pairs 
% the $\cistop$ and $\civtop$ objective function values, 
for varying number of clusters. % for $k = \{ 3,10,20,40\}$. 
With respect to both scores, \textsc{Herm} and \textsc{Herm-RW} are consistently better across all top pairs, and outperform all other methods by a large margin especially for $k=10,20$. Additional experiments for a variant of this data set are deferred to the appendix.

%\afterpage{
\begin{figure}[h]
\centering 
\captionsetup[subfigure]{skip=2pt}
\subcaptionbox{$k=2$}[0.24\columnwidth]{\includegraphics[width=0.26\columnwidth]{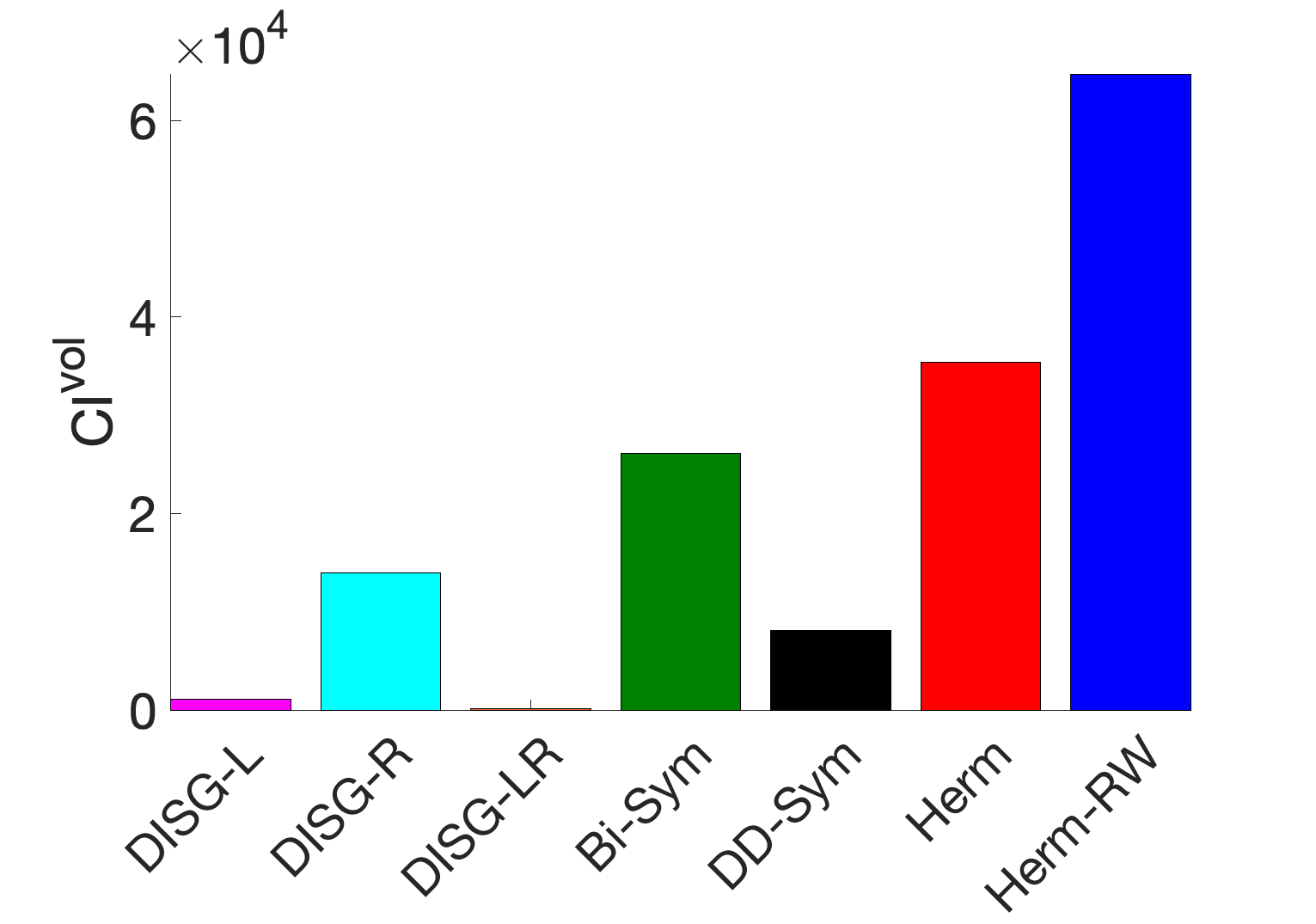} } %  \subcaption{$p = 0.01$}    \label{fig:5a}  % \par\vfill 
%  \hspace{-10mm}
\subcaptionbox{$k=3$}[0.24\columnwidth]{\includegraphics[width=0.26\columnwidth]{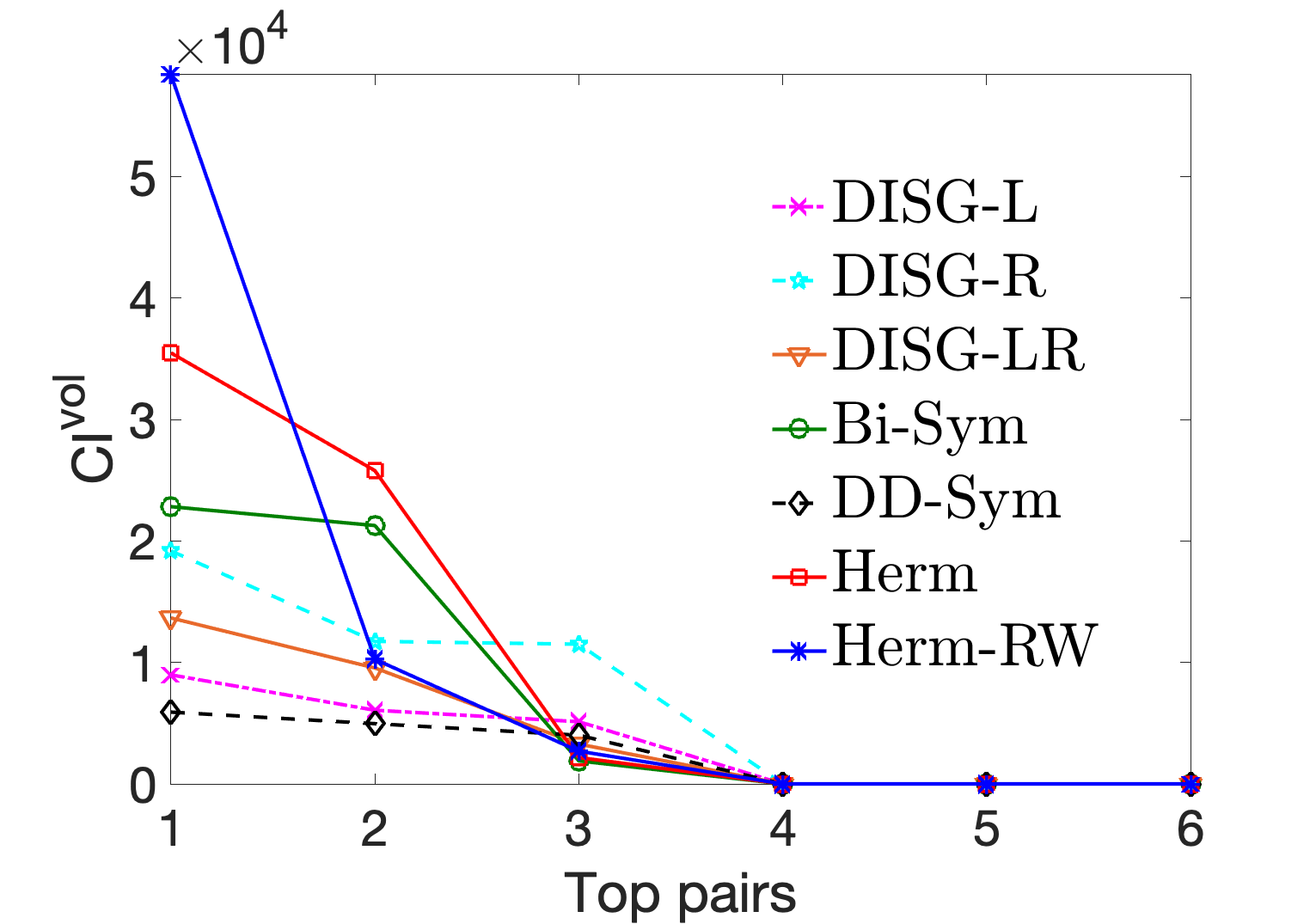} }
\subcaptionbox{$k=10$}[0.24\columnwidth]{\includegraphics[width=0.26\columnwidth]{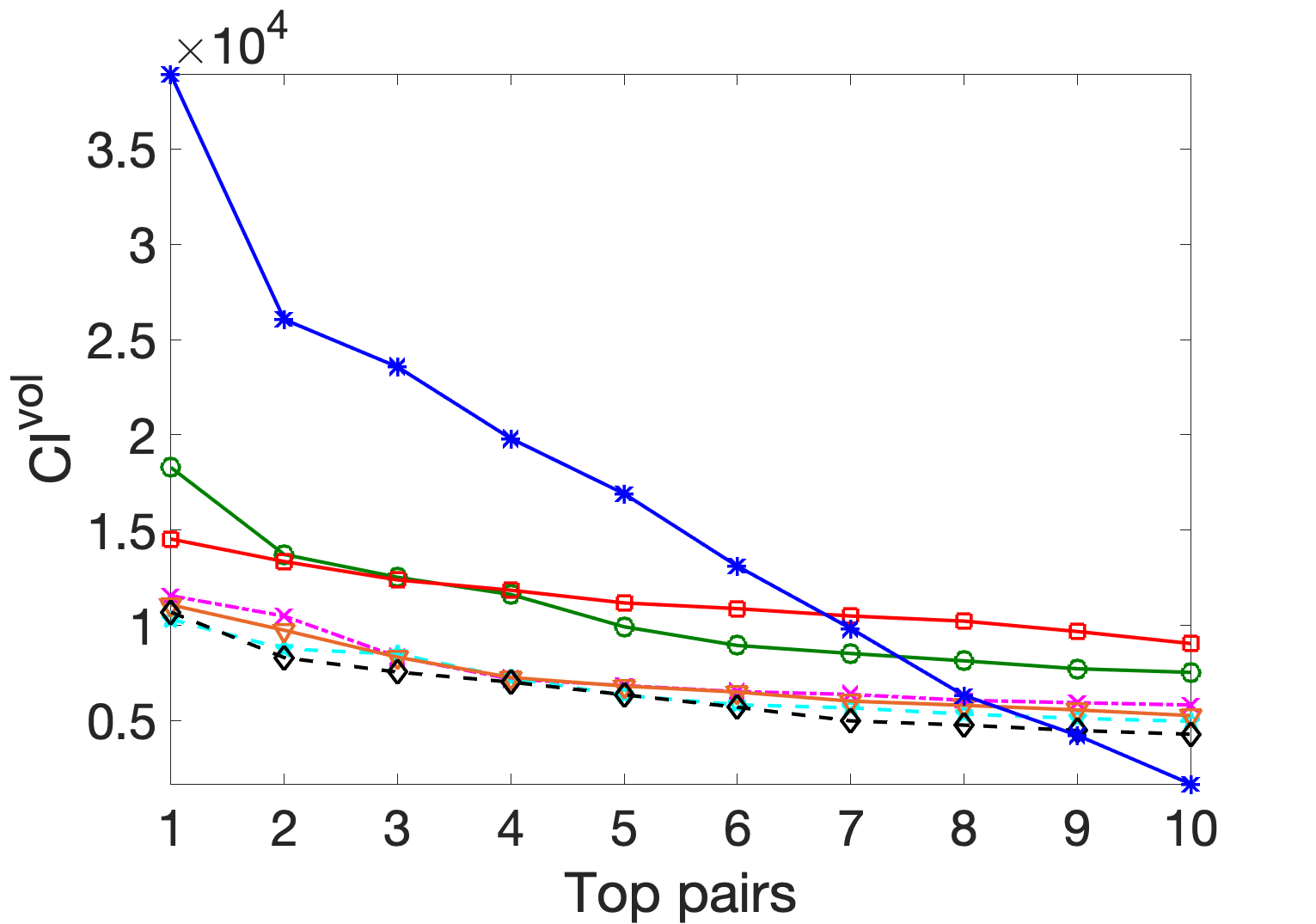} }
\subcaptionbox{$k=20$}[0.24\columnwidth]{\includegraphics[width=0.26\columnwidth]{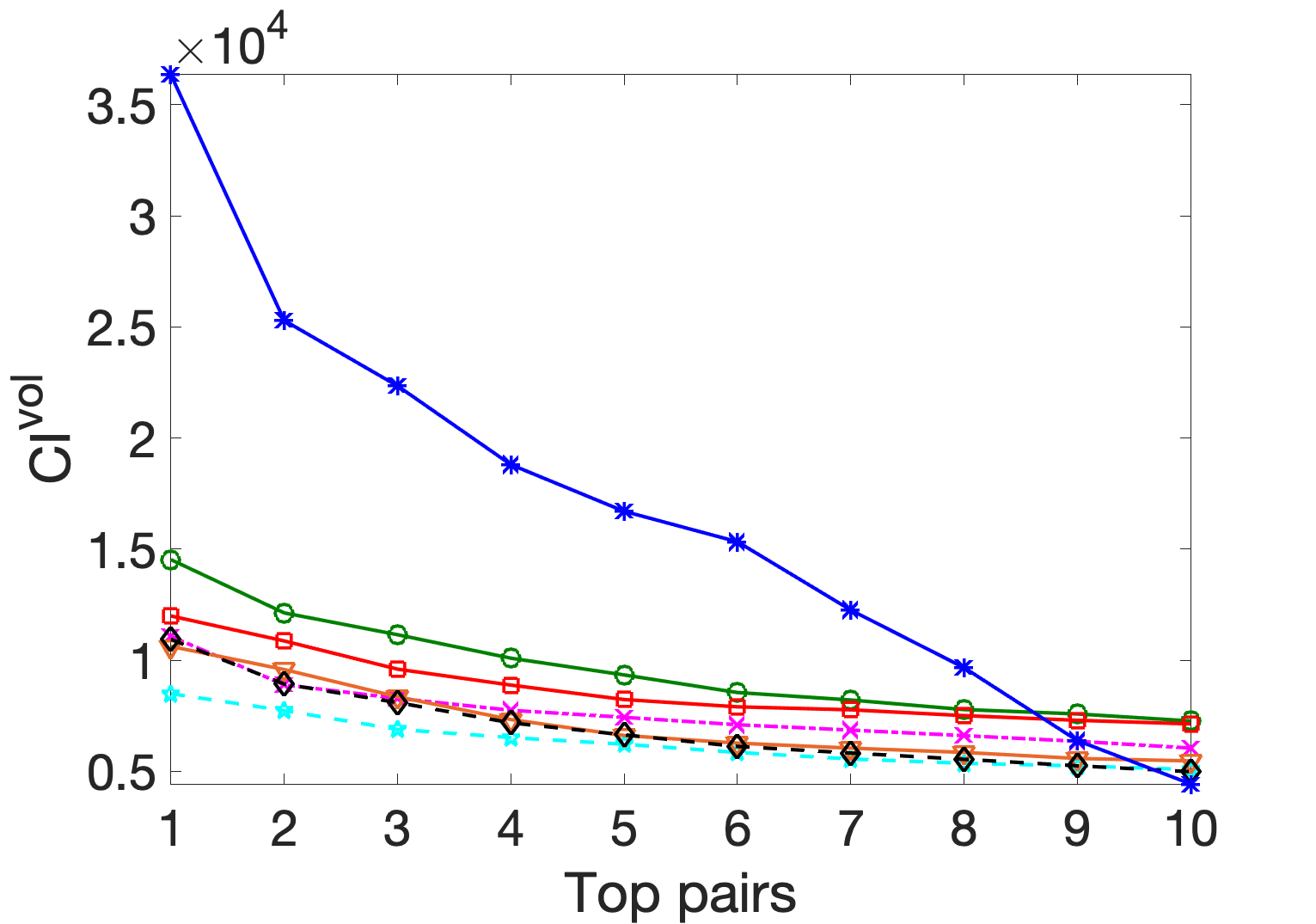} }
% \subcaption{$p = 0.02$} % \label{fig:5b}
% Note: if we turn on the \subcaption{} - the two figures end up one under the other...
\caption{\small Top $\civ$ scores attained by pairs of clusters, for the \textsc{US-migration} data set with varying $k$.
} 
\label{fig:scanID_8a_mig1}
\end{figure}%}
\vspace{-1mm} 

Figure \ref{fig:instances_mig1_ClustObj} shows the clusterings recovered by several methods for $k=10$, and % as well as  
% 
% Figure \ref{fig:mig1_G} shows
heatmaps of the % graph 
adjacency matrices sorted by induced cluster membership, highlighting  the fact that \textsc{DISGLR} and \textsc{DD-Sym} tend to uncover traditional clusters of high internal edge-density, as hinted by the prominent block-diagonal structure.  
On the other hand, % our two methods 
\text{Herm} and \text{Herm-RW} do not exhibit such a structure, and contain block submatrices of high intensity (denoting a large cut imbalance) on the off-diagonal blocks. 
Figure \ref{fig:mig1Top5} shows the three %largest size-normalised cut imbalance pairs, i.e., the 
pairs of clusters for which $\cis(C_j, C_{\ell})$ is the largest.
% , for the case $k=10$. 
We highlighted the two clusters in each pair in red (source) and blue (destination), and provided the % numerical 
values for their respective cut imbalances $ \cip $, $\cis$ and $\civ$. With respect to the two normalised cut imbalances, $ \textsc{Herm-RW} $ vastly outperforms all other methods.  
\begin{figure} % [!htp]
\captionsetup[subfigure]{skip=0pt}
\begin{centering}
% LOUT_END: \subcaptionbox{ \textsc{DISGL} }[0.512\columnwidth]{\includegraphics[width=0.51\columnwidth]{Figures/instances_mig1_ClustObj/mig1_k10_DISGL__clust.jpg}} \hspace{-4mm} 
% LOUT_END: \subcaptionbox{  \textsc{DISGR}  }[0.512\columnwidth]{\includegraphics[width=0.512\columnwidth]{Figures/instances_mig1_ClustObj/mig1_k10_DISGR__clust.jpg}} \hspace{-4mm} 
\subcaptionbox{\textsc{DISGLR}}[0.245\columnwidth]{\includegraphics[width=0.26\columnwidth]{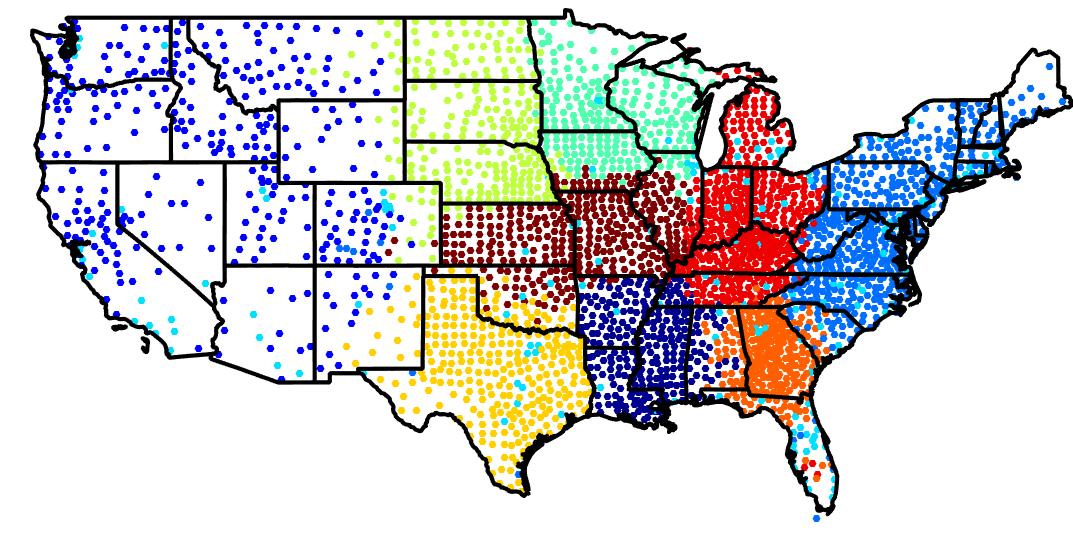}} \hspace{-1mm} 
% LOUT_END: \subcaptionbox{  \textsc{Bi-Sym}}[0.512\columnwidth]{\includegraphics[width=0.512\columnwidth]{Figures/instances_mig1_ClustObj/mig1_k10_BiSym__clust.jpg}} \hspace{-4mm} 
\subcaptionbox{\textsc{DD-Sym}}[0.245\columnwidth]{\includegraphics[width=0.26\columnwidth]{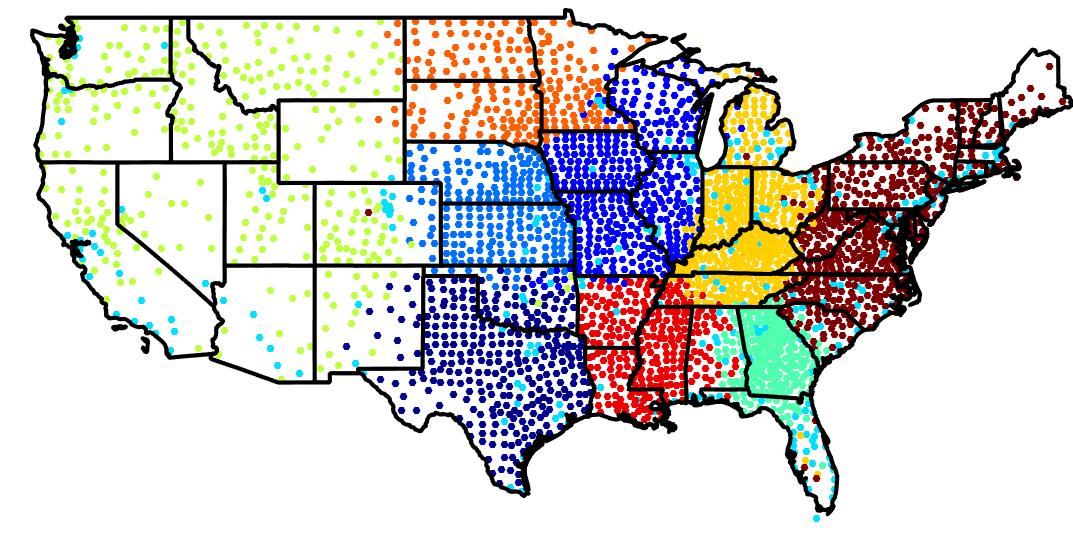}} \hspace{-1mm} 
\subcaptionbox{\textsc{Herm}}[0.245\columnwidth]{\includegraphics[width=0.26\columnwidth]{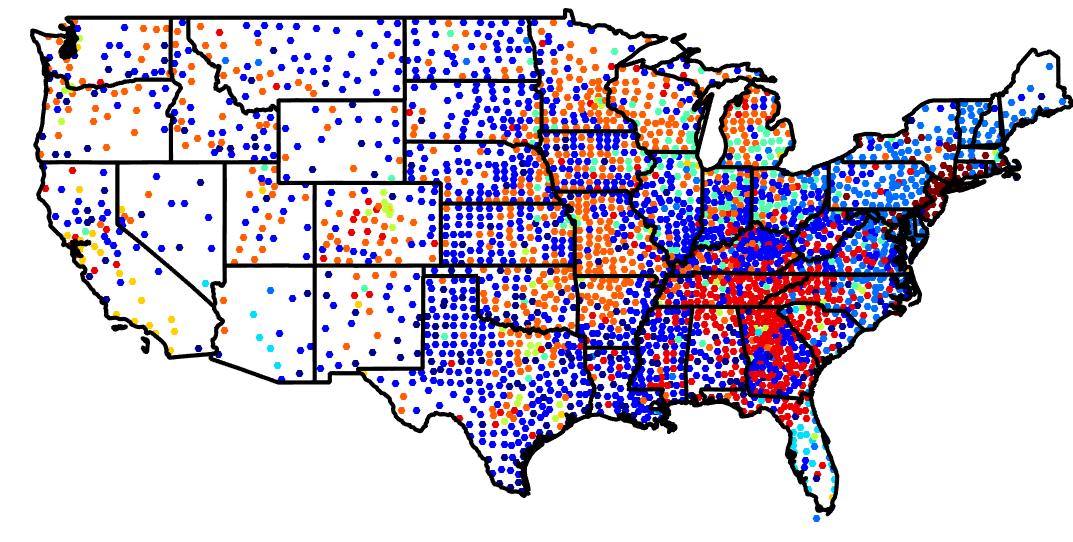}}
\hspace{-1mm} 
% \subcaptionbox{  \textsc{Herm-RW} }[0.512\columnwidth]{\includegraphics[width=0.512\columnwidth]{Figures/instances_mig1_ClustObj/mig1_k10_HermRW__clust.jpg}}
\subcaptionbox{\textsc{Herm-RW}}[0.245\columnwidth]{\includegraphics[width=0.26\columnwidth]{Figures/misc/mig1_k10_HermRW__clust_hsv.jpg}}
% \hspace{-4mm} 
% LOUT_END: \subcaptionbox{  \textsc{Herm-Sym} }[0.512\columnwidth]{\includegraphics[width=0.512\columnwidth]{Figures/instances_mig1_ClustObj/mig1_k10_HermSym__clust.jpg}}\hspace{-4mm} 
% LOUT_END: \subcaptionbox{  \textsc{Naive} }[0.512\columnwidth]{\includegraphics[width=0.512\columnwidth]{Figures/instances_mig1_ClustObj/mig1_k10_Naive__clust.jpg}}
%
%
\subcaptionbox{\textsc{DISGLR}}[0.245\columnwidth]{\includegraphics[width=0.26\columnwidth]{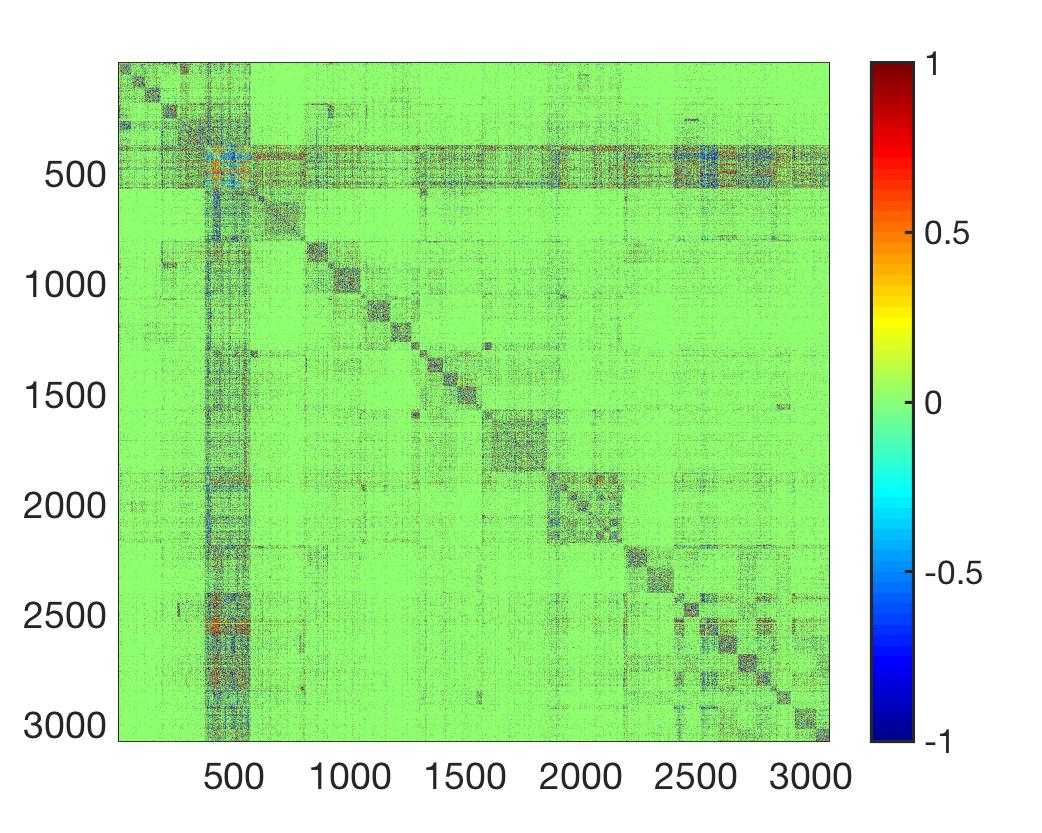}} \hspace{-1mm} 
% LOUT_END: \subcaptionbox{ \textsc{Bi-Sym}}[0.512\columnwidth]{\includegraphics[width=0.512\columnwidth]{Figures/mig1_G/mig1_k10_BiSym__G.jpg}} \hspace{-4mm} 
\subcaptionbox{ \textsc{DD-Sym}}[0.245\columnwidth]{\includegraphics[width=0.26\columnwidth]{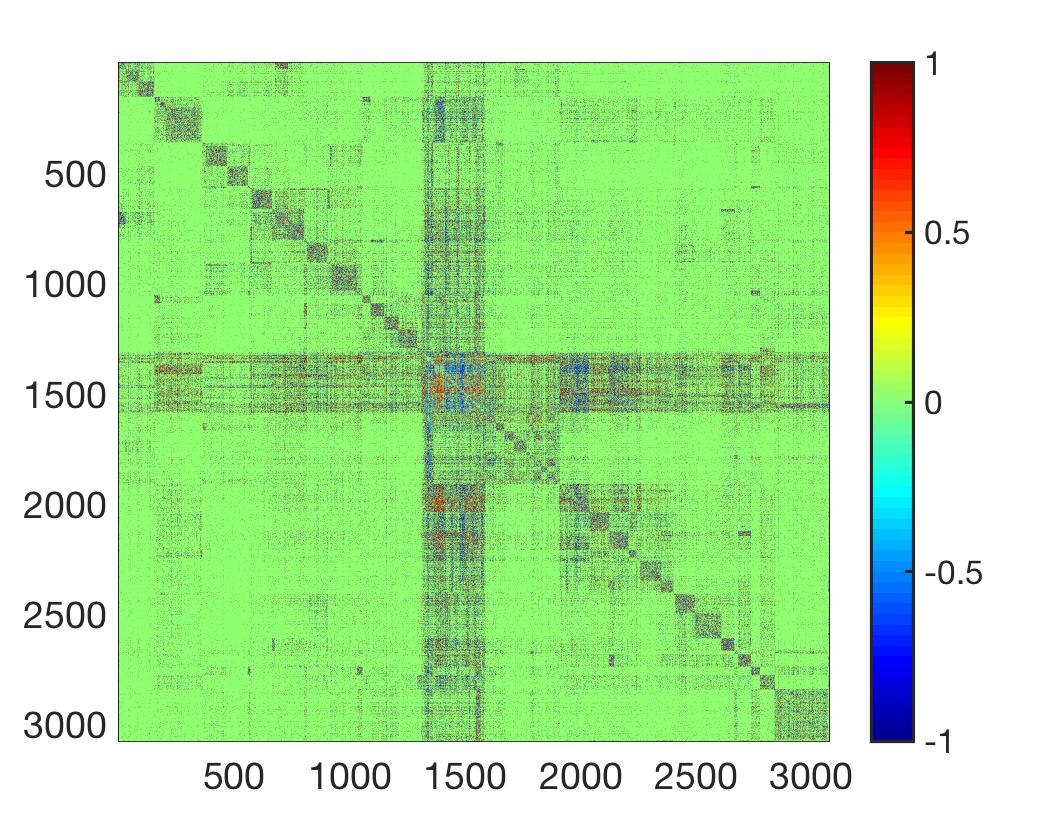}} \hspace{-1mm} 
\subcaptionbox{\textsc{Herm}}[0.245\columnwidth]{\includegraphics[width=0.26\columnwidth]{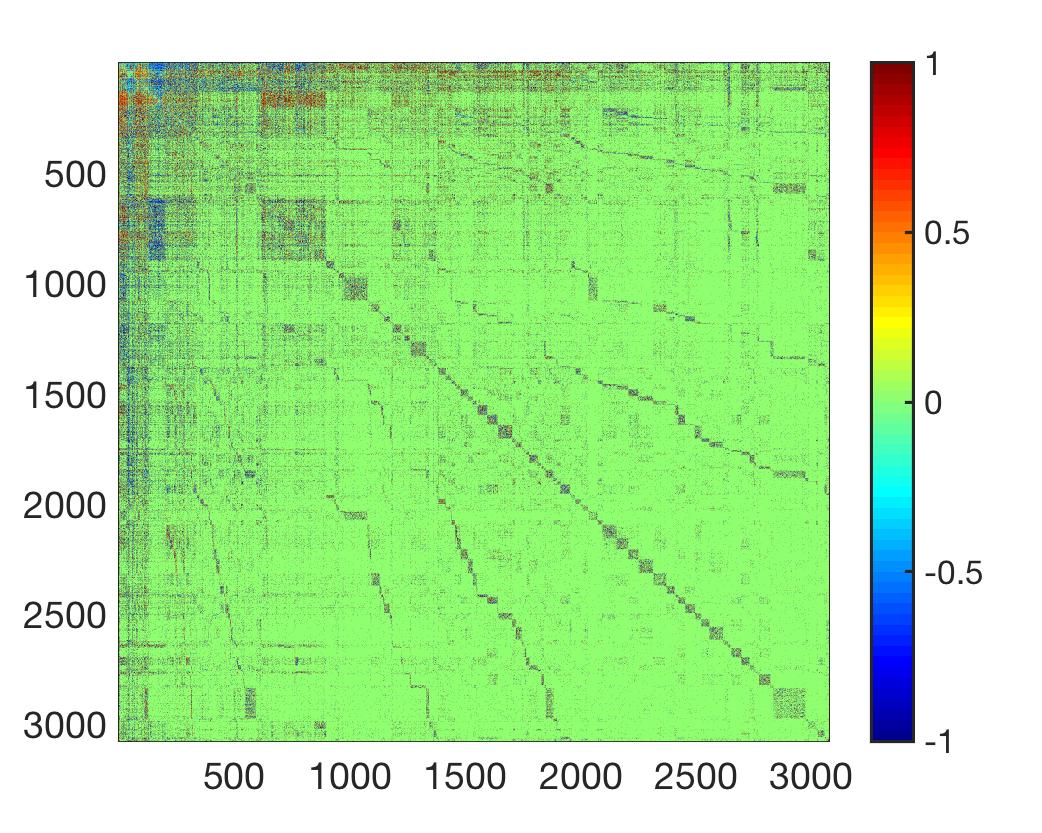}} \hspace{-1mm} 
\subcaptionbox{\textsc{Herm-RW}}[0.245\columnwidth]{\includegraphics[width=0.26\columnwidth]{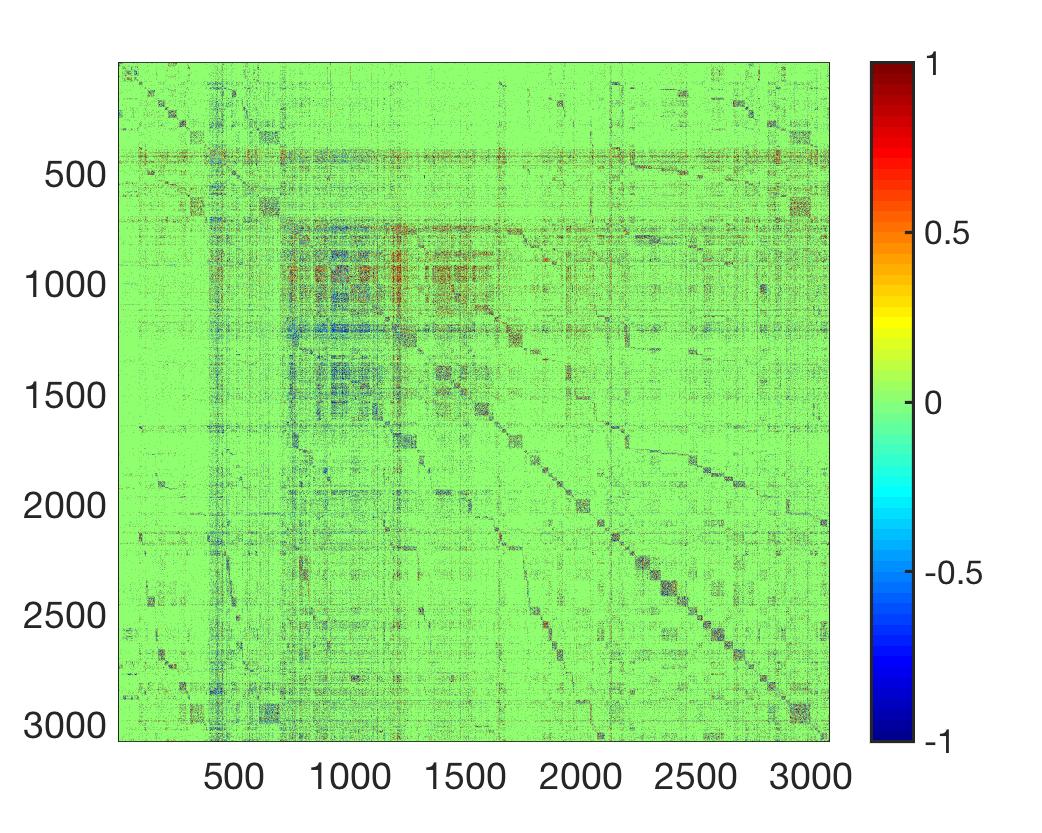}}%  \hspace{-4mm}
\end{centering}
\captionsetup{width=0.99\linewidth}
\caption{Top: Recovered clusterings for the \textsc{US-Migration} data set with $k=10$ clusters. Bottom: Heatmap of the graph adjacency matrices, sorted by induced cluster membership.}
\label{fig:instances_mig1_ClustObj}
\end{figure}
\vspace{0mm}

\vspace{-2mm}
\newcommand{\wid}{1.5in}
\newcolumntype{C}{>{\centering\arraybackslash}m{\wid}}
\begin{table*}[!htp]\sffamily
\hspace{2mm}
 % LOUT_END:     \begin{tabular}{l*5{C}@{}}
   % LOUT_END:   & I & II & III & IV & V \\ 
% \begin{tabular}{l*4{C}@{}}
\begin{tabular}{l*3{C}@{}}
   & I & II & III  \\  
% 
\iffalse     % LOUT_END: 
\textsc{NAIVE}
& \includegraphics[width=\wid]{{Figures/mig1Top5/mig1_k10_Naive__TopIFMpair1.jpg}} 
& \includegraphics[width=\wid]{{Figures/mig1Top5/mig1_k10_Naive__TopIFMpair2.jpg}} 
& \includegraphics[width=\wid]{{Figures/mig1Top5/mig1_k10_Naive__TopIFMpair3.jpg}} 
& \includegraphics[width=\wid]{{Figures/mig1Top5/mig1_k10_Naive__TopIFMpair4.jpg}} 
& \includegraphics[width=\wid]{{Figures/mig1Top5/mig1_k10_Naive__TopIFMpair5.jpg}} \\ 
\fi 
%
\iffalse 
\textsc{DISGL}
& \includegraphics[width=\wid]{{Figures/mig1Top5/mig1_k10_DISGL__TopIFMpair1.jpg}} 
& \includegraphics[width=\wid]{{Figures/mig1Top5/mig1_k10_DISGL__TopIFMpair2.jpg}} 
& \includegraphics[width=\wid]{{Figures/mig1Top5/mig1_k10_DISGL__TopIFMpair3.jpg}} 
& \includegraphics[width=\wid]{{Figures/mig1Top5/mig1_k10_DISGL__TopIFMpair4.jpg}} 
& \includegraphics[width=\wid]{{Figures/mig1Top5/mig1_k10_DISGL__TopIFMpair5.jpg}} \\ 
%
\textsc{DISGR}
& \includegraphics[width=\wid]{{Figures/mig1Top5/mig1_k10_DISGR__TopIFMpair1.jpg}} 
& \includegraphics[width=\wid]{{Figures/mig1Top5/mig1_k10_DISGR__TopIFMpair2.jpg}} 
& \includegraphics[width=\wid]{{Figures/mig1Top5/mig1_k10_DISGR__TopIFMpair3.jpg}} 
& \includegraphics[width=\wid]{{Figures/mig1Top5/mig1_k10_DISGR__TopIFMpair4.jpg}} 
& \includegraphics[width=\wid]{{Figures/mig1Top5/mig1_k10_DISGR__TopIFMpair5.jpg}} \\ 
% 
\fi 
% 
\small{\textsc{DISGLR}}
& \includegraphics[width=\wid]{{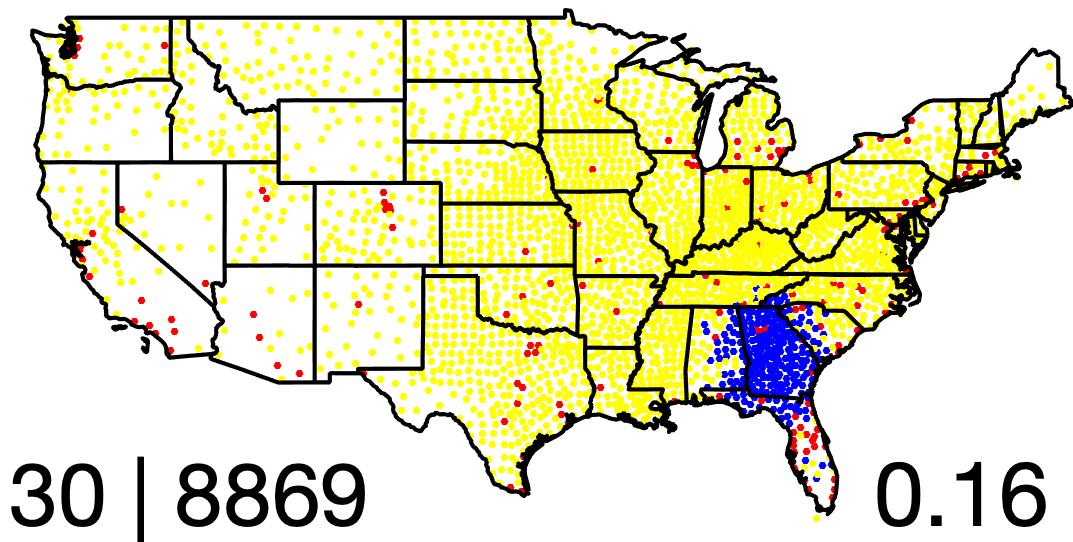}} 
& \includegraphics[width=\wid]{{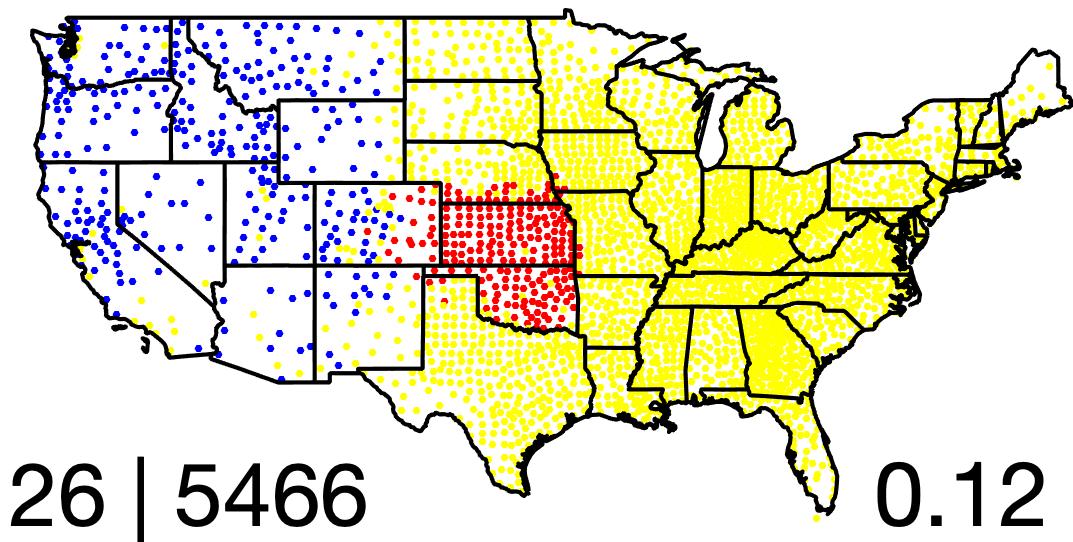}} 
& \includegraphics[width=\wid]{{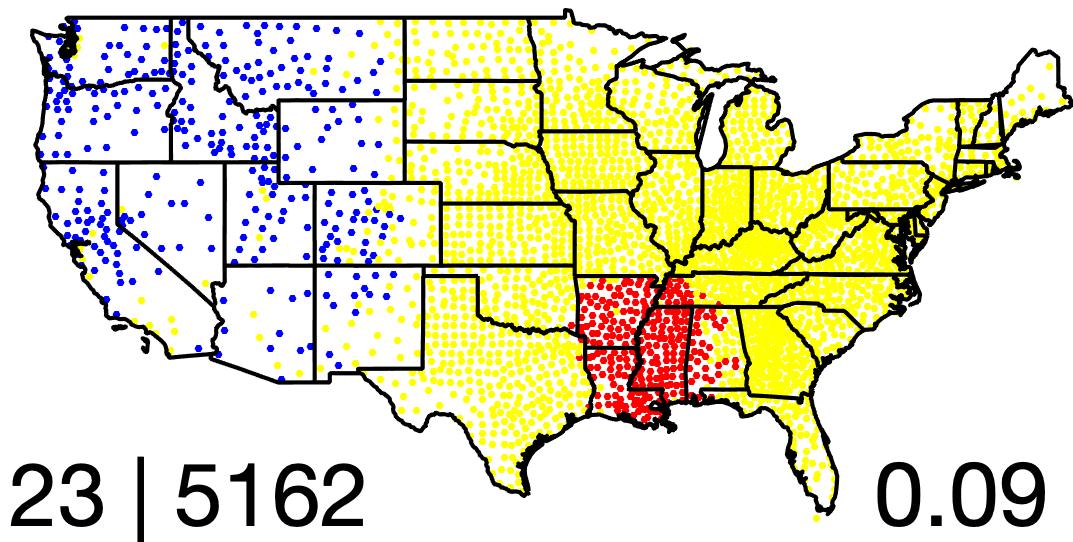}} 
% & \includegraphics[width=\wid]{{Figures/mig1Top5/mig1_k10_DISGLR__TopIFMpair4.jpg}} 
\\
 % LOUT_END: & \includegraphics[width=\wid]{{Figures/mig1Top5/mig1_k10_DISGLR__TopIFMpair5.jpg}} \\ 
%
\iffalse     % LOUT_END: 
\textsc{Bi-Sym}
& \includegraphics[width=\wid]{{Figures/mig1Top5/mig1_k10_BiSym__TopIFMpair1.jpg}} 
& \includegraphics[width=\wid]{{Figures/mig1Top5/mig1_k10_BiSym__TopIFMpair2.jpg}} 
& \includegraphics[width=\wid]{{Figures/mig1Top5/mig1_k10_BiSym__TopIFMpair3.jpg}} 
& \includegraphics[width=\wid]{{Figures/mig1Top5/mig1_k10_BiSym__TopIFMpair4.jpg}} 
& \includegraphics[width=\wid]{{Figures/mig1Top5/mig1_k10_BiSym__TopIFMpair5.jpg}} \\ 
\fi 
%
\small{\textsc{DD-Sym}}
& \includegraphics[width=\wid]{{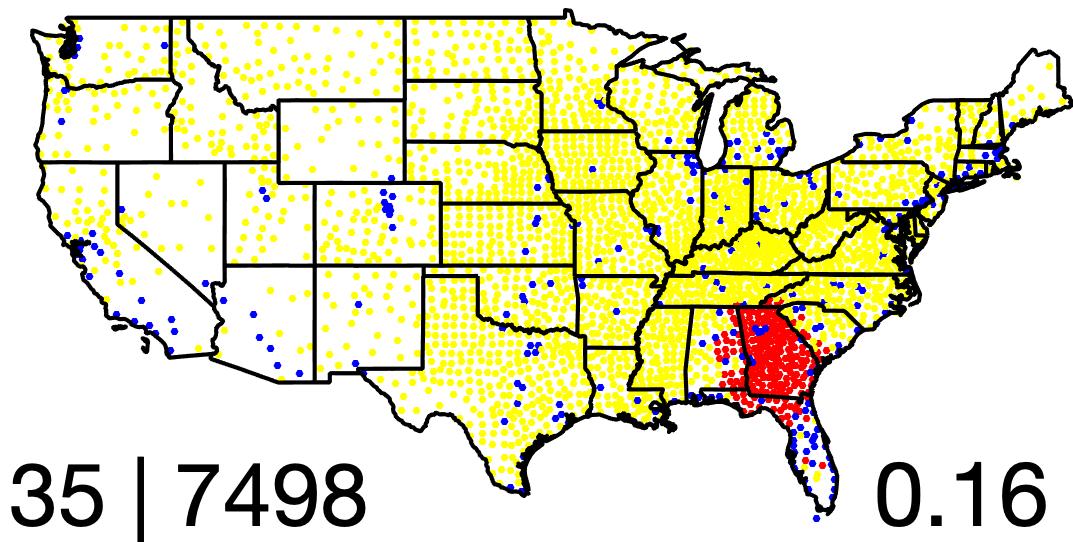}} 
& \includegraphics[width=\wid]{{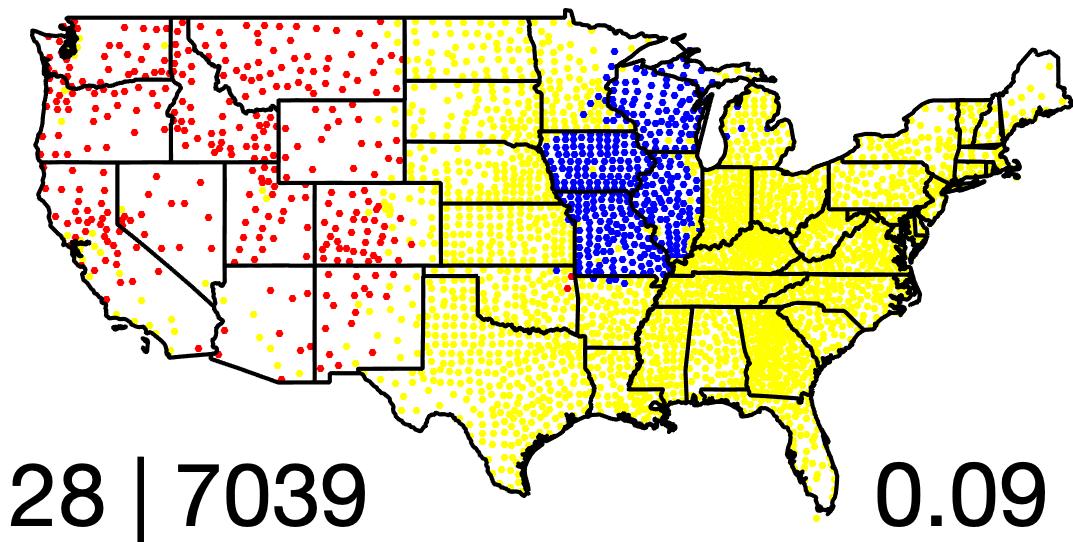}} 
& \includegraphics[width=\wid]{{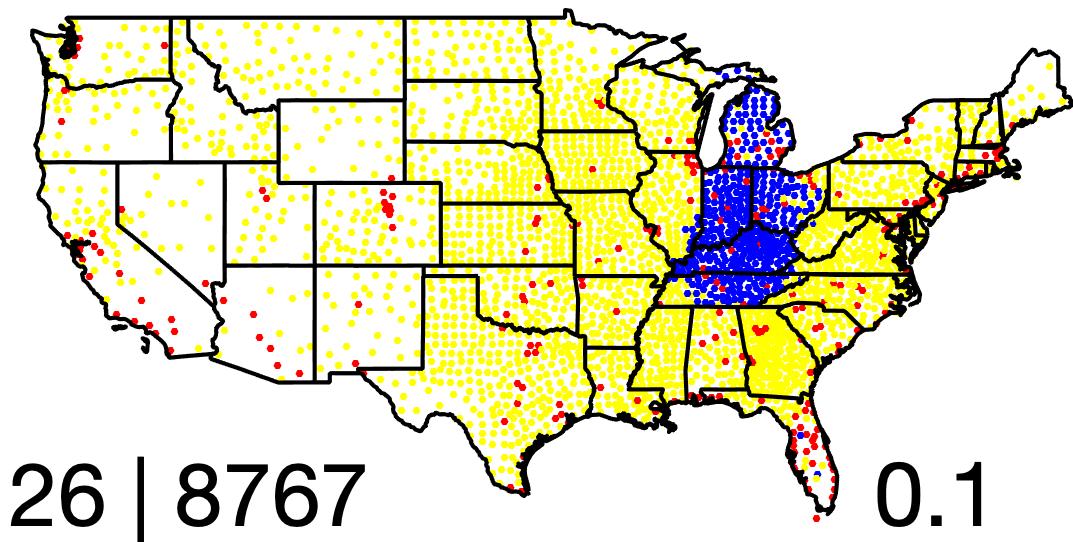}} 
% & \includegraphics[width=\wid]{{Figures/mig1Top5/mig1_k10_DDSym__TopIFMpair4.jpg}}  
\\
 % LOUT_END: & \includegraphics[width=\wid]{{Figures/mig1Top5/mig1_k10_DDSym__TopIFMpair5.jpg}} \\ 
% 
\iffalse     % LOUT_END: 
\textsc{Herm}
& \includegraphics[width=\wid]{{Figures/mig1Top5/mig1_k10_Herm__TopIFMpair1.jpg}} 
& \includegraphics[width=\wid]{{Figures/mig1Top5/mig1_k10_Herm__TopIFMpair2.jpg}} 
& \includegraphics[width=\wid]{{Figures/mig1Top5/mig1_k10_Herm__TopIFMpair3.jpg}} 
% & \includegraphics[width=\wid]{{Figures/mig1Top5/mig1_k10_Herm__TopIFMpair4.jpg}}
\\
\fi 
 % LOUT_END: & \includegraphics[width=\wid]{{Figures/mig1Top5/mig1_k10_Herm__TopIFMpair5.jpg}} \\ 
%
\small{\textsc{Herm-RW}}
& \includegraphics[width=\wid]{{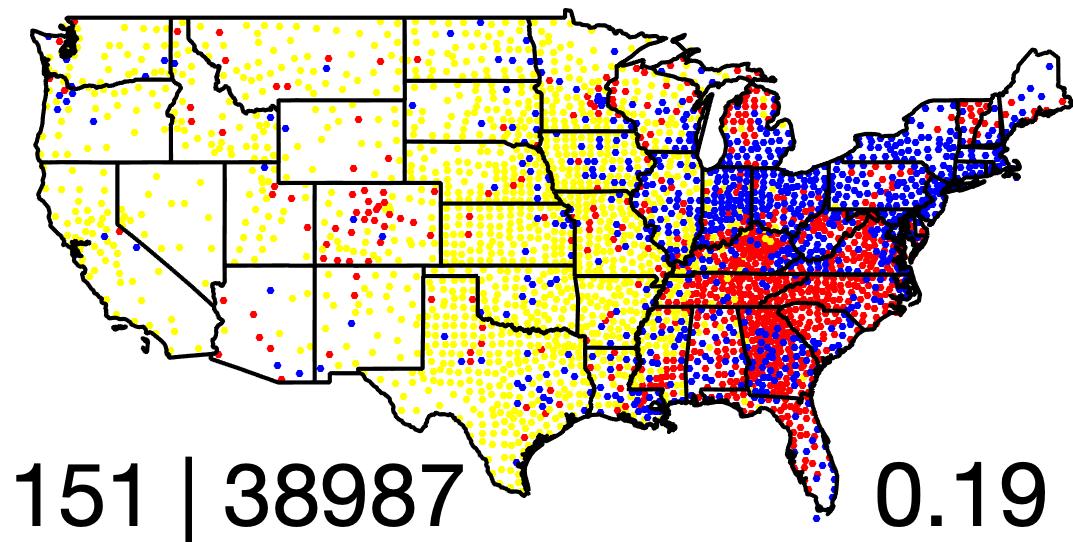}} 
& \includegraphics[width=\wid]{{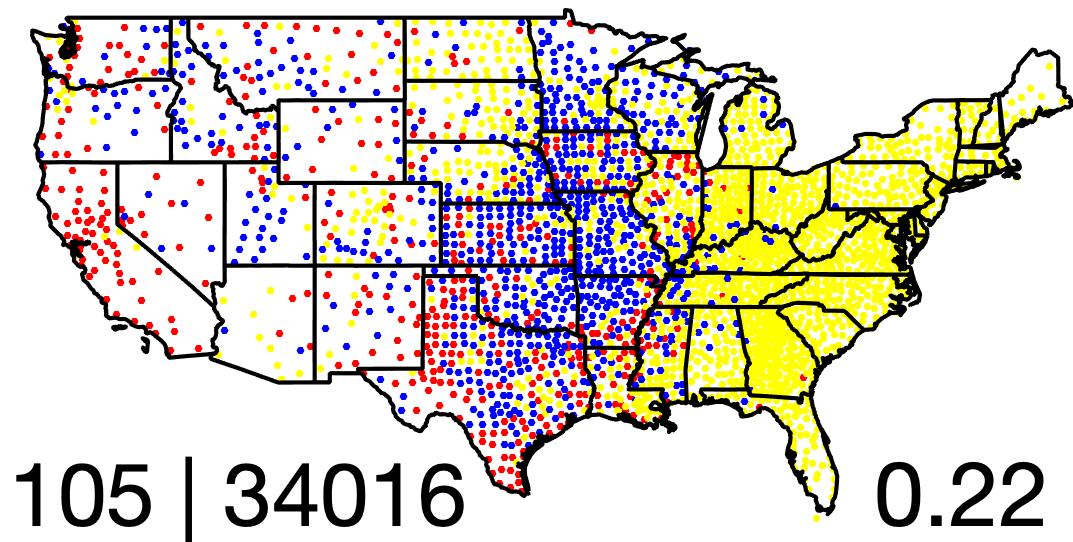}} 
& \includegraphics[width=\wid]{{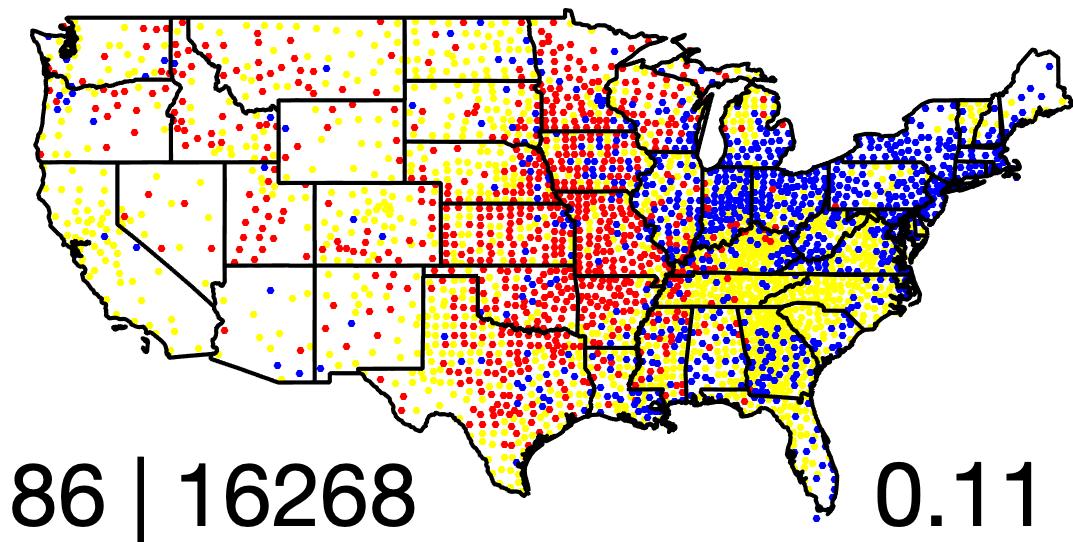}} 
% & \includegraphics[width=\wid]{{Figures/mig1Top5/mig1_k10_HermRW__TopIFMpair4.jpg}}
\\
 % LOUT_END: & \includegraphics[width=\wid]{{Figures/mig1Top5/mig1_k10_HermRW__TopIFMpair5.jpg}} \\
% 
\iffalse     % LOUT_END: 
\textsc{Herm-Sym}
& \includegraphics[width=\wid]{{Figures/mig1Top5/mig1_k10_HermSym__TopIFMpair1.jpg}} 
& \includegraphics[width=\wid]{{Figures/mig1Top5/mig1_k10_HermSym__TopIFMpair2.jpg}} 
& \includegraphics[width=\wid]{{Figures/mig1Top5/mig1_k10_HermSym__TopIFMpair3.jpg}} 
& \includegraphics[width=\wid]{{Figures/mig1Top5/mig1_k10_HermSym__TopIFMpair4.jpg}} 
& \includegraphics[width=\wid]{{Figures/mig1Top5/mig1_k10_HermSym__TopIFMpair5.jpg}} \\
\fi 
%
\end{tabular}
\captionsetup{width=0.99\linewidth}
\vspace{-1mm}
\captionof{figure}{\small The top three largest size-normalised cut imbalance pairs for the \textsc{US-Migration} data with $k=10$ clusters. Red denotes the source cluster, and blue denotes the destination cluster. For each plot, the bottom left text contains the numerical values (rounded to nearest integer) of the % two 
normalised  $\cis$ and $\civ$ pairwise cut imbalance values, and the bottom right text contains the 
% numerical 
$\cip$ cut imbalance value in $ [0,1/2] $.
% \noteMC{todo: drop color bar and insert the CI values [CI,CIM,CIV] in the figure (title or large text)}.
}
\label{fig:mig1Top5}	
\end{table*}

\vspace{0mm} 
\begin{figure}[t]
\centering 
\vspace{-2mm} 
\captionsetup[subfigure]{skip=2pt}
\subcaptionbox{$k=2$}[0.24\columnwidth]{\includegraphics[width=0.26\columnwidth]{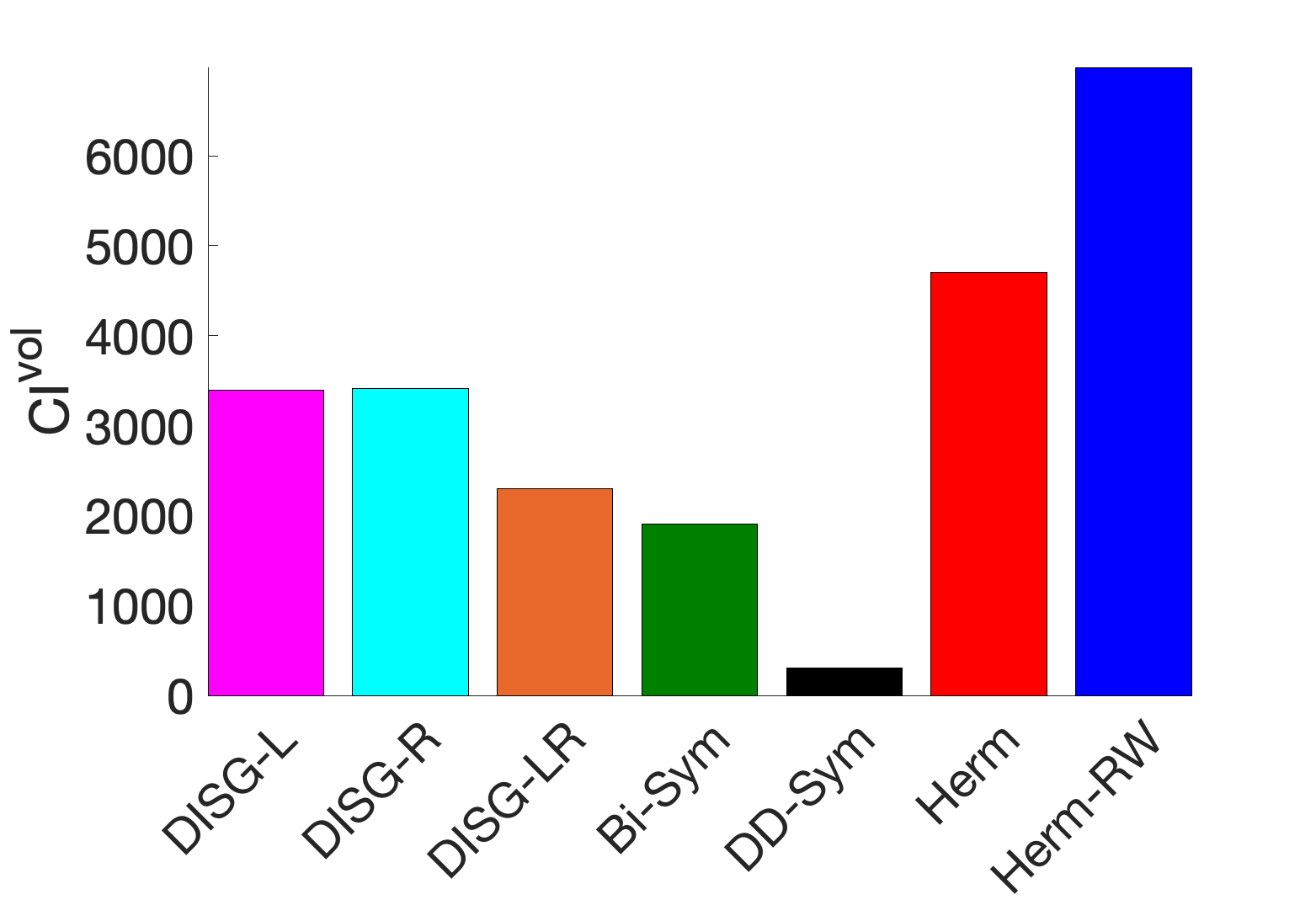} } %  \subcaption{$p = 0.01$}    \label{fig:5a}  % \par\vfill 
%  \hspace{-10mm}
\subcaptionbox{$k=3$}[0.24\columnwidth]{\includegraphics[width=0.26\columnwidth]{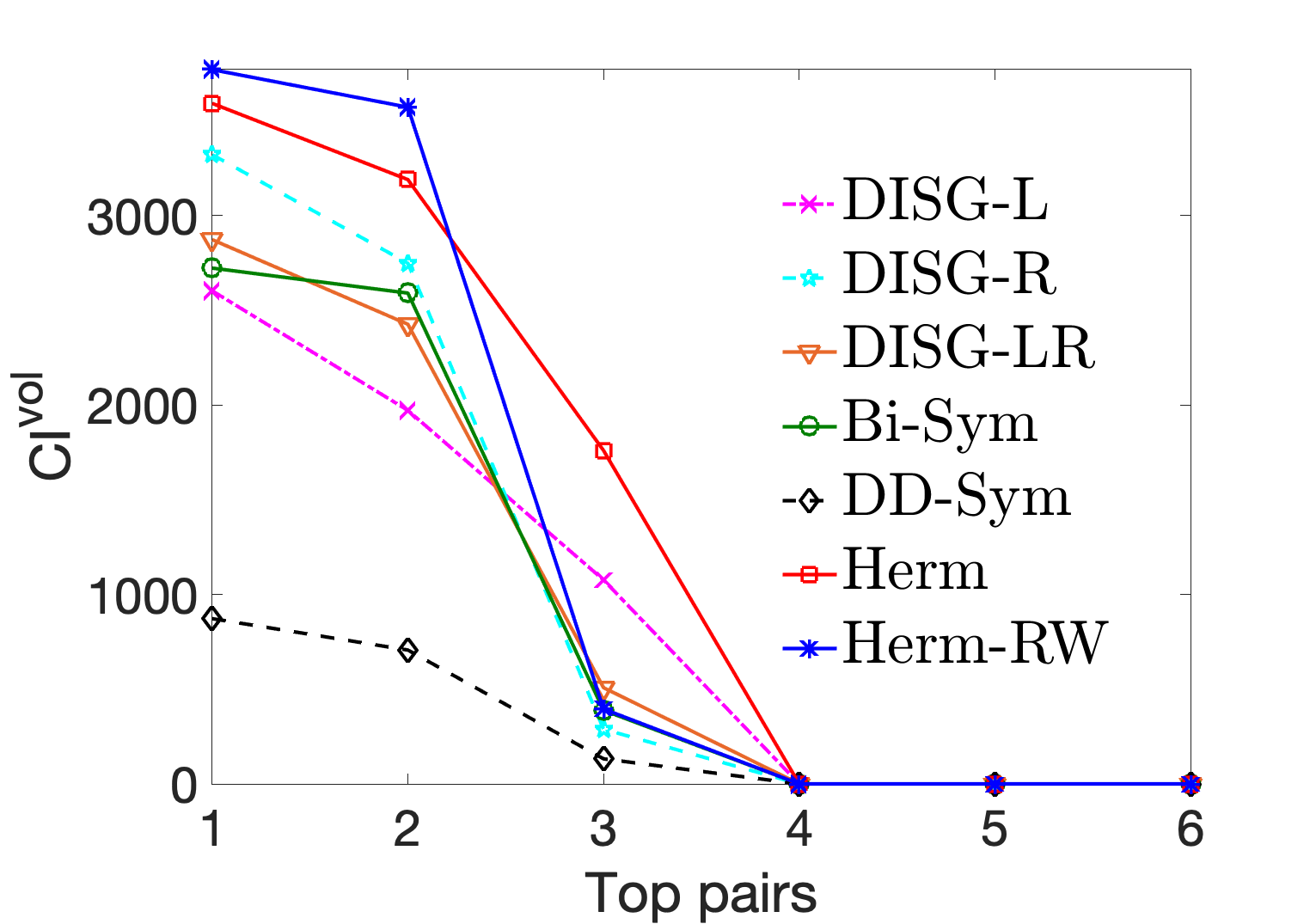} }
\subcaptionbox{$k=5$}[0.24\columnwidth]{\includegraphics[width=0.26\columnwidth]{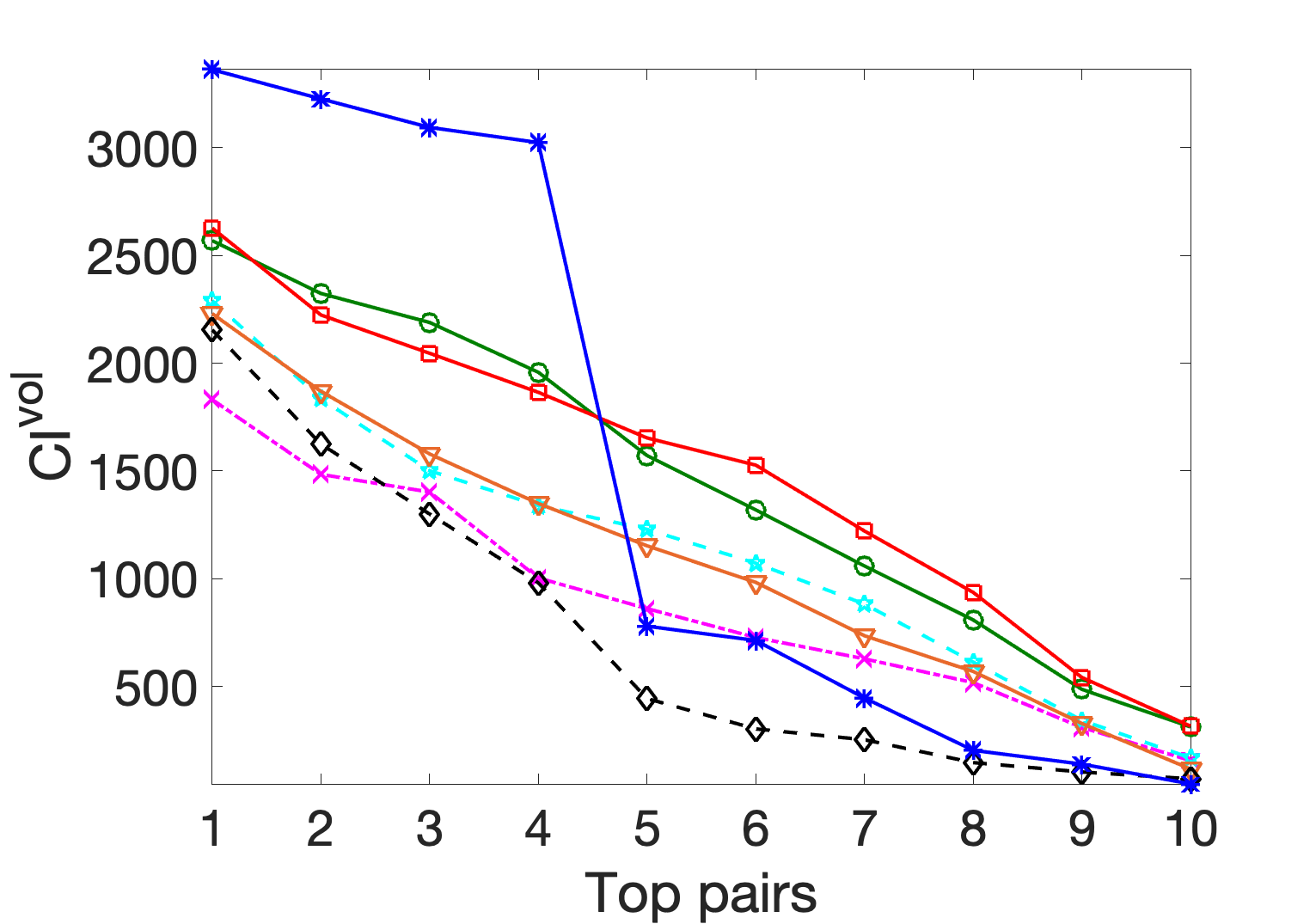} }
\subcaptionbox{$k=8$}[0.24\columnwidth]{\includegraphics[width=0.26\columnwidth]{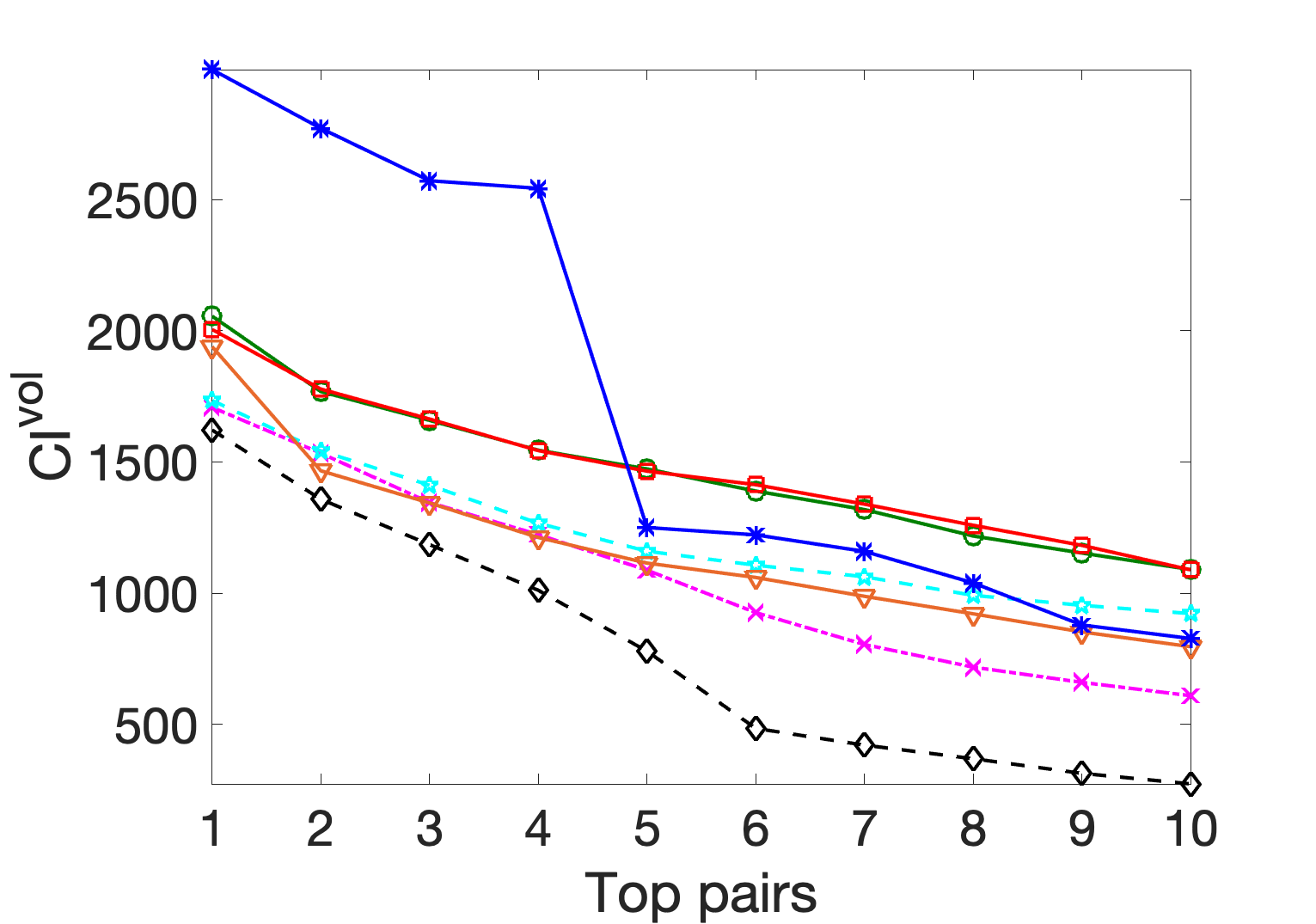} }
% \subcaption{$p = 0.02$} % \label{fig:5b}
% Note: if we turn on the \subcaption{} - the two figures end up one under the other...
\caption{\small The top $\civ$ scores attained by pairs of clusters, for the \textsc{BLOG} data set with varying $k$.  
% with $N=?$ and $k=\{2,3,5,8\}$ clusters (averaged over 20 runs). 
} 
\label{fig:scanID_9_blog}
\end{figure}

\vspace{0.3cm}

\underline{\emph{\textsc{BLOG} Network.}}
We consider the \textsc{BLOG} network from the 2004 US presidential election, as in 
Adamic and Glance \cite{adamic2005political}, who recorded the hyperlinks 
% connections among 
between 
$N=1,212$  political blogs and revealed that such connections were highly dependent on the blog's political orientation. 
% Since some vertices have zero in-degree or out-degrees in this data set, we regularise the adjacency matrix \textcolor{red}{$G$} by shifting it by a multiple of the identity.  
Figure \ref{fig:scanID_9_blog} shows the $\civ$ scores of the top pairs. % , for the largest connected component of the \textsc{BLOG} network, which contains  $N=1,212$ blogs. 
% We have included in the analysis a clustering with 
We also consider the case $k=2$, as    the  % \textsc{BLOG} 
network has an underlying  % cluster 
structure with two clusters corresponding to the Republican and Democratic parties. Overall, the two variants of our algorithm vastly outperform other methods, with \textsc{Herm-RW} as the best performer. 
% \textcolor{red}{Is it sufficient to only report the case of $k=2$? I feel this is already quite representative, and it will save us a lot space.} 
\vspace{-1mm}

\begin{figure}

\end{figure}

\begin{figure}   \end{figure}

%\textcolor{blue}{Since our goal is to uncover structures in the graph in terms of flow imbalance between pairs of subsets of vertices, we quantify this measure for the entire graph by considering all ${k \choose 2} $ pairs of clusters, and summing together the largest $2k$ $\civ$ values:  \begin{equation}     \text{Top}\civ = \sum_{t=1}^{2k} \civ({C}_{j_t},C_{{\ell}_t}) \end{equation}   where $({C}_{j_t},C_{\ell_t})$ denotes the $t$-th largest $\civ$ cut imbalance pair. We only consider the top $2k$ largest values since we do not expect for cut imbalances to manifest between all pairs of clusters (in other words, the meta-graph $F$ is typically also sparse). We also define an analogous measure for the cluster-size normalisation and we denote it by $\cistop$. }\textcolor{red}{We'll drop this paragraph completely. right?}

\vspace{-2mm}

\bibliographystyle{plain}
\bibliography{references}

\begin{thebibliography}{10}

\bibitem{adamic2005political}
Lada~A Adamic and Natalie Glance.
\newblock The political blogosphere and the 2004 {US} election: divided they
  blog.
\newblock In {\em Proceedings of the 3rd international workshop on Link
  discovery}, pages 36--43, 2005.

\bibitem{kmeanspp}
David Arthur and Sergei Vassilvitskii.
\newblock k-means++: the advantages of careful seeding.
\newblock In {\em Proceedings of the Eighteenth Annual {ACM-SIAM} Symposium on
  Discrete Algorithms, {SODA} 2007}, pages 1027--1035, 2007.

\bibitem{census}
U.~S. {Census Bureau}, 2002.
\newblock \url{www.census.gov/population/www.cen2000/ctytoctyflow/index.html}.

\bibitem{ChungRadcliffe}
Fan Chung and Mary Radcliffe.
\newblock On the spectra of general random graphs.
\newblock {\em Electronic Journal of Combinatorics}, 18(1), 2011.

\bibitem{syncRank}
M.~Cucuringu.
\newblock {Sync-Rank: Robust Ranking, Constrained Ranking and Rank Aggregation
  via Eigenvector and Semidefinite Programming Synchronization}.
\newblock {\em IEEE Transactions on Network Science and Engineering},
  3(1):58--79, 2016.

\bibitem{belgium2010}
M.~Cucuringu, V.~Blondel, and P.~Van~Dooren.
\newblock Extracting spatial information from networks with low order
  eigenvectors.
\newblock {\em Physical Review E}, 87, 2013.

\bibitem{asap2d}
M.~Cucuringu, Y.~Lipman, and A.~Singer.
\newblock Sensor network localization by eigenvector synchronization over the
  {E}uclidean group.
\newblock {\em ACM Transactions on Sensor Networks}, 8(3):19:1--19:42, 2012.

\bibitem{davisKahan}
C.~Davis and W.~M. Kahan.
\newblock The rotation of eigenvectors by a perturbation. {III}.
\newblock {\em SIAM Journal on Numerical Analysis}, 7:1--46, 1970.

\bibitem{UKmigdata}
Office for National~Statistics.
\newblock Internal migration: detailed estimates by origin and destination
  local authorities, age and sex, 2018.

\bibitem{ARI_JMLR_Gates_Ahn}
Alexander~J. Gates and Yong-Yeol Ahn.
\newblock The impact of random models on clustering similarity.
\newblock {\em Journal of Machine Learning Research}, 18(87):1--28, 2017.

\bibitem{GM17}
Krystal Guo and Bojan Mohar.
\newblock Hermitian adjacency matrix of digraphs and mixed graphs.
\newblock {\em Journal of Graph Theory}, 85(1):217--248, 2017.

\bibitem{sbm}
Paul~W. Holland, Kathryn~Blackmond Laskey, and Samuel Leinhardt.
\newblock Stochastic blockmodels: first steps.
\newblock {\em Social Networks}, 5(2):109--137, 1983.

\bibitem{kmeansalgo}
Amit Kumar, Yogish Sabharwal, and Sandeep Sen.
\newblock A simple linear time (1+$\varepsilon$)-approximation algorithm for
  $k$-means clustering in any dimensions.
\newblock In {\em Proceedings of the 45th Symposium on Foundations of Computer
  Science}, pages 454--462, 2004.

\bibitem{hoc}
James~R. Lee, Shayan~Oveis Gharan, and Luca Trevisan.
\newblock Multiway spectral partitioning and higher-order {C}heeger
  inequalities.
\newblock {\em Journal of the ACM}, 61(6), 2014.

\bibitem{leiRinaldo}
Jing Lei and Alessandro Rinaldo.
\newblock Consistency of spectral clustering in stochastic block models.
\newblock {\em The Annals of Statistics}, 43(1):215--237, 2015.

\bibitem{snapnets}
Jure Leskovec and Andrej Krevl.
\newblock {SNAP Datasets}: {Stanford} large network dataset collection.
\newblock \url{http://snap.stanford.edu/data}, June 2014.

\bibitem{MV13}
Fragkiskos~D. Malliaros and Michalis Vazirgiannis.
\newblock Clustering and community detection in directed networks: {A} survey.
\newblock {\em Physics Reports}, 533(4):95--142, 2013.

\bibitem{RankCentrality}
Sahand Negahban, Sewoong Oh, and Devavrat Shah.
\newblock Iterative ranking from pair-wise comparisons.
\newblock In {\em Advances in Neural Information Processing Systems 25}, pages
  2474--2482, 2012.

\bibitem{Ng}
A.~Y. Ng, M.~I. Jordan, and Y.~Weiss.
\newblock On spectral clustering: Analysis and an algorithm.
\newblock In {\em Advances in Neural Information Processing Systems}, pages
  849--856, 2001.

\bibitem{census_rep}
M.~J. Perry.
\newblock State-to-{S}tate {M}igration {F}lows: 1995 to 2000.
\newblock {\em Census 2000 Special Reports}, 2003.

\bibitem{rand1971}
W.M. Rand.
\newblock Objective criteria for the evaluation of clustering methods.
\newblock {\em Journal of the American Statistical Association},
  66(336):846--850, 1971.

\bibitem{co-clustering}
Karl Rohe, Tai Qin, and Bin Yu.
\newblock Co-clustering directed graphs to discover asymmetries and directional
  communities.
\newblock {\em Proceedings of the National Academy of Sciences},
  113(45):12679--12684, 2016.

\bibitem{SP11}
Venu Satuluri and Srinivasan Parthasarathy.
\newblock Symmetrizations for clustering directed graphs.
\newblock In {\em Proceedings of the 14th International Conference on Extending
  Database Technology}, pages 343--354, 2011.

\bibitem{ShiM00}
J.~Shi and J.~Malik.
\newblock Normalized cuts and image segmentation.
\newblock {\em {IEEE} Transactions on Pattern Analysis and Machine
  Intelligence}, 22(8):888--905, 2000.

\bibitem{SingerVDM}
A.~Singer and H.~T. Wu.
\newblock {Vector diffusion maps and the connection Laplacian}.
\newblock {\em Communications on Pure and Applied Mathematics}, 2012.

\bibitem{luxburg07}
Ulrike von Luxburg.
\newblock A tutorial on spectral clustering.
\newblock {\em Statistics and Computing}, 17(4):395--416, 2007.

\bibitem{white86a}
J.G. White, E.~Southgate, J.~N. Thomson, and S.~Brenner.
\newblock The structure of the nervous system of the nematode c. elegans.
\newblock {\em Philosophical transactions Royal Society London}, 314:1--340,
  1986.

\end{thebibliography}

\pagebreak

\appendix
%   \nolinenumbers

% \begin{center}
% %\large{\textbf{Supplementary material of submission  ID 1008  }}\\
% \large{\textbf{Supplementary material}}\\

% \vspace{0.5cm}

% \end{center}

%\setcounter{page}{1}

\section{Appendix}

In this appendix, we present a more detailed analysis of our algorithm and its performance on various additional  data sets. This appendix is structured as follows: Section~\ref{sec:aa} presents all the omitted proofs of the theorems and lemmas from Section~\ref{sec:analysis}; in Section~\ref{sec:appexp}, through additional experimental results, we give a detailed comparison between our algorithm and existing methods from  the literature. 

\subsection{Omitted proof details\label{sec:aa}}

In this section we present the omitted technical details about the analysis from Section~\ref{sec:analysis}. We first introduce some notation that will be used in the analysis. 
For any Hermitian matrix $A$ and parameters $\alpha\leq \beta$, let $P_{(\alpha,\beta)}(A)$  be the projection on the subspace spanned by the  eigenvectors of $A$ with the corresponding  eigenvalues in $(\alpha, \beta)$, and we define the matrix $P_{[\alpha,\beta]}(A)$ in a similar way. Notice that the matrix $P$ defined in Algorithm~\ref{alg:spectral} can be written as $P_{( -\infty, -\eps)\cup (\eps, +\infty) }(A)$.  
 
We now state two  results that will be used in the proofs below. The first is the well-known Davis-Kahan theorem, which bounds the perturbation of the eigenspaces of a matrix $H$ subject to random noise expressed by a matrix $R$. It will be used in the proof of \lemref{pqclose}.
\begin{thm}[Davis-Kahan, \cite{davisKahan}]
\label{thm:daviskahan}
Let $H,R \in \R^{d \times d}$ be Hermitian matrices. Then, for any $a \le \beta$ and $\delta > 0$ it holds that
\[
\left\|P_{[\alpha,\beta]}(H) - P_{(\alpha-\delta,\beta+\delta)}(H+R)\right\| \le \frac{\|R\|}{\delta}.
\]
\end{thm}

The other lemma that will be used in the analysis is the 
 following matrix concentration inequality.

\begin{thm}[\cite{ChungRadcliffe}]
\label{thm:matrixcher}
Let $X_1,X_2,\dots,X_m$ be independent random $d \times d$ Hermitian matrices. Moreover, assume that $\|X_j - \mathbb{E}{X_j}\| \le M$ for all $j$, and let $\sigma^2 = \|\sum_{j=1}^m \mathbb{E}{\left(X_j - \mathbb{E}{X_j}\right)^2}\|$.  Let $X = \sum_{j=1}^m X_j$. Then, for any $a>0$, it holds that 
\[
\mathbb{P}\left[\|X - \mathbb{E}{X}\| > a\right] \le 2d \exp\left( - \frac{a^2}{2\sigma^2 + 2Ma/3} \right).
\]
\end{thm}

We can now present the omitted proofs from Section \ref{sec:analysis}.

\begin{proof}[Proof of Proposition~\ref{prop2}]
First of all, we assume that $F$ is the matrix with two identical rows indexed by $j$ and $\ell$, and we prove that $\widetilde{F}$ has a nondistinguishing image.  To this end, notice that 
$
\widetilde{F}_{j,j} = \widetilde{F}_{\ell,\ell} = \widetilde{F}_{j,\ell} =\widetilde{F}_{\ell,j} = 0,
$
and there is an automorphism that swaps $j$ and $\ell$ such that the remaining rows look like the same. This implies that $
P_{\Im(\widetilde{F})}(j,\cdot) = P_{\Im(\widetilde{F})}(\ell,\cdot),
$
which proves the claim.

Secondly, we prove the other direction. Assume that  $P_{\Im(\widetilde{F})}(j,\cdot) = P_{\Im(\widetilde{F})}(\ell,\cdot)$ with  $0 \le j \ne \ell \le k-1$, and consider the vector $\chi \in \{-1,0,1\}^k$
which is $1$ in the $j$th entry, $-1$ in the $\ell$th entry, and zero otherwise. It is easy to check that  $P_{\Im(\widetilde{F})} \chi = \mathbf{0}$. This means that $\chi \in \ker(\widetilde{F})$ and $\widetilde{F} \chi = 0$, which implies that the columns
of $\widetilde{F}$ indexed by $j$ and $\ell$, as well as the corresponding rows are equal~(since $\widetilde{F}$ is Hermitian). Hence, the corresponding rows of $F$ must be equal. \qedhere
\end{proof}

We now devote our attention to prove \thmref{main}. The following lemma shows that the Hermitian adjacency matrix of a random graph generated from the  DSBM is concentrated around its expectation.
\begin{lem}
\label{lem:concentration}
Let $G \sim \mathcal{G}\left(k,n,p,q,F\right)$ with  $p=q$. Then, with high probability, we have that $\|A - \mathbb{E}{A}\| \le 10 \sqrt{pkn \log{n}}$.
\end{lem}
% Using \thmref{daviskahan} we will bound the distance between the image of $\mathbb{E}{A}$ and the top eigenspaces of $A$, which in turn depends on the spectral norm of $A-\mathbb{E}{A}$.  
% We now apply this theorem to bound $\|A-\mathbb{E}{A}\|$.
\begin{proof}
Let $M^{uv}\in\mathbb{C}^{N\times N}$ be the matrix 
with exactly two non-zero entries defined by 
$
\left(M^{uv} \right)_{u,v}=1$, $\left( M^{uv}\right)_{v,u} =-1. $
By definition, $(M^{uv})^2$ has exactly two nonzero entries, i.e., 
\begin{equation}
\label{eq:nonzeromsq}
\left(M^{uv}\right)^2_{u,u} = \left(M^{uv}\right)^2_{v,v} = -1.
\end{equation}
Let $X^{uv}$ be a random matrix defined by
\[
X^{uv} =  \begin{cases}
				i\cdot M^{uv} 	& \text{ if } \, u \leadsto v \\
				-i\cdot M^{uv}  	& \text{ if } \, v \leadsto u \\
				0 		& \text{ otherwise.}\\
			\end{cases}
\]
Observe that $\sum_{\{u,v\}} X^{uv} = A$, the adjacency matrix of $G$. 

Let $u,v \in V$ be a pair of vertices such that $u \in C_j$ and $v \in C_{\ell}$. Then, we have 
\begin{align*}
\mathbb{E}{X^{uv}} & = p \ \left( F_{j,\ell} M^{u,v}\cdot i  + F_{\ell,j} M^{v,u}\cdot i  \right)\\
& = p \ \left( F_{j,\ell} M^{u,v}\cdot i  - (1- F_{j,\ell}) M^{u,v}\cdot i  \right)\\
& = p \  (2 F_{j,\ell} -1 )M^{u,v}\cdot i\\
& = p \ \widetilde{F}_{j,\ell} M^{uv},
\end{align*}
and
% \begin{align*}
% \mathbb{E}{\left(X^{uv} - \mathbb{E}{X^{uv}}\right)^2}& = (1-p)\left(-\mathbb{E}{X^{uv}}\right)^2 + p F_{j,\ell} \left(M^{uv}\cdot i -\mathbb{E}{X^{uv}}\right)^2\\
% &  \qquad + p F_{\ell,j} \left(-M^{uv}\cdot i-\mathbb{E}{X^{uv}}\right)^2\\
% &= (1-p)\left(-\mathbb{E}{X^{uv}}\right)^2 + p \left(M^{uv}\cdot i -\mathbb{E}{X^{uv}}\right)^2 \\
% & = p\left(M^{uv}\right)^2\left((1-p)p \left(\widetilde{F}_{j,\ell}\right)^2  
%     + \left(i-p\widetilde{F}_{j,\ell}\right)^2\right).
% \end{align*}
    \begin{align*}
        \mathbb{E}{\left(X^{uv} - \mathbb{E}{X^{uv}}\right)^2}
        % &=
        % \mathbb{E}\left( \left(X^{uv}\right)^2
        % - X^{uv}\left(\mathbb{E}{X^{uv}}\right) - \left(\mathbb{E}{X^{uv}}\right) X^{uv}
        % + \left(\mathbb{E}{X^{uv}}\right)^2 \right) \\
        &=
        \mathbb{E} \left(X^{uv}\right)^2 - \left(\mathbb{E}{X^{uv}}\right)^2 \\
        &=
        -p \left(M^{uv}\right)^2 - p^2 \left(\widetilde{F}_{j,\ell}\right)^2 \left(M^{uv}\right)^2 \\
        &= \left(-p \left(\widetilde{F}_{j,\ell}\right)^2 - 1\right) p \left(M^{uv}\right)^2.
    \end{align*}
Moreover, $\left|-p \left(\widetilde{F}_{j,\ell}\right)^2 - 1\right| \leq 2$.
%Notice that $\left|(1-p)p \left(\widetilde{F}_{j,\ell}\right)^2  + \left(i-p\widetilde{F}_{j,\ell}\right)^2\right| \le 5$. 
    Therefore, since the spectral norm of a matrix is upper bounded by the sum of the absolute values of the entries in each row, by \eq{nonzeromsq} it holds that
  $
\left\|\sum_{u,v \in V} \mathbb{E}{\left(X^{uv} - \mathbb{E}{X^{uv}}\right)^2}\right\| \le 2pkn$.
Setting $a = 10 \sqrt{pkn \log{n}}$, $M = 1$, $\sigma^2 \le pkn$ and $d = kn$, we apply \thmref{matrixcher} to obtain the statement.
\end{proof}

We now combine \lemref{concentration} and \thmref{daviskahan} to bound how far the matrix $P$ computed by Algorithm~\ref{alg:spectral} is to the projection on the image of $\mathbb{E}A$. 
\begin{lem}
\label{lem:pqclose}
Let $G \sim \mathcal{G}\left(k,n,p,q,F\right)$ with  $p=q$. Let $Q$ be the projection on the image of $\mathbb{E}{A}$, i.e., $Q = P_{\Im(\mathbb{E}{A})}$ and $P$ as in Algorithm~\ref{alg:spectral}. Moreover, set the parameter $\eps$ of Algorithm~\ref{alg:spectral} to $\eps = 20 \sqrt{pkn \log{n}}$ and assume \eqref{eq:assumption} holds. Then, it holds with high probability that
\[
\|P - Q\| = O\left(\frac{\sqrt{k \log{n}}}{\widetilde{\rho} \sqrt{pn}} \right).
\]
\end{lem}
\begin{proof}
Let $(\lambda_1,g_1),\ldots, (\lambda_{\ell}, g_{\ell})$ be the pairs of the eigenvalues and eigenvectors computed by Algorithm~\ref{alg:spectral}.  Then, by \lemref{concentration} it holds for any $1\leq j\leq \ell$ that 
$ |\lambda_j |
\geq \widetilde{\rho}pn -\eps$ .
Notice that the other eigenvalues of $A$ have absolute value less than $\eps$.
Therefore, based on assumption
(\ref{eq:assumption}), and the relationship between the eigenvalues of $\widetilde{F}$ and $\mathbb{E}A$, 
we apply
 \thmref{daviskahan} and obtain
\[
\|P - Q\| \le \frac{\|A - \mathbb{E}{A}\|}{\widetilde{\rho}pn - 2\eps} = O\left(\frac{\sqrt{k \log{n}}}{\widetilde{\rho} \sqrt{pn}}\right). \qedhere
\] 
\end{proof}

We are now ready to prove the main theorem, which gives
an upper bound on the number of  vertices misclassified by Algorithm~\ref{alg:spectral}. More precisely, given a graph $G=(V,E)$ with clusters $C_0,\dots,C_{k-1} \subset V$ and a partition $A_0,\dots,A_{k-1}$ of $V$, the number of misclassified vertices is defined as 
\vspace{-1mm}
\[
\mathcal{M} = \min_{\sigma \in S_k} \sum_{j=0}^{k-1} \left(|A_{\sigma(j)} \setminus C_j| + |C_j \setminus A_{\sigma(j)}|\right),
\]
\vspace{-3mm}

\noindent where $S_k$ is the symmetric group on $[k]$. We also assume that the $k$-means algorithm used in Algorithm~\ref{alg:spectral} achieves a constant approximation ratio (e.g., \cite{kmeansalgo}). Now we are ready to prove the main result of the submission.

% \vspace{-1mm}
% \begin{thm}[Main Theorem]
% \label{thm:main}
% Let $G \sim \mathcal{G}\left(k,\{n_j\}_{j=0}^{k-1},p,q, F \right)$ with $n_i=n$ for all $0 \le i \le k-1$ and $p=q$. Assume  \eq{assumption} holds and  $\widetilde{F}$ has a $\theta$-distinguishing image with $\theta>0$. Then, the number of misclassified vertices by Algorithm~\ref{alg:spectral} is $O\left( \frac{k^2 \log{n}}{\widehat{\rho}^2 \theta^2 p}\right)$ with high probability.
% \end{thm}
 
\begin{proof}[Proof of \thmref{main}]
Let $Q=P_{\Im(\mathbb{E}{A})}$ and $P$ as in Algorithm~\ref{alg:spectral}. Observe that $Q$ is a block matrix with the following properties: rows corresponding to vertices belonging to the same cluster are equal, while the distance between rows corresponding to different clusters is at least $\theta$. 
For any cluster $C_j$, let $c_j$ be the row of $Q$ corresponding to any vertex in $C_j$ (they are all equal). Let $\overline{c_j}$ be the average of the rows of $P$ corresponding to $C_j$. By \lemref{pqclose} we know that $\|c_j - \overline{c_j}\| = O\left( \frac{\sqrt{k \log{n}}}{\widetilde{\rho} \sqrt{pn}}\right) $, which implies, for any $\ell \ne j$, 
\begin{equation}
\label{eq:distance}
\|\overline{c_j} - \overline{c_\ell}\| \ge \theta - \frac{20\sqrt{k \log{n}}}{\widetilde{\rho} \sqrt{pn}} = \theta/2,
\end{equation}
where the second inequality follows from assumption (\ref{eq:assumption}). Moreover, the optimal $k$-means cost is at most 
\begin{equation}
\label{eq:kmeanscost}
\sum_{j=0}^{k-1} \sum_{u \in C_j} \|P(u,\cdot) - c_j \|^2 \le \tr(P-Q)^2 \le \|P - Q\|^2 \cdot kn = O\left( \frac{k^2 \log{n}}{\widetilde{\rho}^2 p}\right)
\end{equation}
where the last equality follows from \lemref{pqclose}.

Let $c_0^{\star}, \dots, c_{k-1}^{\star}$ be the optimal centroids of a $k$-means clustering on the rows of $P$. For any $\ell \ne j$, we claim that $\|c_j^{\star} - c_{\ell}^{\star}\| \ge \theta/4$. Assume this  isn't true. By \eq{distance}, then, there must exist a $\overline{c_{\ell}}$ which is at least $\theta/4$ far from any point $c_j^{\star}$. We now show this implies that the optimal $k$-means cost is large, contradicting \eq{kmeanscost}. Let $c^{\star}(u)$ be the centroid $c_j^{\star}$ which is closest to $P(u,\cdot)$. Then, by the triangle inequality and the trivial inequality $(x-y)^2 \ge x^2/2 - y^2$, the optimal cost is lower bounded by
\begin{align*}
 \sum_{u \in C_{\ell}} \|P(u,\cdot) - c^{\star}(u) \|^2
 &\ge \sum_{u \in C_{\ell}} \left(\|\overline{c_{\ell}} - c^{\star}(u) \| - \|P(u,\cdot) - \overline{c_{\ell}} \|\right)^2 \\
 &\ge  \sum_{u \in C_{\ell}}  \left(\frac{1}{2}\|\overline{c_{\ell}} - c^{\star}(u) \|^2 - \|P(u,\cdot) - \overline{c_{\ell}} \|^2\right) \\
 &\ge \frac{n \theta^2}{32} - O\left( \frac{k^2 \log{n}}{\widetilde{\rho}^2 p}\right),
\end{align*}
which, by assumption (\ref{eq:assumption}), contradicts the fact that the optimal $k$-means cost is upper bounded by \eq{kmeanscost}.
Therefore, it holds that $\|c_j^{\star} - c_{\ell}^{\star}\| \ge \theta/4$ for any $\ell \ne j$. Hence, every time we misclassify a vertex we pay a cost of $\Omega(\theta^2)$. Because of this, any constant factor approximation algorithm for $k$-means will misclassify at most $O\left( \frac{k^2 \log{n}}{\widetilde{\rho}^2 \theta^2 p}\right)$ vertices.
\end{proof}

\begin{proof}[Proof of Corollary \ref{cor:cyclic}]
We start investigating the matrix $\widetilde{F} = \left(2
F- \mathbf{1}_{k \times k}\right)\cdot i$, which, in cyclic block models, can be rewritten as follows: $\widetilde{F}_{j,\ell} = (1-2\eta)\cdot i$  if $j \equiv \ell - 1 \mod k$, $\widetilde{F}_{j,\ell} = -(1-2\eta)\cdot i$ if $j \equiv \ell + 1 \mod k$, and $\widetilde{F}_{j,\ell} = 0$ otherwise. Therefore, $\widetilde{F}$ is a circulant matrix. From the theory of circulant matrices, we can deduce that $\widetilde{F}$ has a set of $k$ orthonormal eigenvectors $\widetilde{f}_0,\dots,\widetilde{f}_{k-1}$, such that, for any $0\le j,\ell\le k-1$, $\widetilde{f}_j(\ell) = \omega_k^{j\ell}k^{-1/2}$,
where $\omega_k$ is the $k$-th root of unity.  Let $\rho_0,\dots,\rho_{k-1}$ be the eigenvalues of $\widetilde{F}$ ordered so that $\rho_j$ is the eigenvalue corresponding to $\widetilde{f}_j$. It holds that
\vspace{-1mm}
\begin{align}
\rho_j &= (1-2\eta) \left(\omega_k^j -  \overline{\omega_k}^j \right)\cdot i = - 2\sin(2\pi j/k) (1-2\eta), \label{eq:lambdai}
\end{align}
where the second equality holds because the difference between a complex number $c$ and its conjugate is equal to twice the imaginary part of $c$.

From this we can easily obtain a bound on the spectral gap $\widetilde{\rho}$:
\vspace{-1mm}
\begin{equation}
\nonumber
\widetilde{\rho} = \min_{j\in[k]\setminus\{0,k/2\}} 2(1-2\eta)\cdot|\sin(2\pi j/k)|
                    = \Theta\left(\frac{1-2\eta}{k}\right).
\end{equation}
\vspace{-3mm}

From \eq{lambdai} we know the kernel of $\widetilde{F}$ is spanned by $\widetilde{f}_0$ and, if $k$ is even, by $\widetilde{f}_{k/2}$.  In both cases, however, $\widetilde{F}$ has a $\Omega(1)$-distinguishing image. Therefore,
the assumption of \eq{assumption} holds whenever 
$p = \omega\left(\frac{k^3 \log{n}}{(1-2\eta)^2n}\right)$. We can apply \thmref{main} to deduce that  the number of misclassified vertices is $O\left(\frac{k^4 \log{n}}{(1-2\eta)^2p} \right)$ with high probability.
\end{proof}

%%%%%%%%%%%%==============================================================================================
%%%%%%%%%%%%==============================================================================================
\subsection{Additional experimental results\label{sec:appexp}}

This section presents more detailed comparison on the performance of our algorithm with other spectral clustering algorithms for digraphs on both synthetic and real-world data sets.
All of  our experiments are performed in Matlab R2017b, on a MacBook Pro, with 2.8 GHz Intel Core i7 and 16 GB of memory. The spectral clustering algorithms are implemented using the Matlab function \texttt{eigs} to compute  eigenvectors, and the $k$-means++ algorithm~\cite{kmeanspp}.% to perform $k$-means clustering.
 
%  A_{\mathrm{rw}}

\noindent \textbf{More detailed experimental results for the DSBM.} In Figure~\ref{fig:DiSBM_AdjSpectrum} we consider two instances of graphs generated from the \textsc{DSBM} 
with $k=5$ clusters, where each cluster is of size $n=100$, $p=50\%$, and  noise parameter $\eta=0.15$. The figures at the top concern a cyclic block model, while the figures at the bottom  a randomly oriented complete meta-graph. We report the heatmap of the Hermitian adjacency matrices, the spectrum of $A_{\mathrm{rw}}$ leveraged by \textsc{Herm-RW}, as well as the final recovered cluster-structure with colours representing the $\cip$ score.  From Figures~\ref{fig:DiSBM_AdjSpectrum} (a) and (c) (resp. (d) and (f)), we can clearly see the cyclic (resp. complete) pattern between clusters. Moreover, Figures~\ref{fig:DiSBM_AdjSpectrum} (b) and (e) show that the bulk of the eigenvalues of $A_{\mathrm{rw}}$ is concentrated around $0$, with exactly $4$ outliers with larger absolute value: $4$ corresponds to the rank of the matrix $\tilde{F}$ of the corresponding DSBM. Also notice that the eigenvalues outside these outliers are more concentrated in the case of the block model with complete meta-graph. This is not a surprise since in the latter we have a noise level of $0.15$ between any pair of clusters, while in the cyclic block model we have noise level of $0.15$ between $k$ pairs, and of $0.5$ (corresponding to completely random orientations of the edges) between the remaining pairs.

Figure \ref{fig:largeclusters} shows the recovery rate of spectral clustering algorithms for DSBM with a randomly oriented complete meta-graph with $k=50$  clusters, each of size $n=100$. In this regime with a very large number of clusters, our proposed methods perform drastically better than competing approaches. For edge density $p=1\%$ only our approaches are able to achieve a meaningful ARI value, at least for low level of noise $\eta$ (recall that when $\eta=0$ noise due to intra-cluster edges is still present). When $p=2\%$, other methods perform reasonably well up to a noise level of $\eta=0.1$. Our method, instead, is able to achieve very good accuracy up to $\eta=0.15$, and non-trivial accuracy up to $\eta=0.2$.

Figure~\ref{fig:scanID_2a_b} (a) and (b) is a comparison on the DSBM model with $N=10,000$ and $k=20$ clusters, for both a complete meta-graph and a cyclic block model. This is the largest graph we have experimented with, and it shows that our Hermitian-based algorithms vastly outperform the competing methods, especially in the case of the cyclic block model where   there are less pairwise interactions between the clusters. Note that in Figure~\ref{fig:scanID_2a_b} (b) we left out \textsc{Bi-Sym} and  \textsc{DD-Sym} from the comparison, due to their computational cost.  It is also easy to see that 
our algorithms not only significantly outperform all the other tested methods, but  also run significantly  faster than \textsc{Bi-Sym} and \textsc{DD-Sym} which involve  matrix multiplication operations. For instance, \figref{scanID_2a_b} (c) compares the runtime of all algorithms on graphs randomly generated from the DSBM for $N=10,000$, $k=20, p=0.4\%$, and different $\eta$ values, and this quantitative comparison holds for different choices of parameters in general.
 
% , which also makes the case these two approaches do not scale as well

\iffalse 
\begin{wrapfigure}{o}{0.35\textwidth}
\captionsetup[subfigure]{skip=0pt}
% \vspace{-4mm}
% \hspace{-6mm}
\begin{centering}
\includegraphics[width=0.33\columnwidth]{Figures/scanID_2a_Kk/scanID_2a_n500_k20_p0p004_Kk_nrReps10_time.png}
\end{centering}
% \vspace{-1mm}
\captionsetup{width=0.99\linewidth}
\caption{\small Runtime analysis}
\label{fig:scanID_2a_RunningTime}
% \vspace{-1mm}
\end{wrapfigure}
\fi 
% as the rest of the methods. 
\begin{figure}
\captionsetup[subfigure]{skip=0pt}
\begin{centering}
\subcaptionbox{ $G$ }[0.3\columnwidth]{\includegraphics[width=0.25\columnwidth]{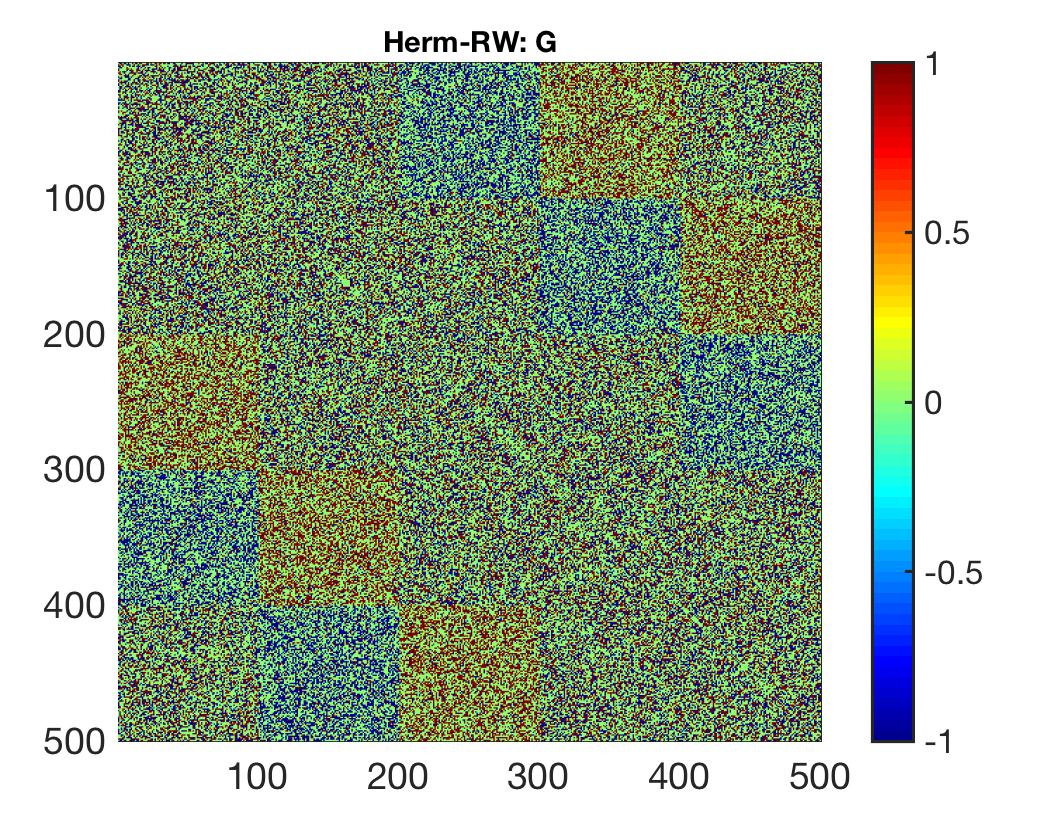}}  \hspace{-2mm} 
\subcaptionbox{ Spectrum of $ A_{\mathrm{rw}}$ }[0.3\columnwidth]{\includegraphics[width=0.25\columnwidth]{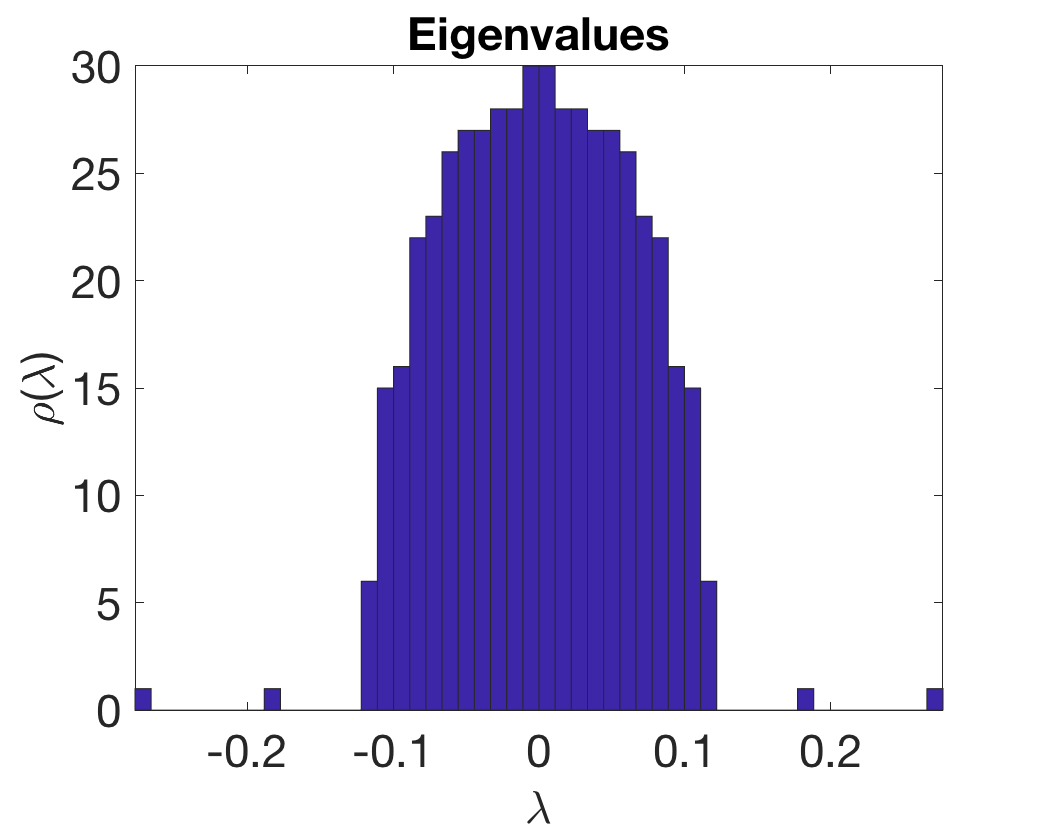}}  \hspace{-2mm} 
\subcaptionbox{$\cip$ matrix}[0.3\columnwidth]{\includegraphics[width=0.25\columnwidth]{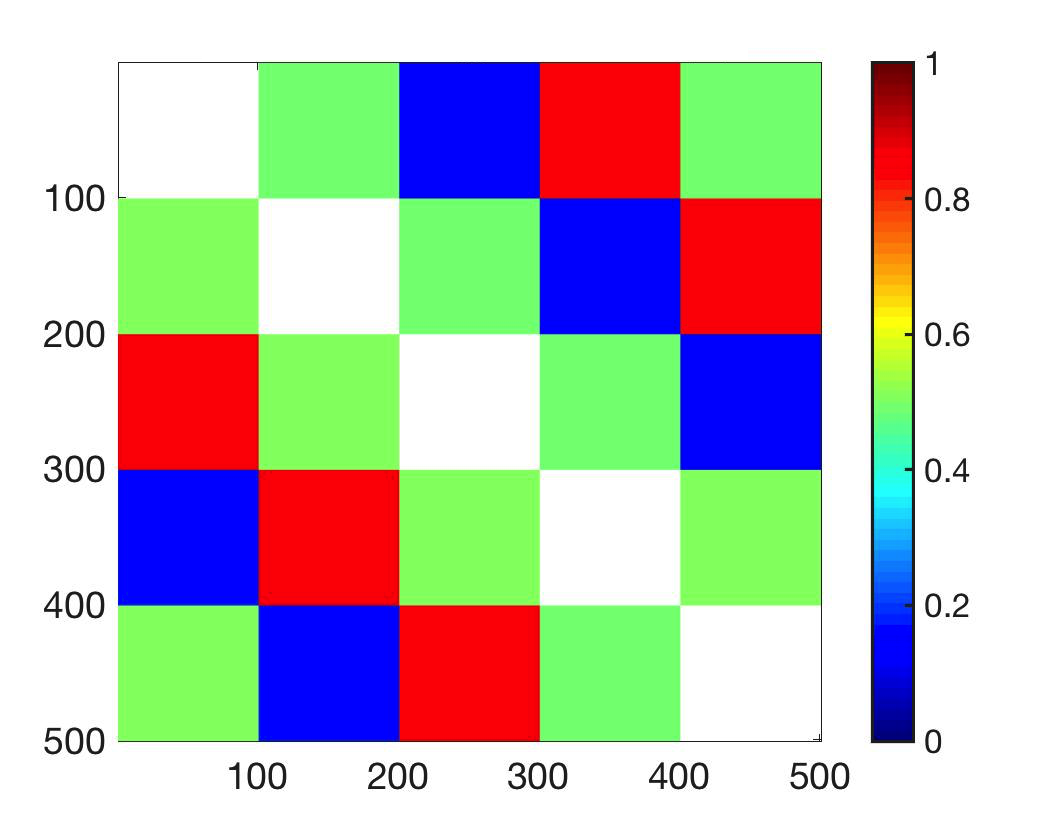}}
\end{centering}
\begin{centering}
\subcaptionbox{ $G$ }[0.3\columnwidth]{\includegraphics[width=0.25\columnwidth]{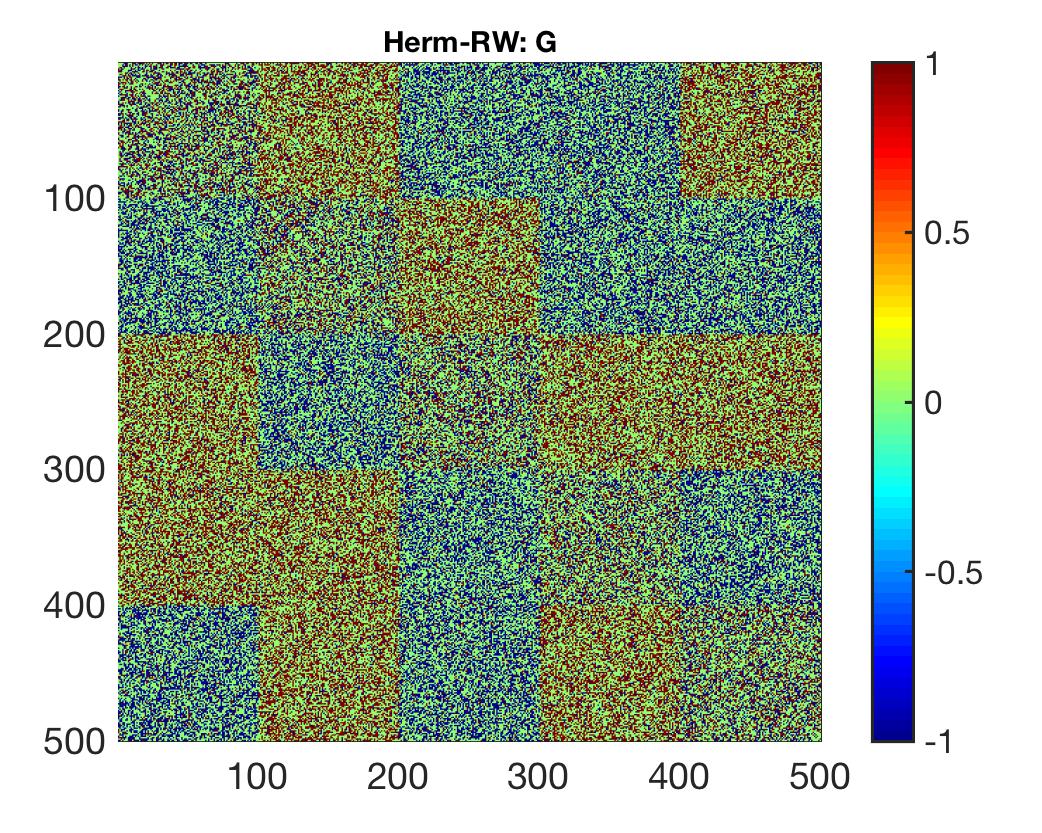}}
\hspace{5mm} 
\subcaptionbox{  Spectrum of $ A_{\mathrm{rw}}$ }[0.3\columnwidth]{\includegraphics[width=0.25\columnwidth]{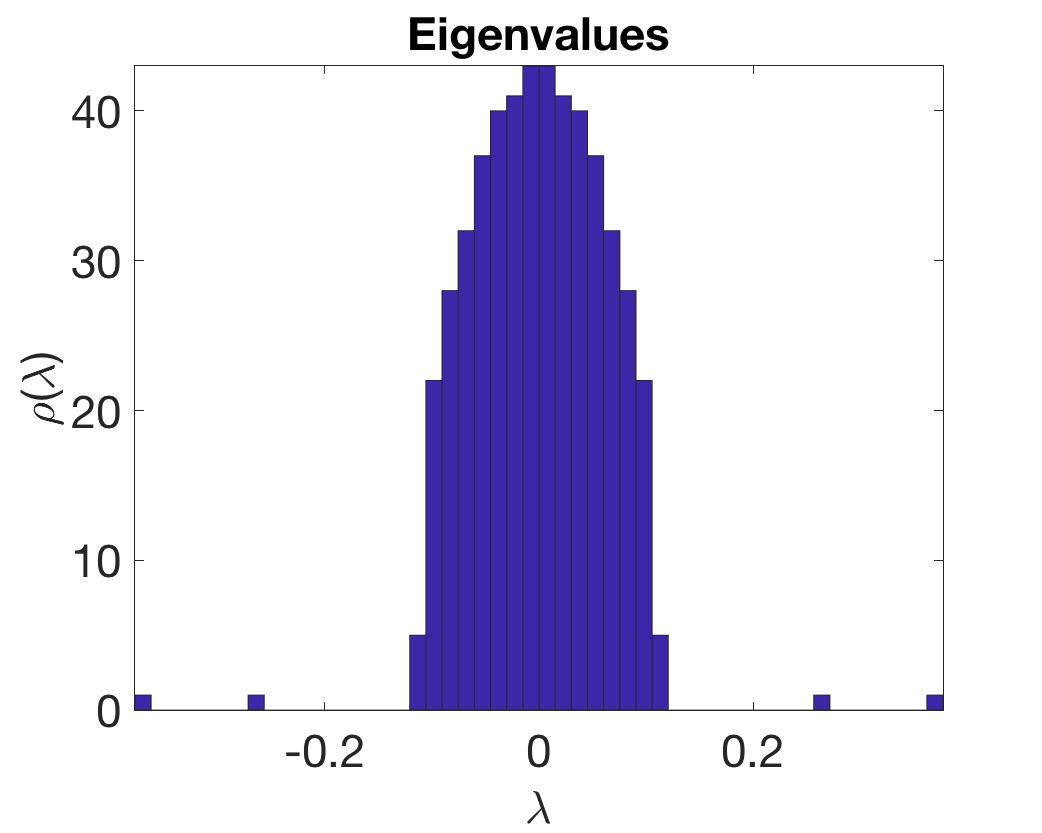}}
\hspace{5mm} 
\subcaptionbox{$\cip$ matrix}[0.3\columnwidth]{\includegraphics[width=0.25\columnwidth]{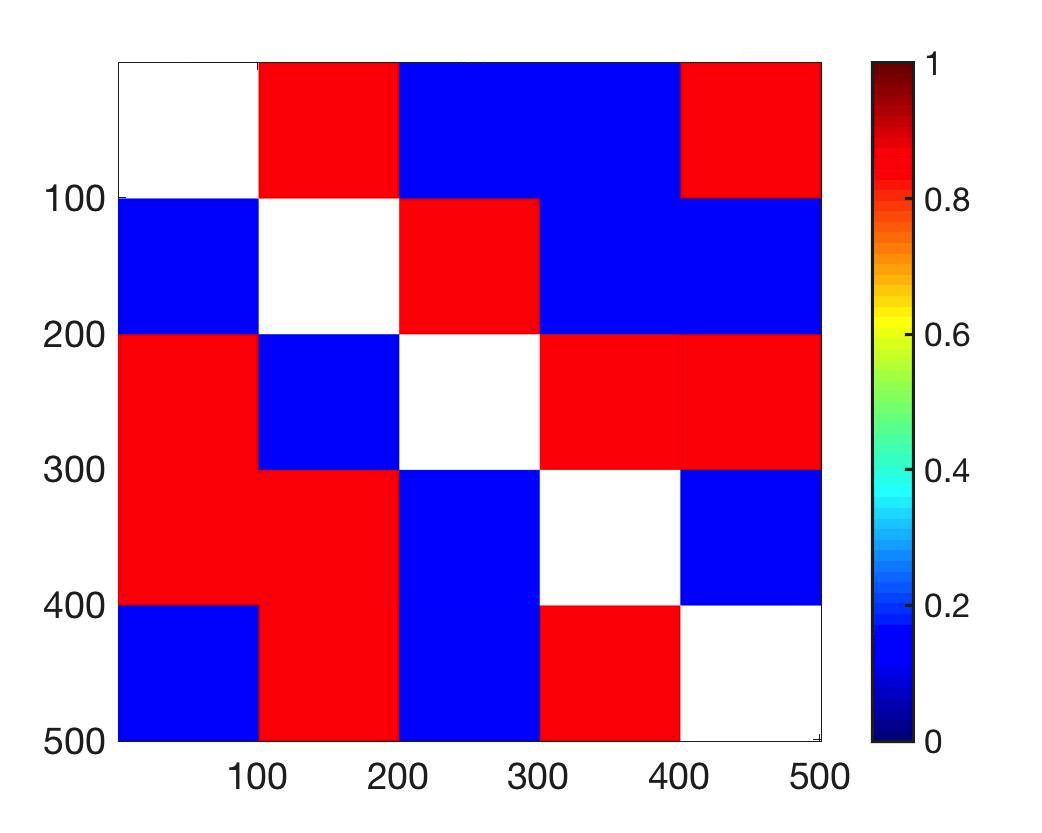}}
\end{centering}
\captionsetup{width=0.99\linewidth}
\caption{\small Recovery of an instance of the \textsc{DSBM} model, $N=500$, $p=50\%$, $\eta=0.15$ and $k=5$ clusters, for a cyclic block model (top) and a randomly oriented complete  meta-graph (bottom).}
\label{fig:DiSBM_AdjSpectrum}	
\end{figure}
\vspace{0mm}

% \todo{Figure 9: replace $H_{RW}$ by $A_{RW}$? }
% Fixed

% \textcolor{orange}{replace word ``IF'' with ``CI'' at the top of Figure 9 (c) (f)?} \todo{M: hm, at this point it's easiest to crop the image and simply get rid of the title in those two figures - it should not be there to begin with.}
\begin{figure}
\centering 
\captionsetup[subfigure]{skip=2pt}
\subcaptionbox{$p=1\%$}[0.49\columnwidth]{\includegraphics[width=0.35\columnwidth]{Figures/scanID_1a_Kk/scanID_1a_n100_k50_p0p01_Kk_nrReps10_ARI.png} } %  \subcaption{$p = 0.01$}    \label{fig:5a}  % \par\vfill 
%  \hspace{-10mm}
\subcaptionbox{$p=2\%$}[0.49\columnwidth]{\includegraphics[width=0.35\columnwidth]{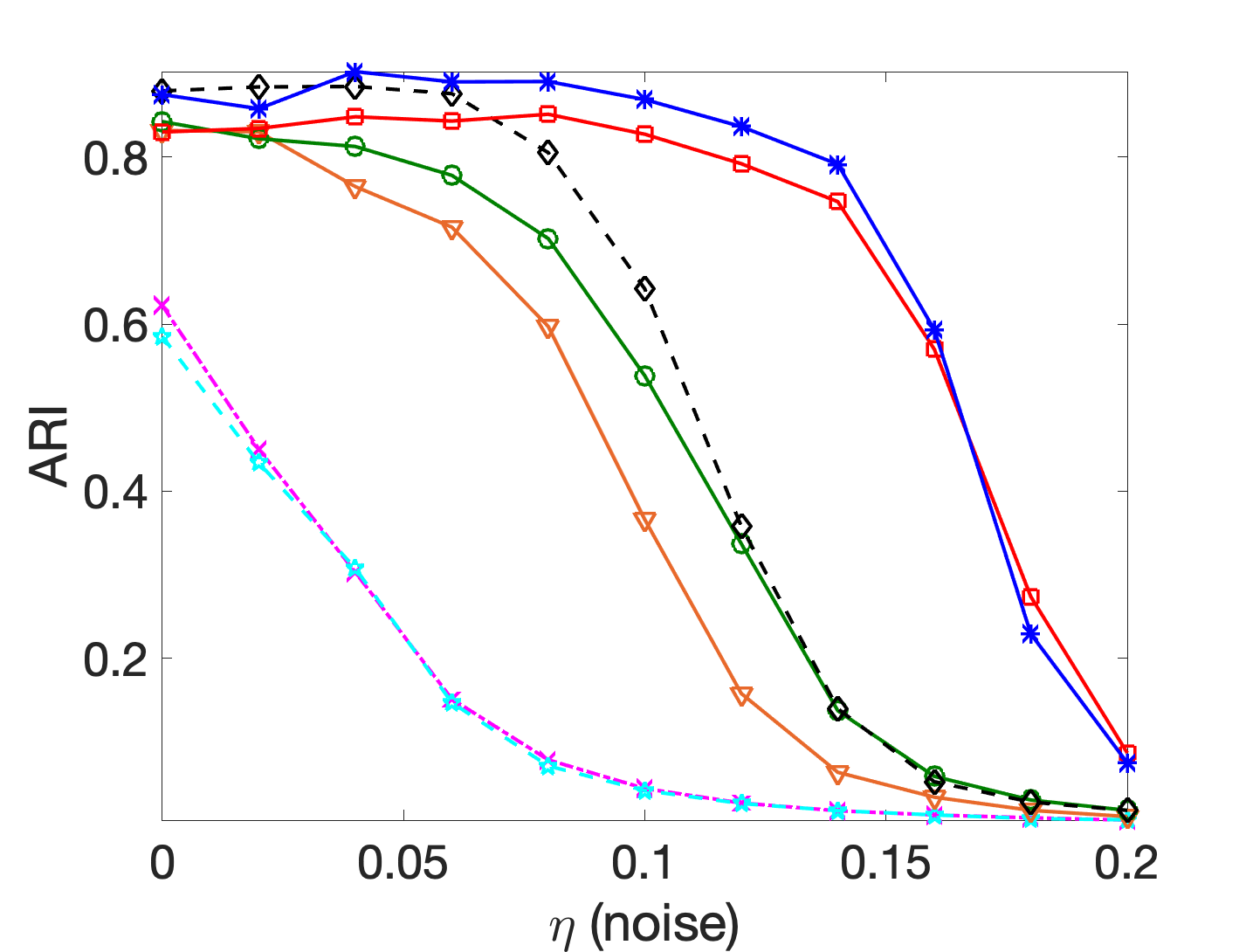} }
% \subcaption{$p = 0.02$} % \label{fig:5b}
% Note: if we turn on the \subcaption{} - the two figures end up one under the other...
\caption{Recovery rates for the complete meta-graph in the \textsc{DSBM} with  $k=50$ and $n=100$. }
\label{fig:largeclusters}%\textcolor{red}{Is (a) the same as Figure 5? Yes, we should perhaps remove it, and then decide if we want to show (b) or no.}
\end{figure}

\begin{figure}[h!]
\captionsetup[subfigure]{skip=0pt}
\begin{centering}\hspace{2mm}
\subcaptionbox{Recovery rates for the complete meta-graph $p=0.4\%$. }[0.3\columnwidth]{\includegraphics[width=0.3\columnwidth]{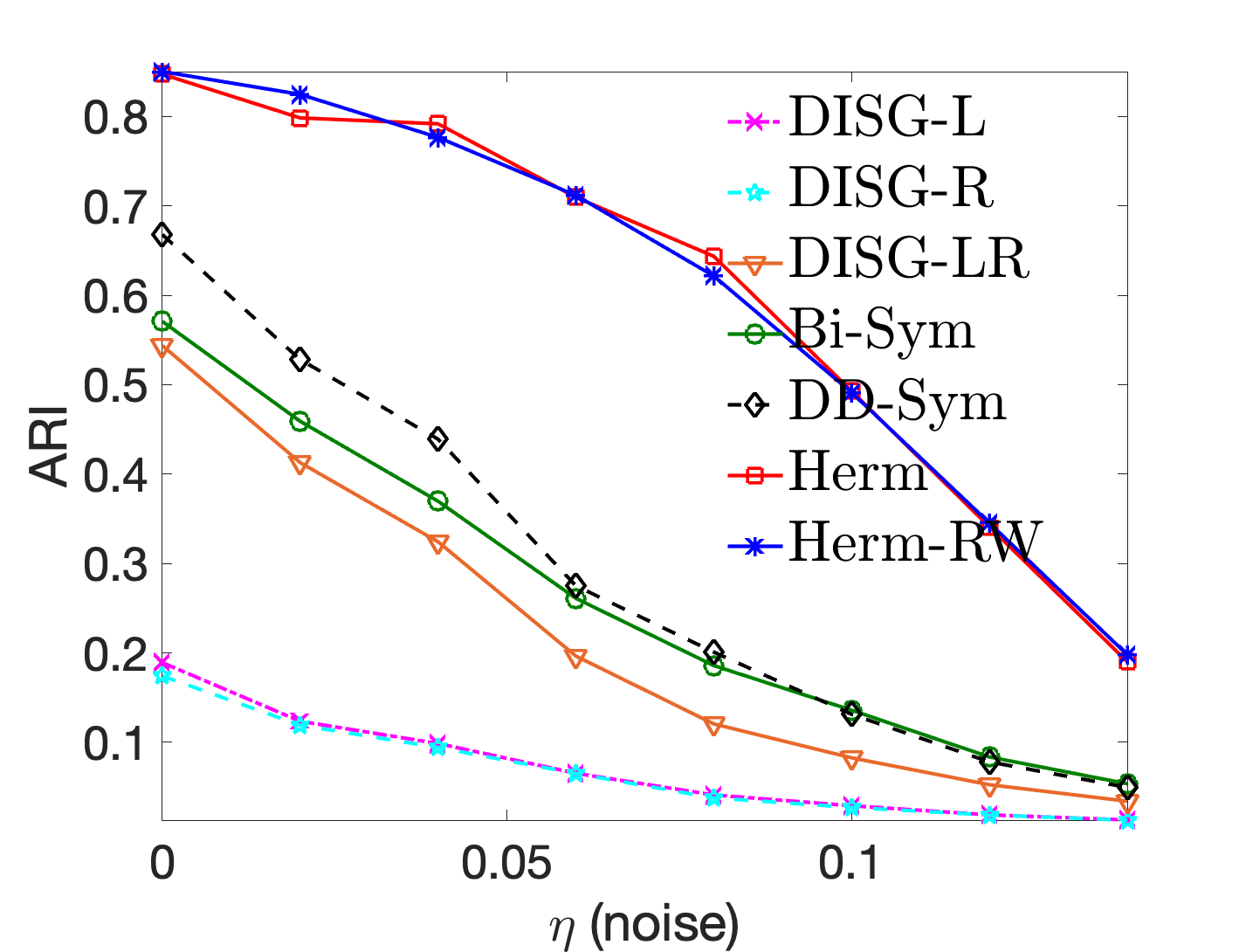}}
\hspace{3mm} 
\subcaptionbox{Recovery rates for the cyclic block model $ p=3\% $. }[0.3\columnwidth]{\includegraphics[width=0.3\columnwidth]{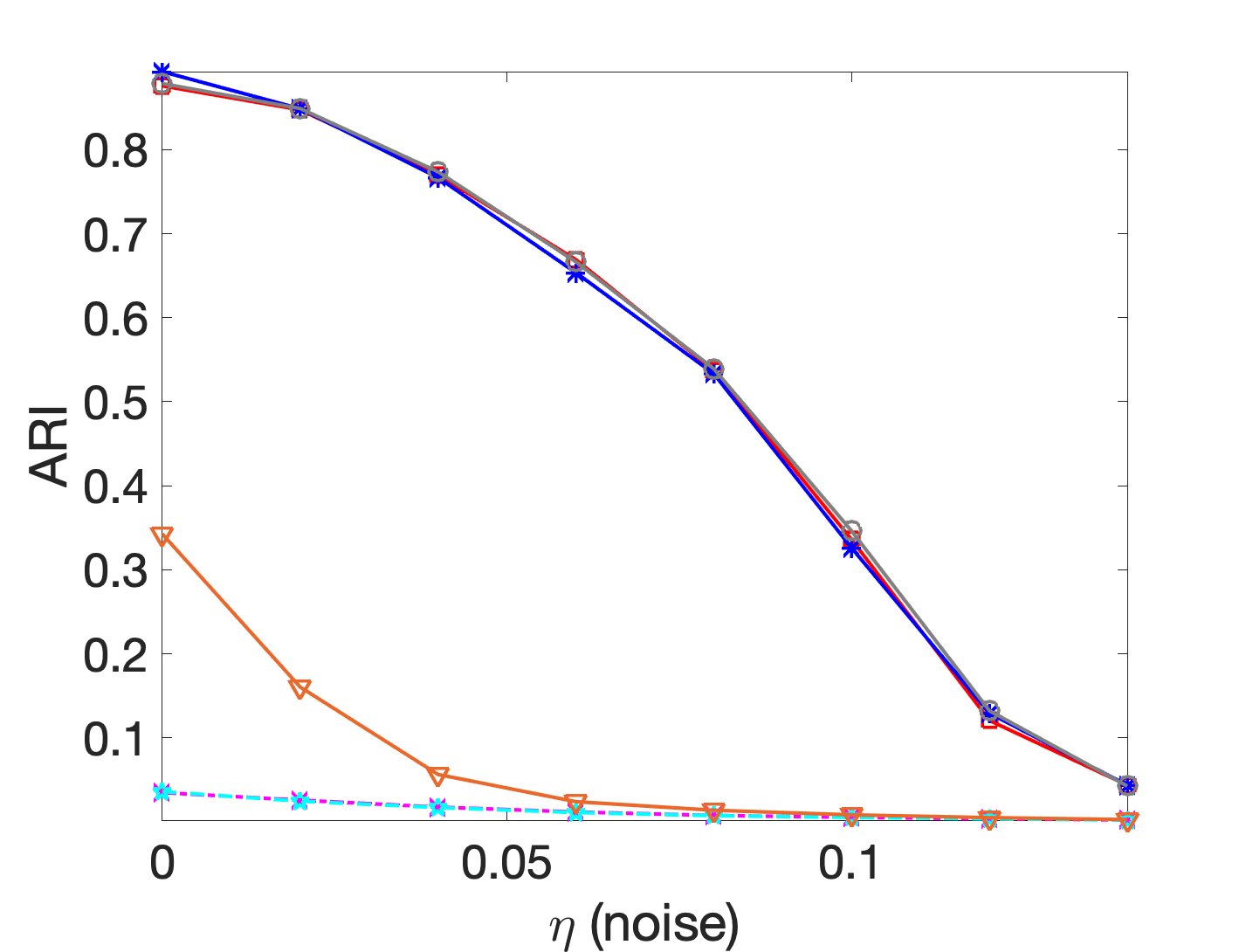}}
 \hspace{3mm} 
\subcaptionbox{Runtime analysis for the complete meta-graph $p=0.4\%$. }[0.3\columnwidth]{\includegraphics[width=0.3\columnwidth]{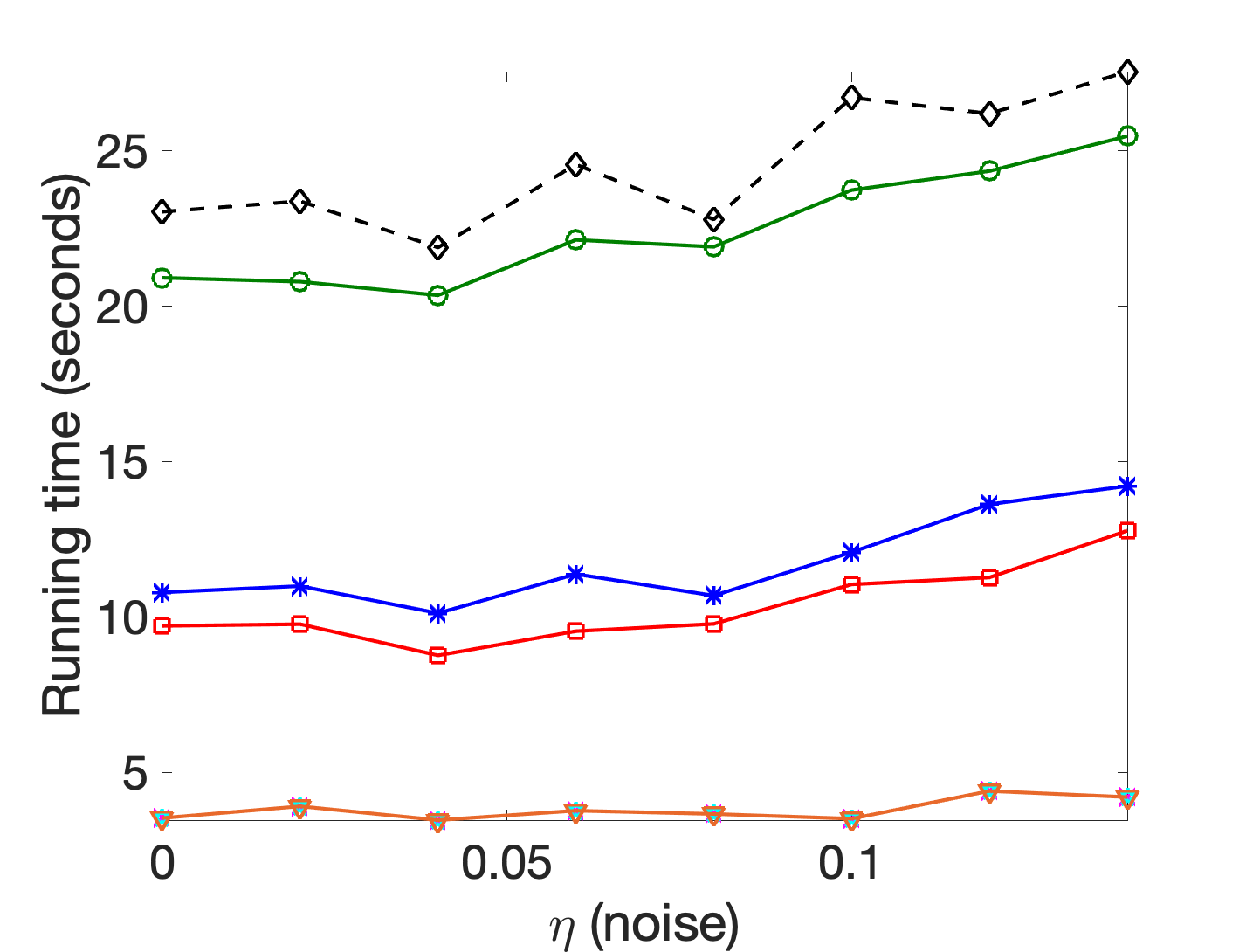}}
\end{centering} 
\captionsetup{width=0.99\linewidth} 
\caption{\small Recovery rates and running time for a complete meta-graph and a cyclic block model in a \textsc{DSBM} with $k=20$ clusters, $N=10,000$ at various levels of noise.  Averaged over 10 runs.} 
\label{fig:scanID_2a_b} 
\end{figure}

\noindent \textbf{\textsc{US-Migration}}: 
We present further numerical results for the main data set of our submission, omitted from the main text due to page limit.  
We compare the performance of all the variants of the algorithms listed in the submission, and Figure  \ref{fig:mig1Top5_ALL_Methods} is a visualisation of the top three largest pairs in terms of the size-normalised cut imbalance ratio for $k=10$. The \textsc{Naive} method performs standard spectral clustering on the symmetrised matrix $G = M + M^{\rot}$. In addition to our two proposed Hermitian-based approaches, \textsc{Herm} and \textsc{Herm-RW},  considered thus far throughout the paper, we also evaluate the performance of a third method which we denote \textsc{Herm-Sym}. In a similar spirit, \textsc{Herm-Sym} considers the top $k$ largest eigenvalues of the following matrix 
\begin{equation}
    A_{\mathrm{sym}} = D^{-1/2} A  D^{-1/2}, 
\label{def:Lsym}
\end{equation}
and recovers the clusters via $k$-means in this spectral embedding. The normalisation in $A_{\mathrm{sym}}$ is  particularly suitable for the skewed degree distributions often encountered in real data. For each highlighted pair (shown in red and blue in the US map colorings, while yellow denotes the remaining nodes), we also show the numerical scores achieved by the respective pair in terms of the three performance  metrics~($\cip$,  $\cis$, and  $\civ$). Here are our conclusions: 
\begin{itemize}\itemsep 0.3pt
\item 
In terms of the  $\cip$ score, the top three methods are \textsc{Herm-Sym} (0.26), \textsc{Herm-Rw} (0.19), followed by \textsc{DD-SYM} (0.16) and \textsc{DISG-LR} (0.16). We remind the reader that an imbalance score of $\cip=0.26$  as achieved by \textsc{Herm-Sym}  essentially denotes that $26\% + 50\%  = 76\% $ of the total weight of the edges between   a pair of clusters is oriented in one direction, and the remaining $24\%$ in the other direction. 
\item 
In terms of the  $\cis$ score, the top three  methods are \textsc{Herm-RW} (105), \textsc{Herm-Sym} (63), \textsc{Herm} (30), and  \textsc{Bi-Sym} (30).
\item 
Finally, in terms of the  $\civ$ score, the top three methods are \textsc{Bi-Sym} (20,062),  \textsc{Herm-Sym}(19,570) and \textsc{Herm-RW} (16,268). 
\end{itemize}

% \clearpage 
\bigskip 
%%%%%%%%%%%%-------------------------------------------------------------------------------------------------
%%%%%%%%%%%%-------------------------------------------------------------------------------------------------
 
 \vspace{-0.4cm}

\noindent \textbf{\textsc{US-Migration-II}}: 
%  mig2:
Due to a small number of very large entries in the initial migration matrix $M$, many of the methods we compare against are not able to produce meaningful results.%\footnote{We remark that without dropping the outlier entries, \text{Herm-RW} performs the best with respect to  the $ \civtop$ values  for every single value of $k = \{3, 5, 8, 10, 20,30,40 \}$.}. 
To this end, we pre-process the migration matrix $M$ and cap all entries at $10,000$, which corresponds to the  
% 99.96\% percentile.  -- dropped one decimal and saved a line!!
99.9\% percentile. 
As shown in Figure \ref{fig:intronaive}, a simple symmetrisation of the input matrix $M \mapsto M+M^{\rot}$, followed by standard spectral clustering of undirected graphs \cite{luxburg07}, will reveal clusters that align very well with the state boundaries \cite{belgium2010}.  
The top, respectively bottom, plots in Figure  \ref{fig:scanID_8b_mig2_new} show the $\cip$, respectively $\civ$, score for the top pairs. 
% for the \textsc{US-Migration-II} data set, 
For the CI score, \textsc{Herm} and \textsc{Herm-RW} are among the top performing methods along with \textsc{Bi-Sym}, while for $\civ$, \textsc{Herm-RW} is the best performing method across all values of $k=\{2,10,20,40\}$.

\renewcommand{\wid}{1.5in}
%\newcolumntype{C}{>{\centering\arraybackslash}m{\wid}}
 
\begin{figure*}[h!]\sffamily
\hspace{-1cm}
\begin{tabular}{l*4{C}@{}}
 & $k=2$ & $k=10$ & $k=20$ & $k=40$ \\ 
& \includegraphics[width=0.2470\columnwidth]{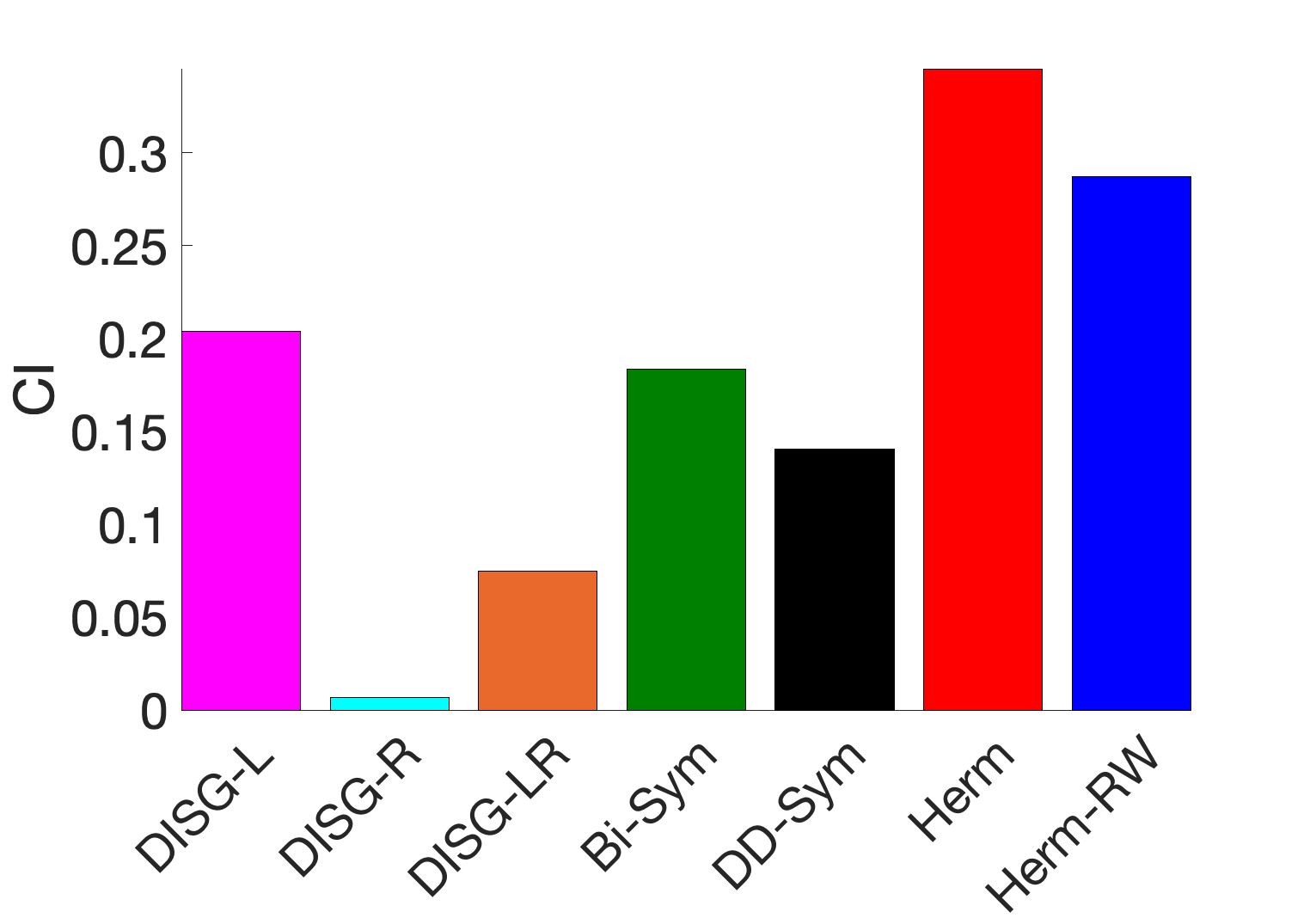}
& \includegraphics[width=0.2470\columnwidth]{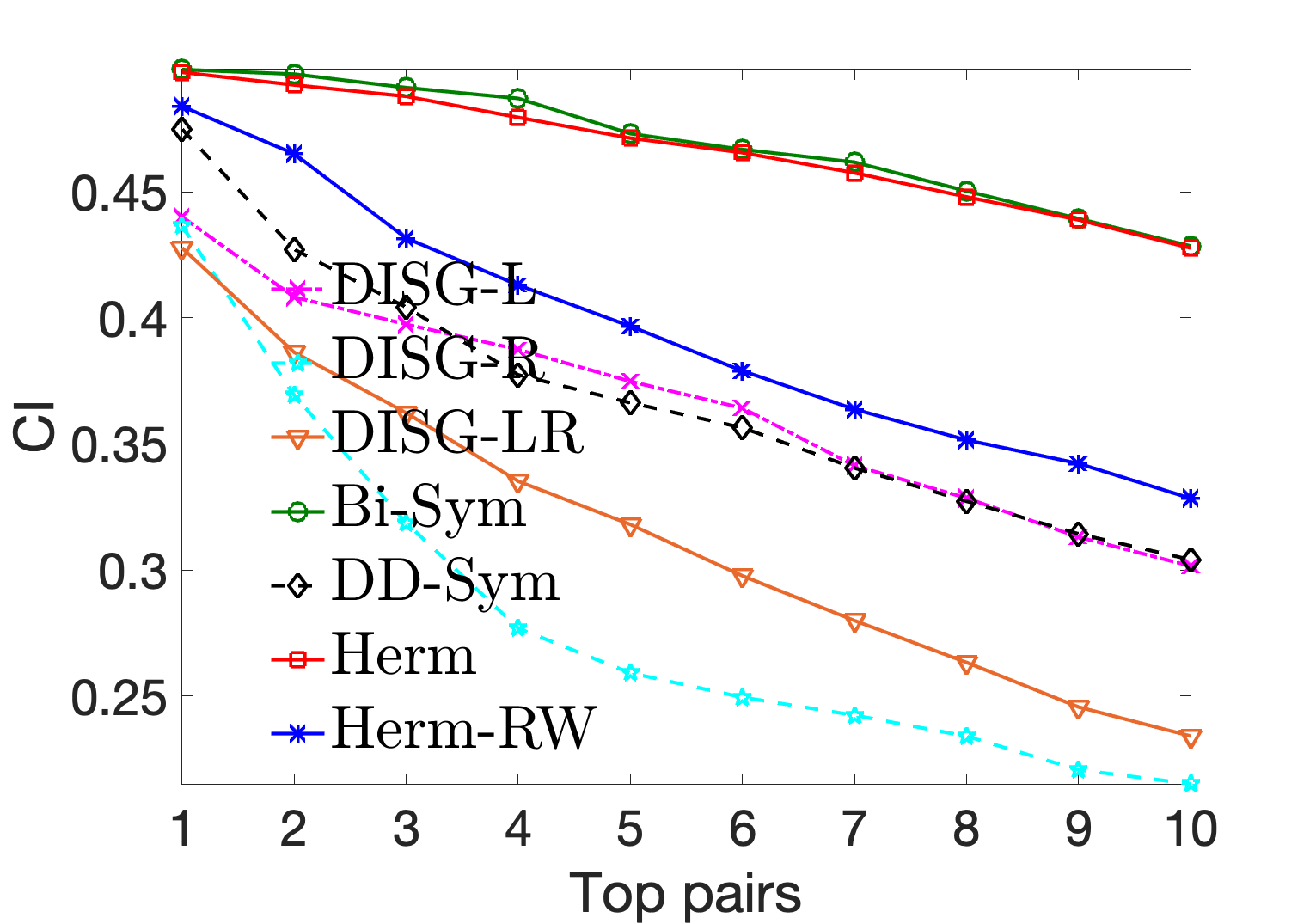}
& \includegraphics[width=0.2470\columnwidth]{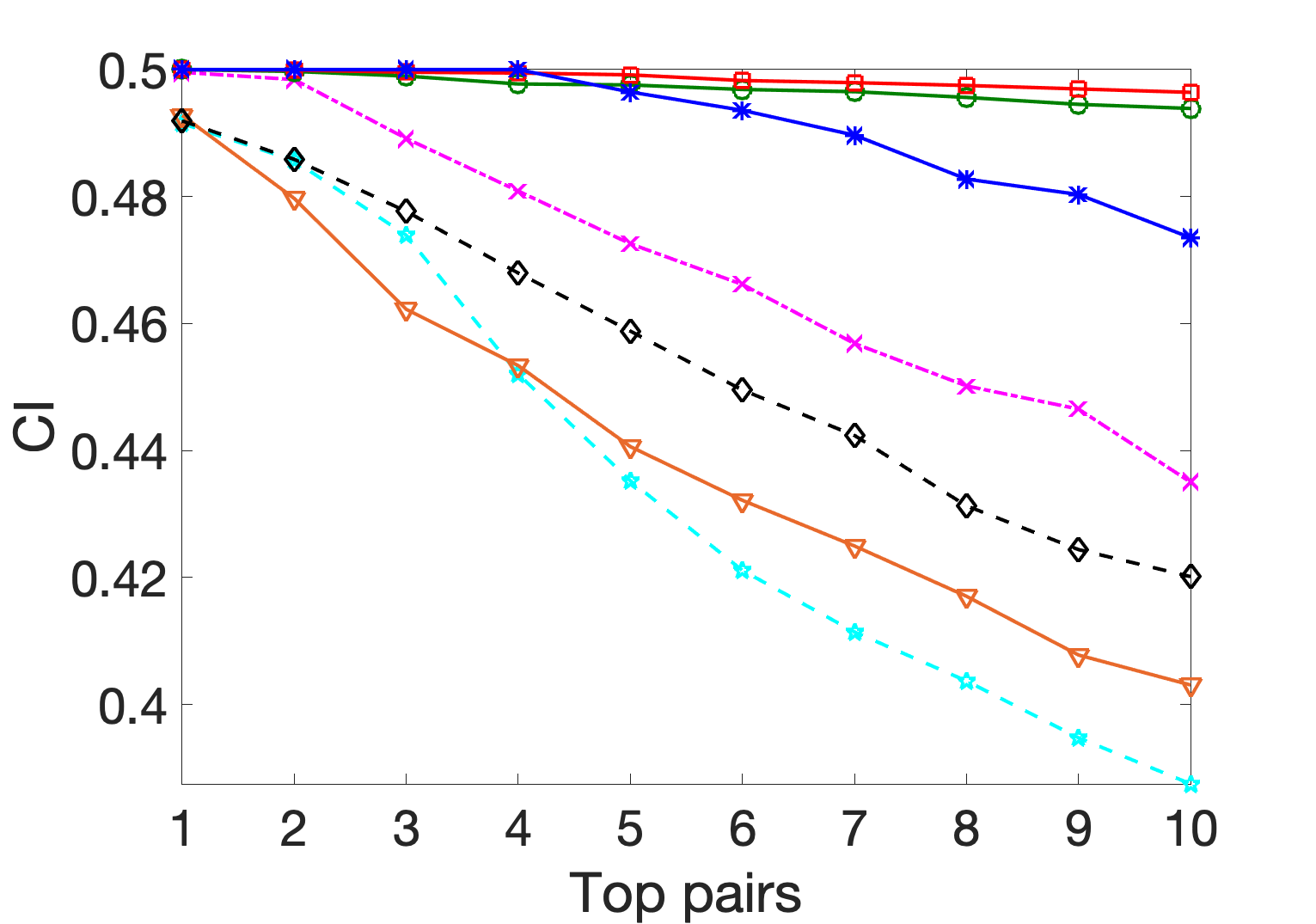}
& \includegraphics[width=0.2470\columnwidth]{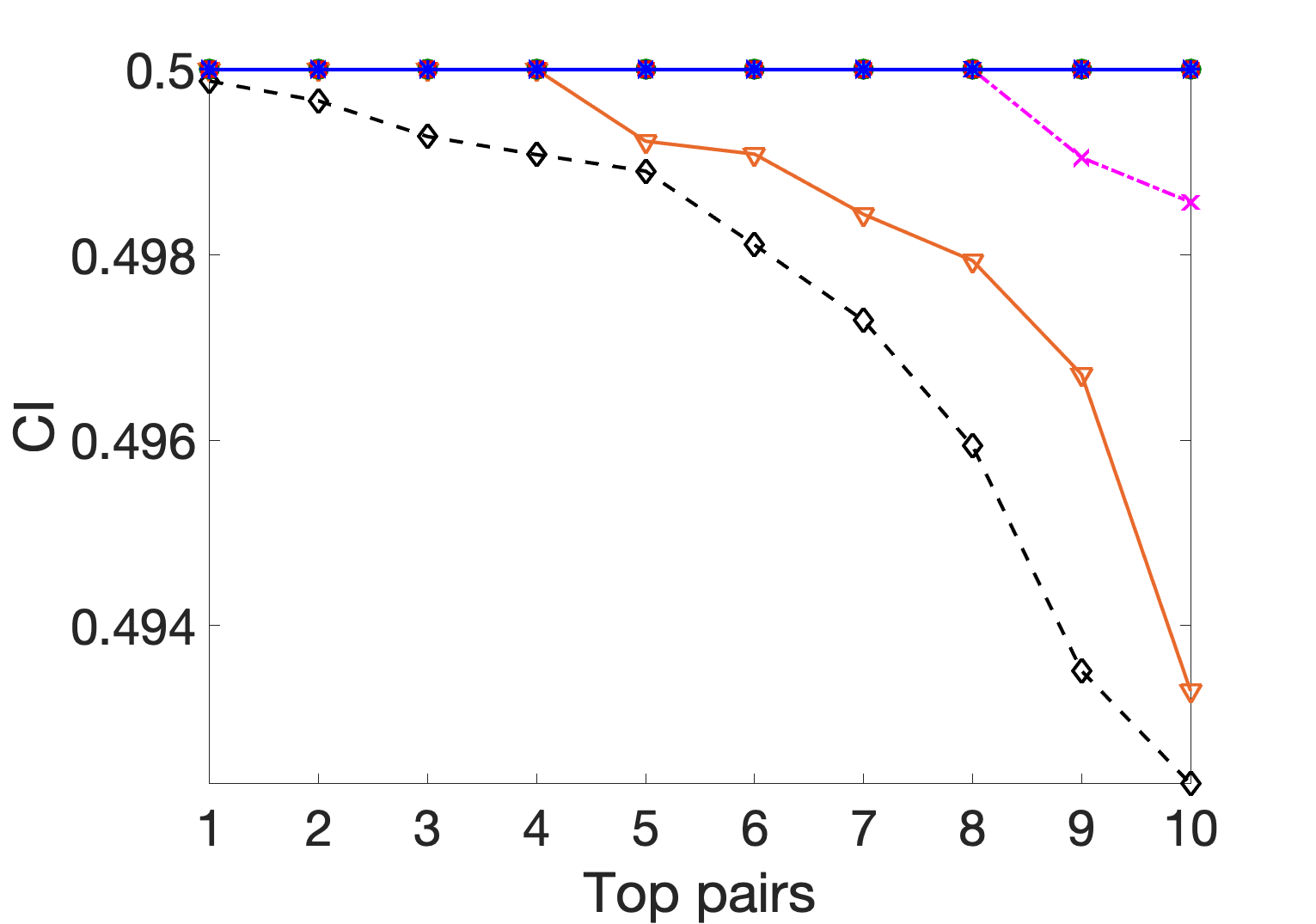} \\ 
& \includegraphics[width=0.2470\columnwidth]{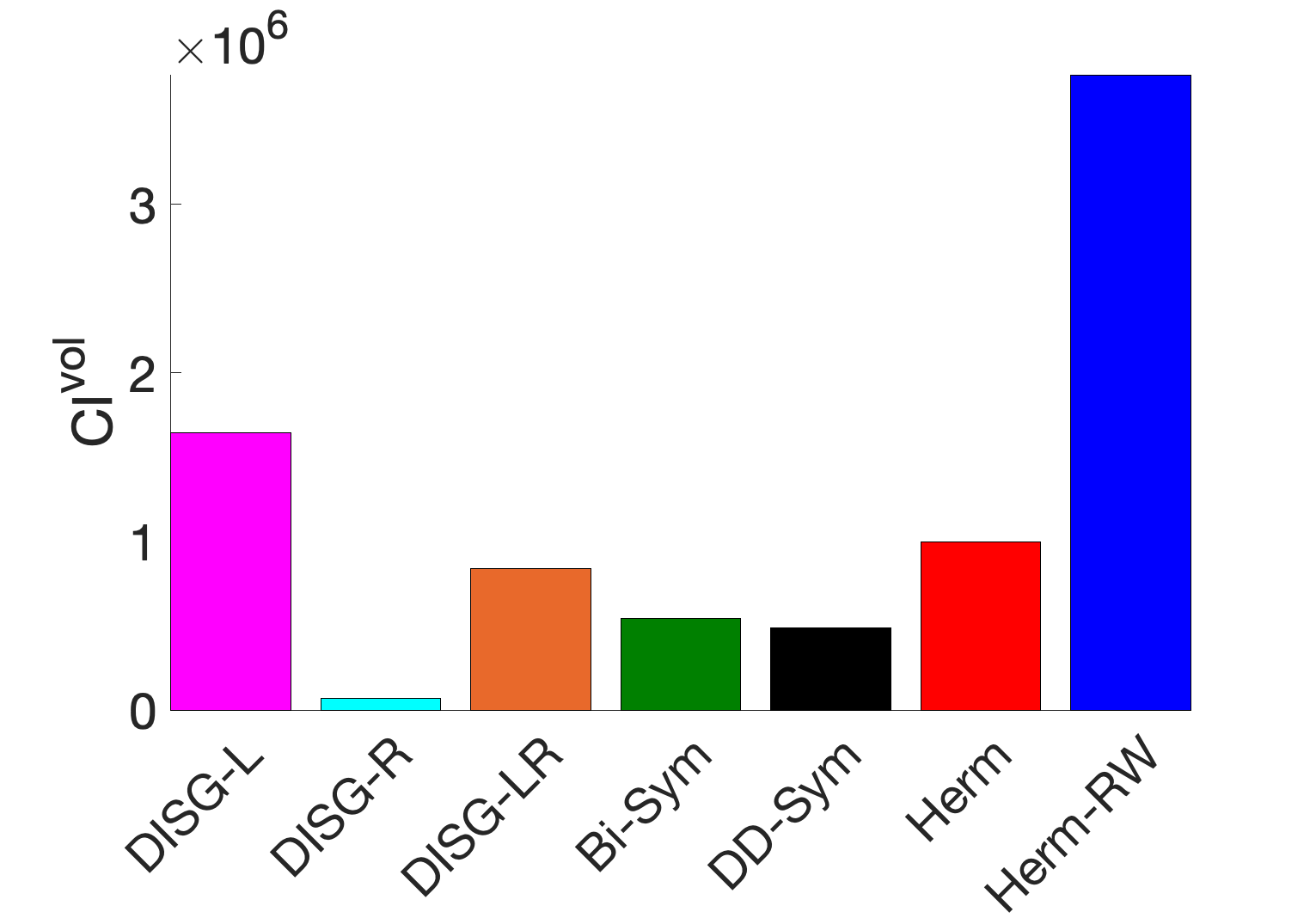}
& \includegraphics[width=0.2470\columnwidth]{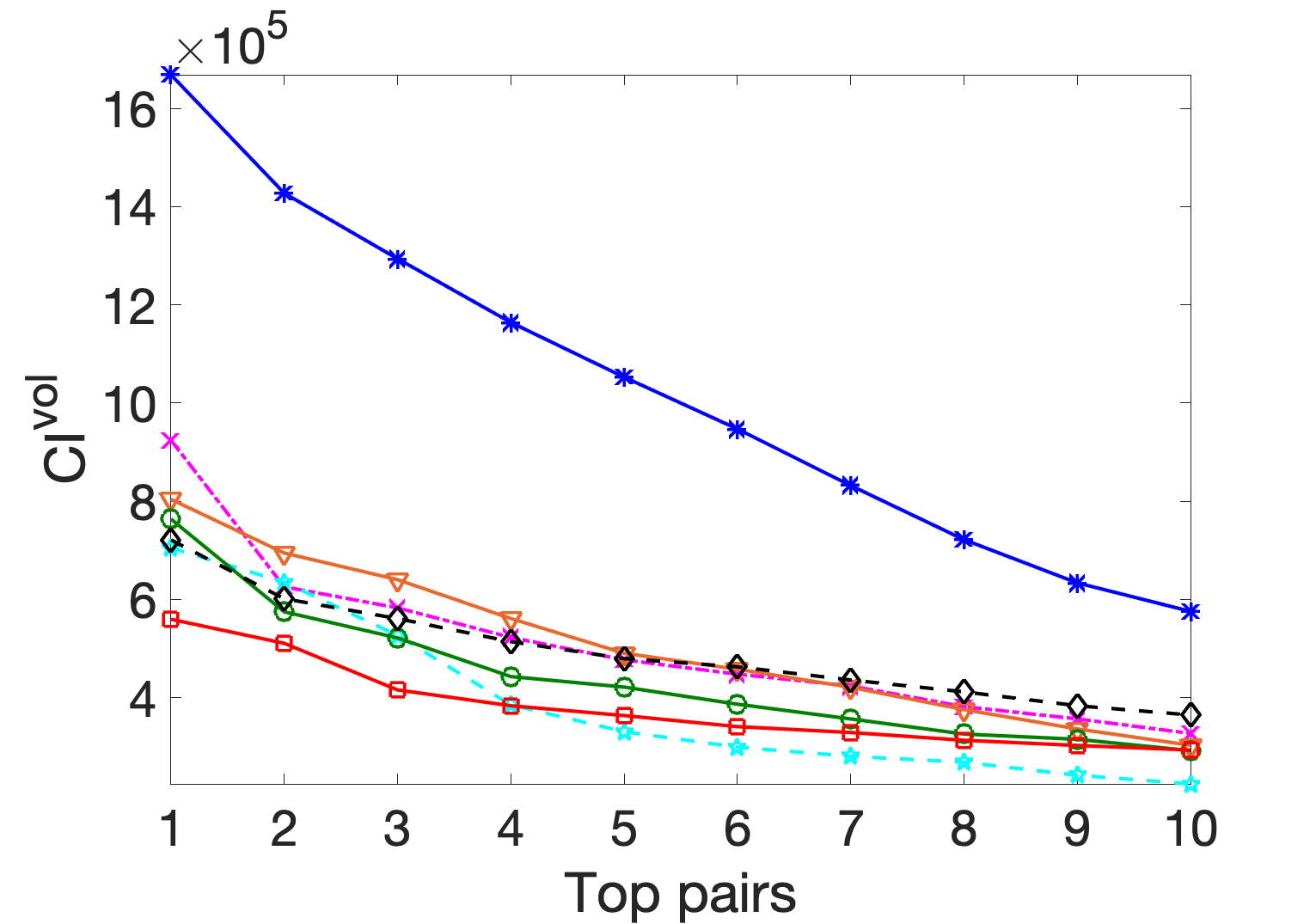}
& \includegraphics[width=0.2470\columnwidth]{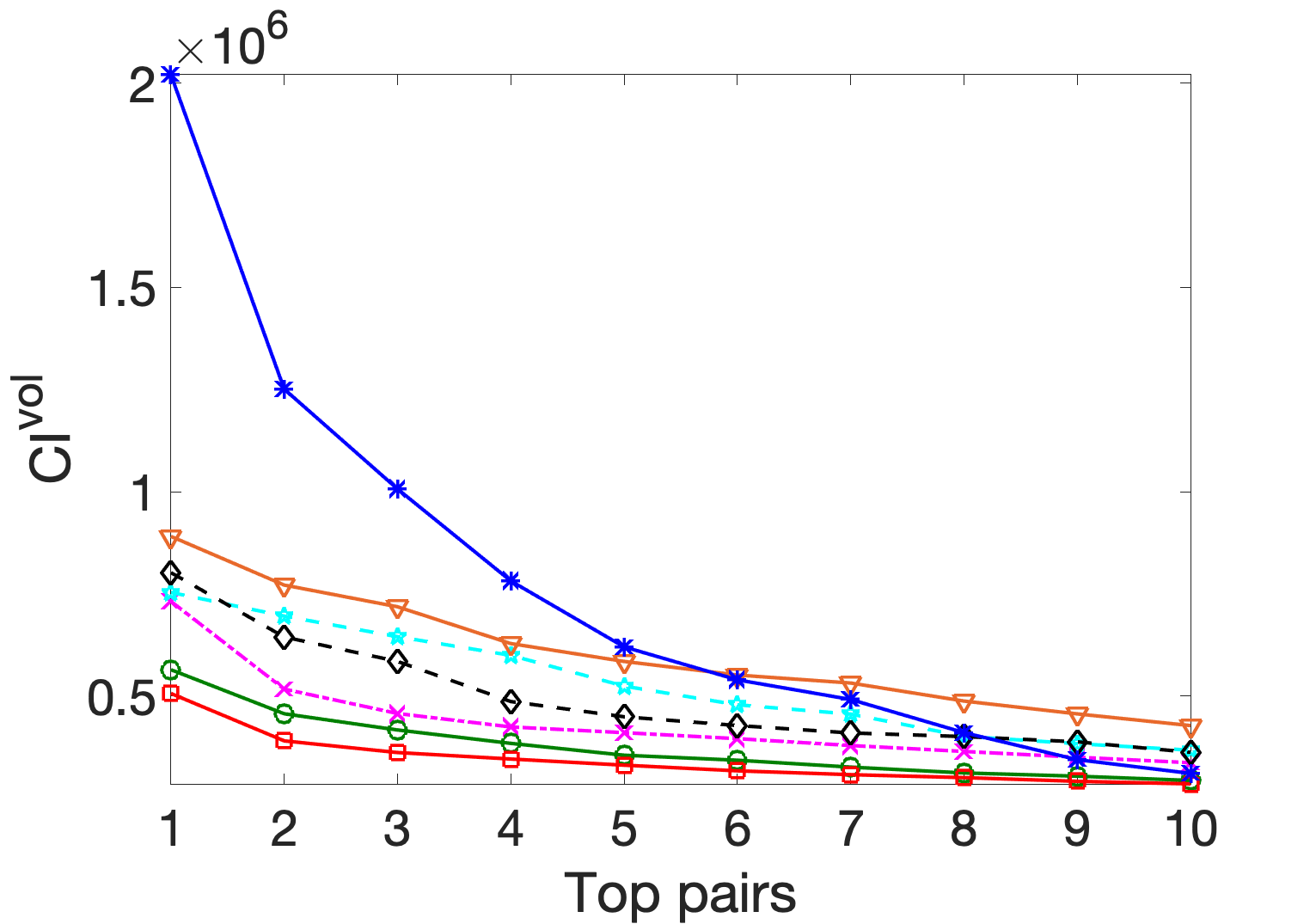}
& \includegraphics[width=0.2470\columnwidth]{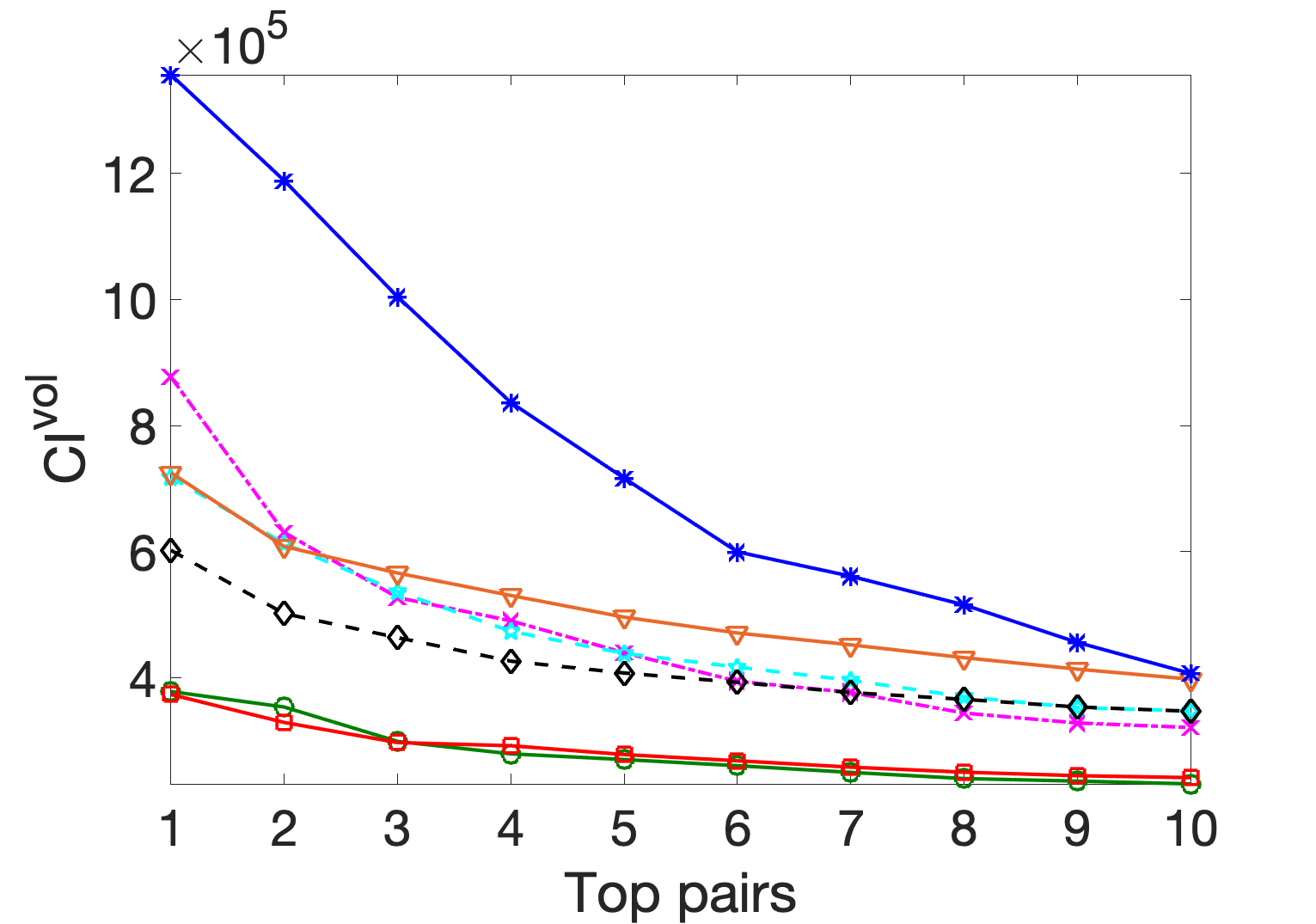} \\ 
\end{tabular}
% \captionsetup{width=0.8\linewidth}
\vspace{-2mm}
\captionof{figure}{The $ \cip $ and   $\civ$ scores attained by the top pairs, for the \textsc{US-Migration-II} data set with $N=3,107$ and $k=\{2,10,20,40\}$ clusters (averaged over 20 runs).}
\label{fig:scanID_8b_mig2_new}
\end{figure*}

 \vspace{-0.2cm}

\begin{figure*}[h]\sffamily
\hspace{-1cm}
\begin{tabular}{l*4{C}@{}}
& $k=2$ & $k=3 $ & $k=5$ & $k=8$ \\ 
& \includegraphics[width=\wid]{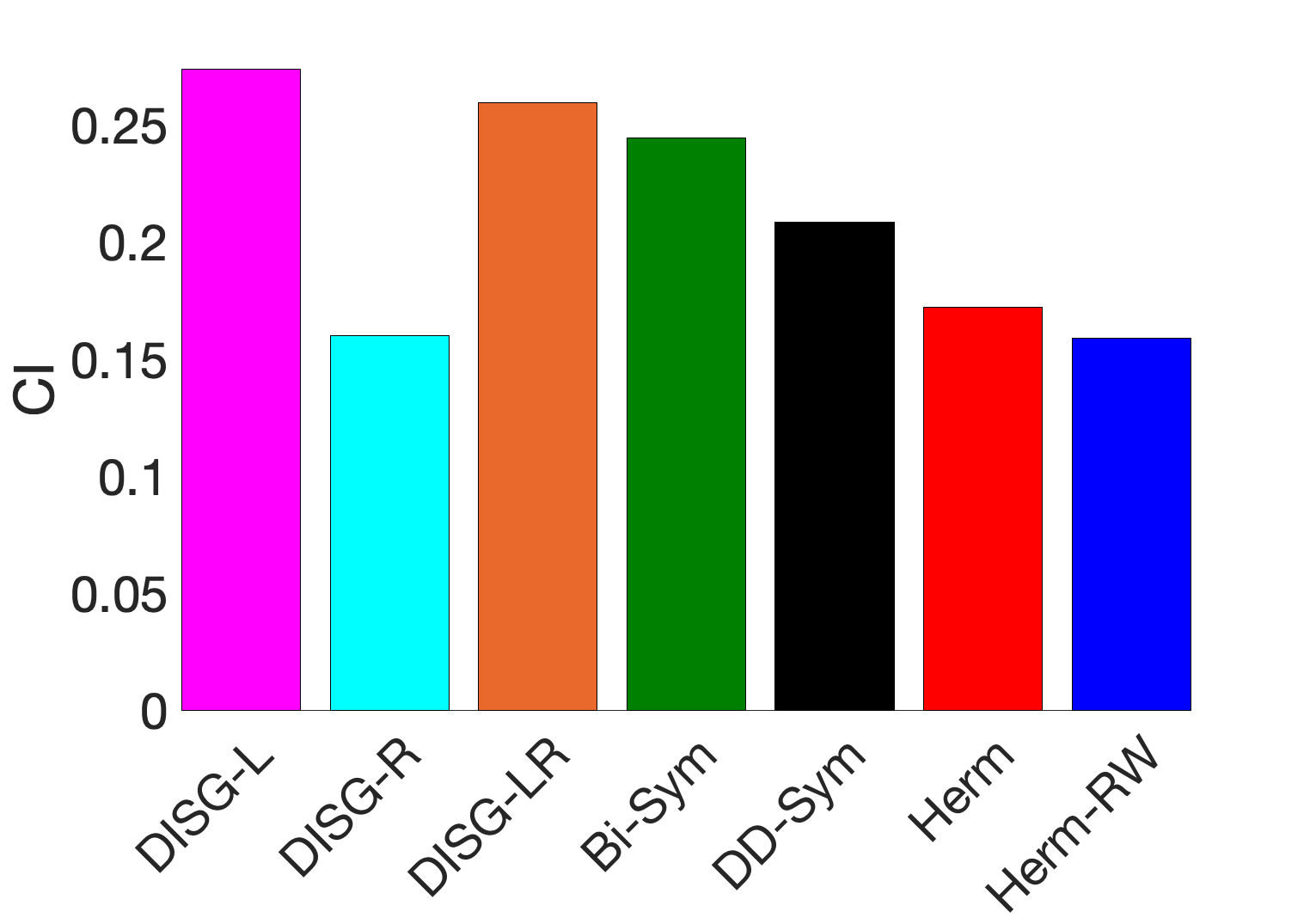}
& \includegraphics[width=\wid]{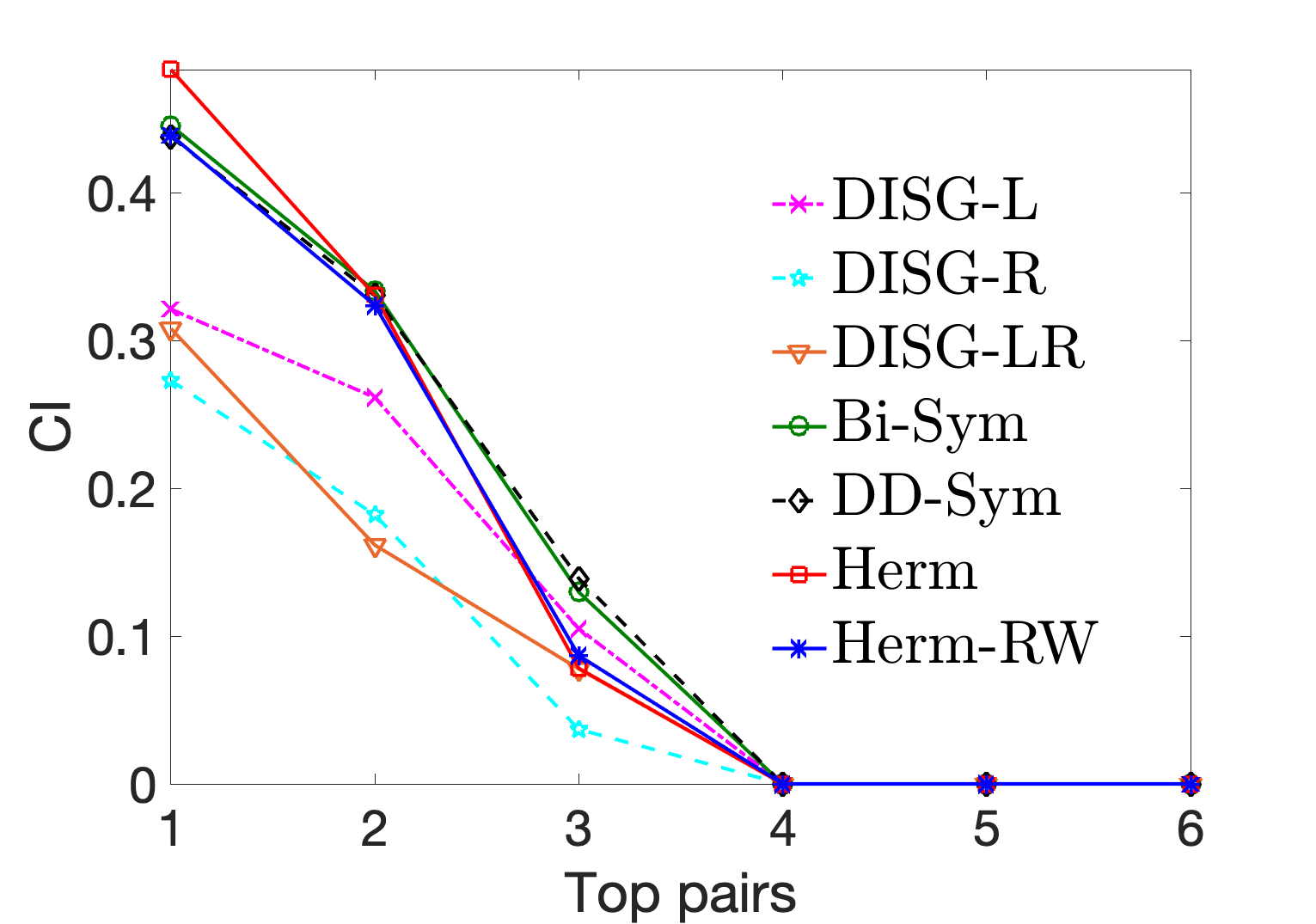}
& \includegraphics[width=\wid]{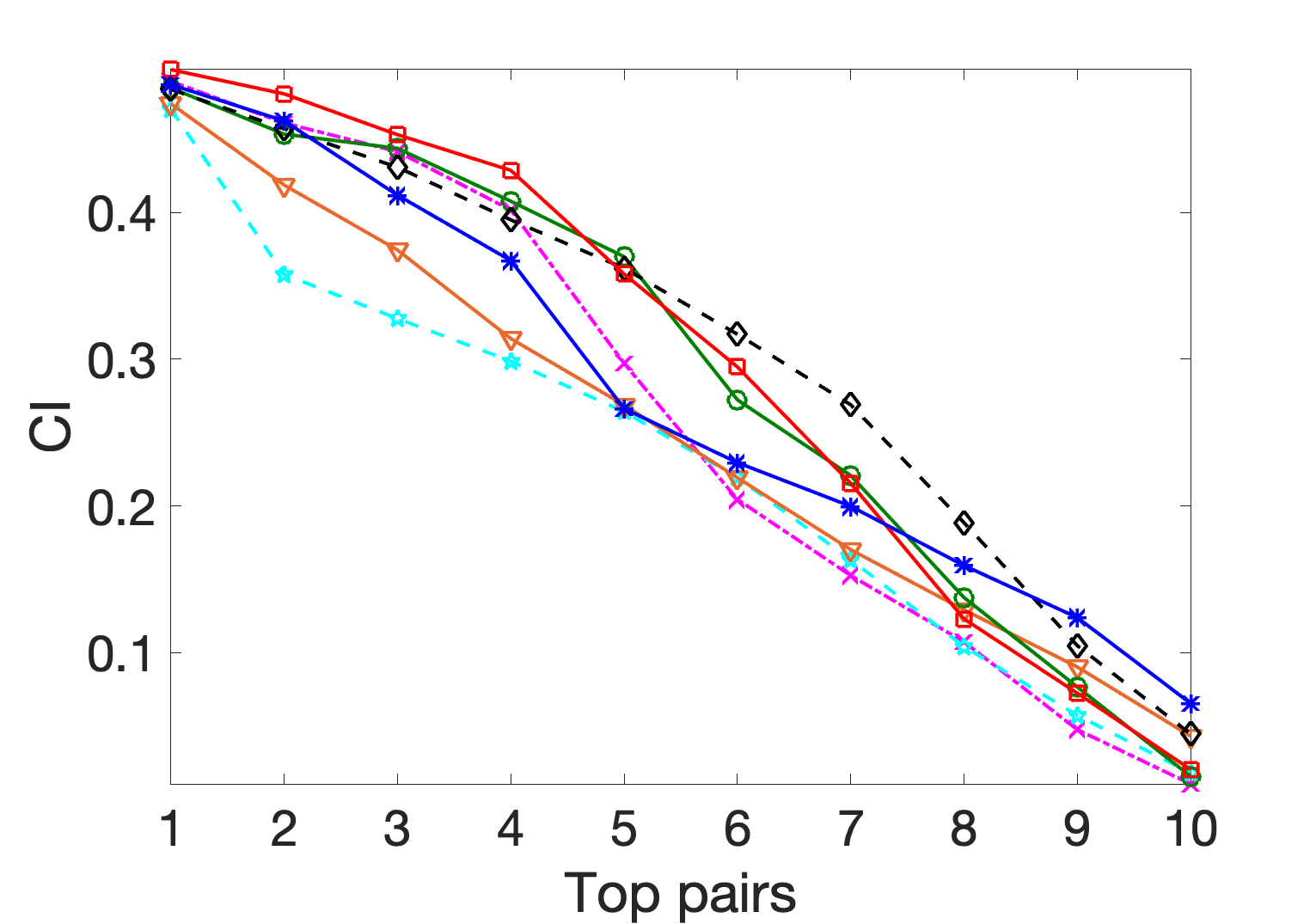}
& \includegraphics[width=\wid]{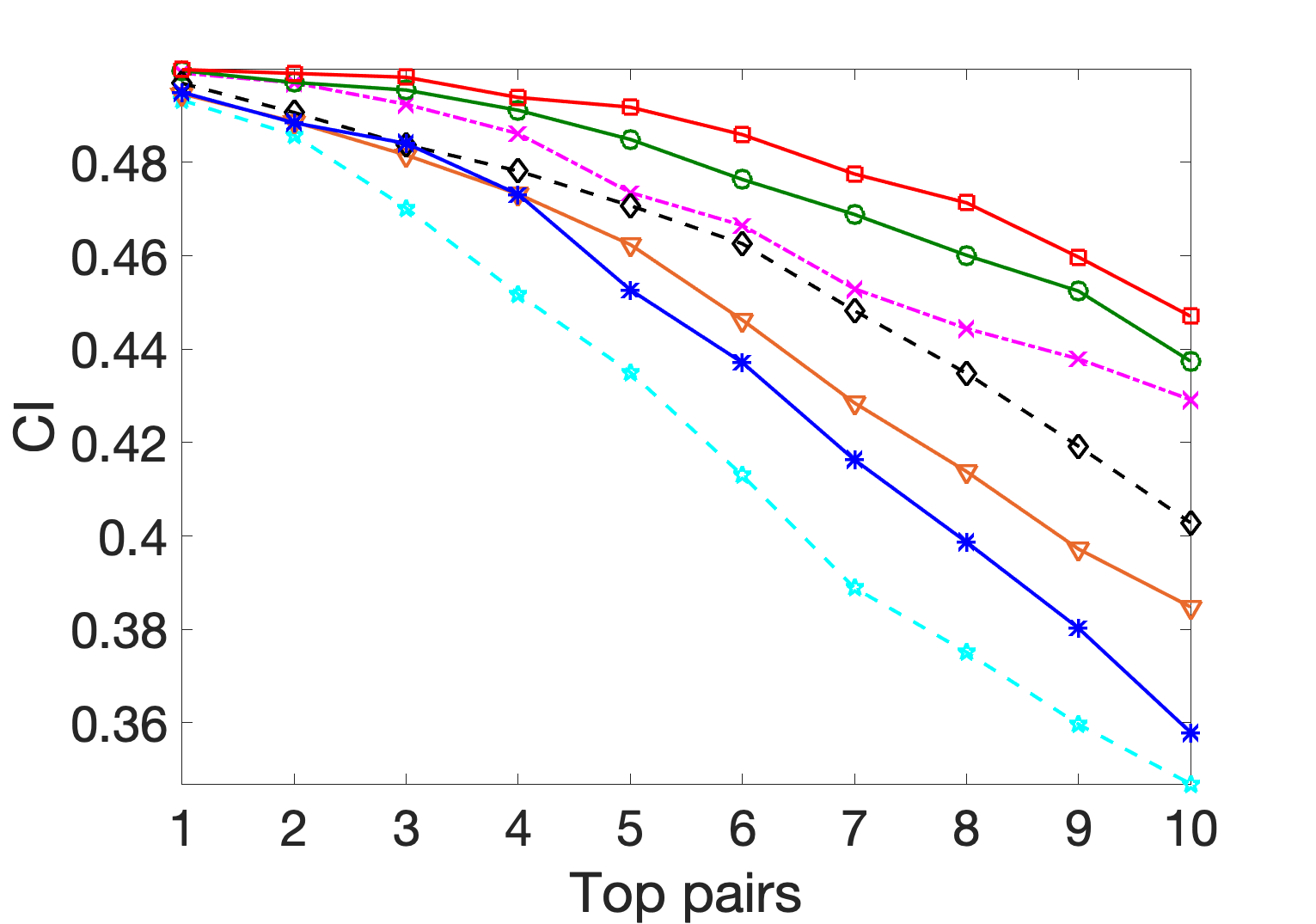} \\ 
& \includegraphics[width=\wid]{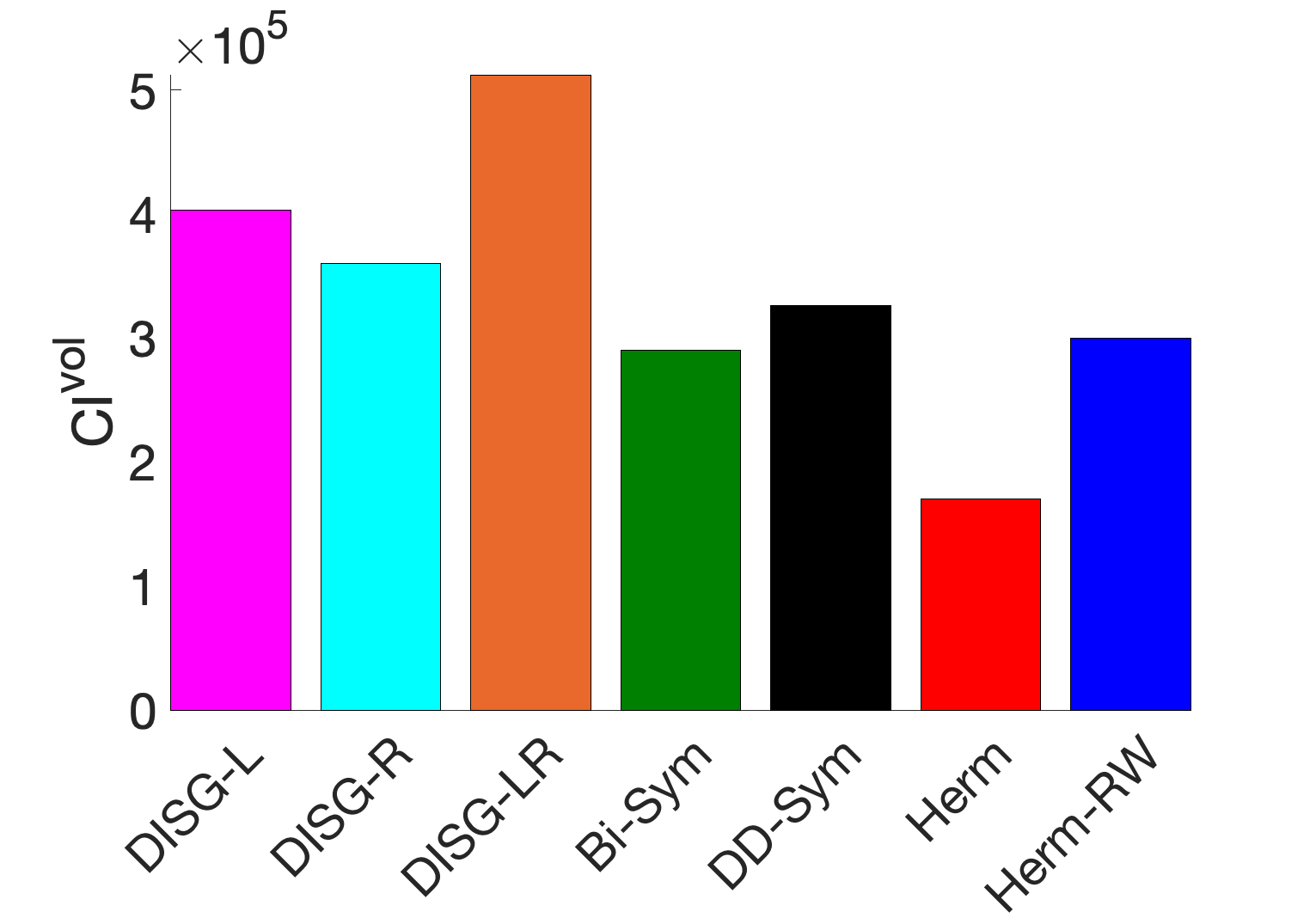}
& \includegraphics[width=\wid]{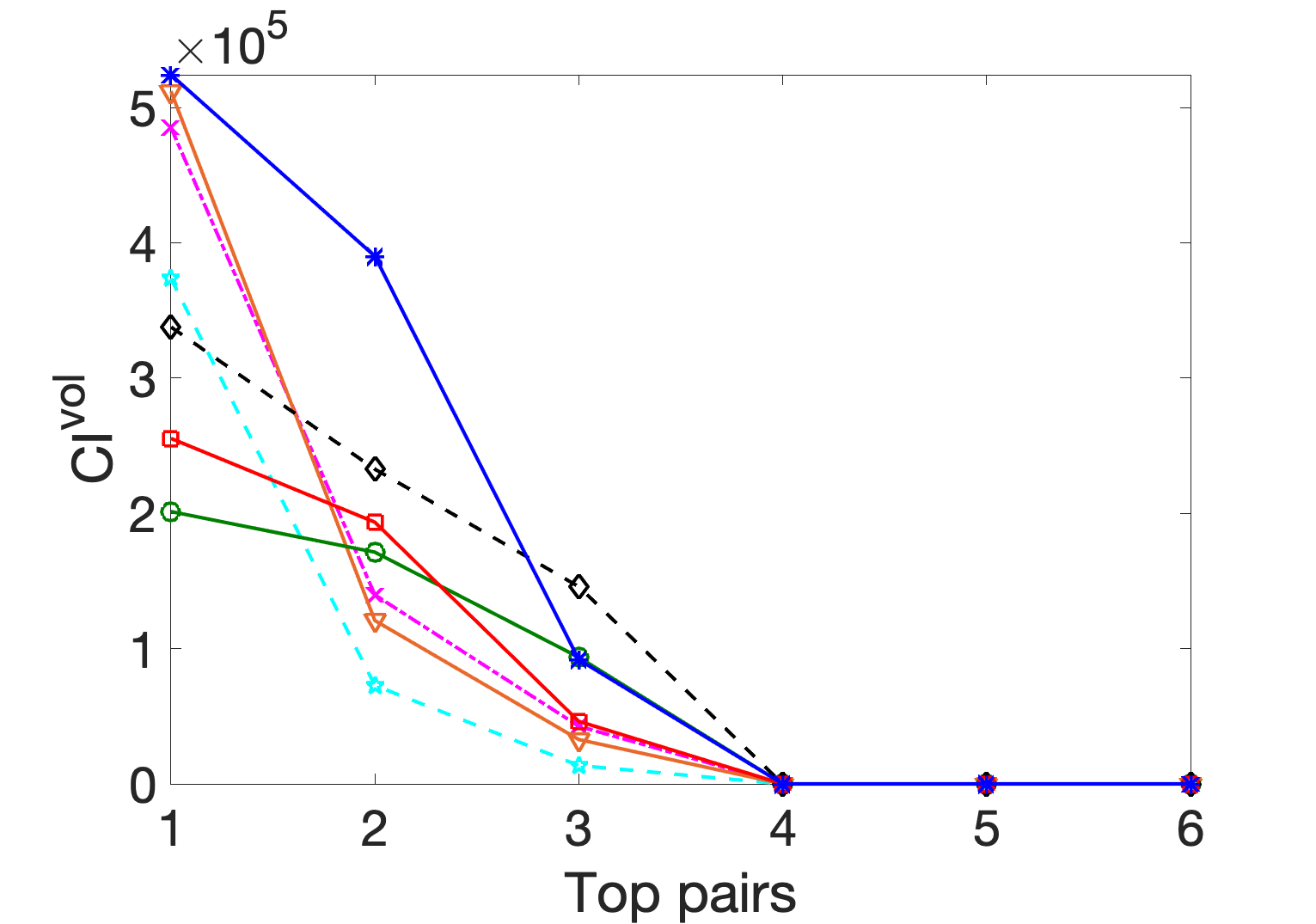}
& \includegraphics[width=\wid]{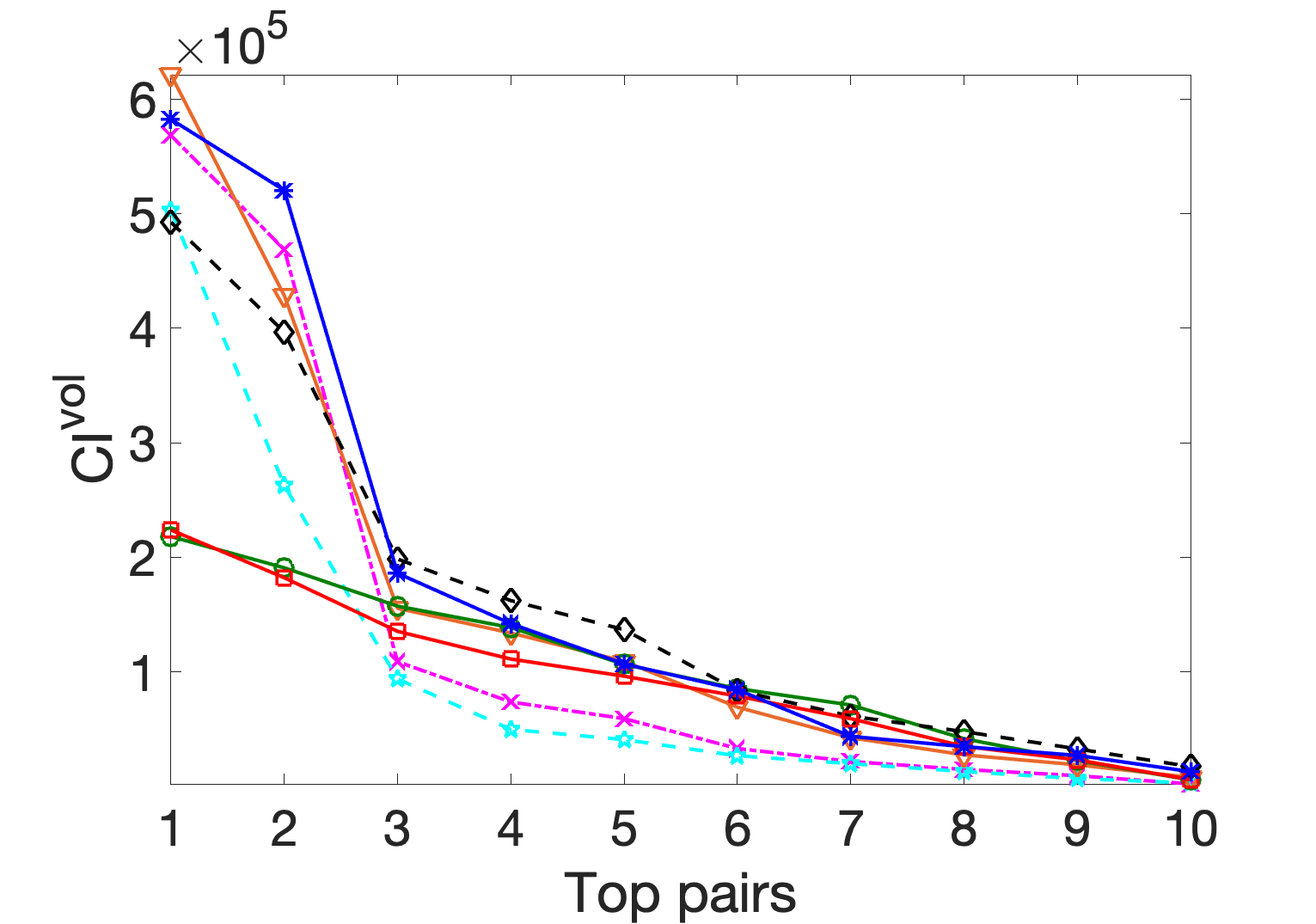}
& \includegraphics[width=\wid]{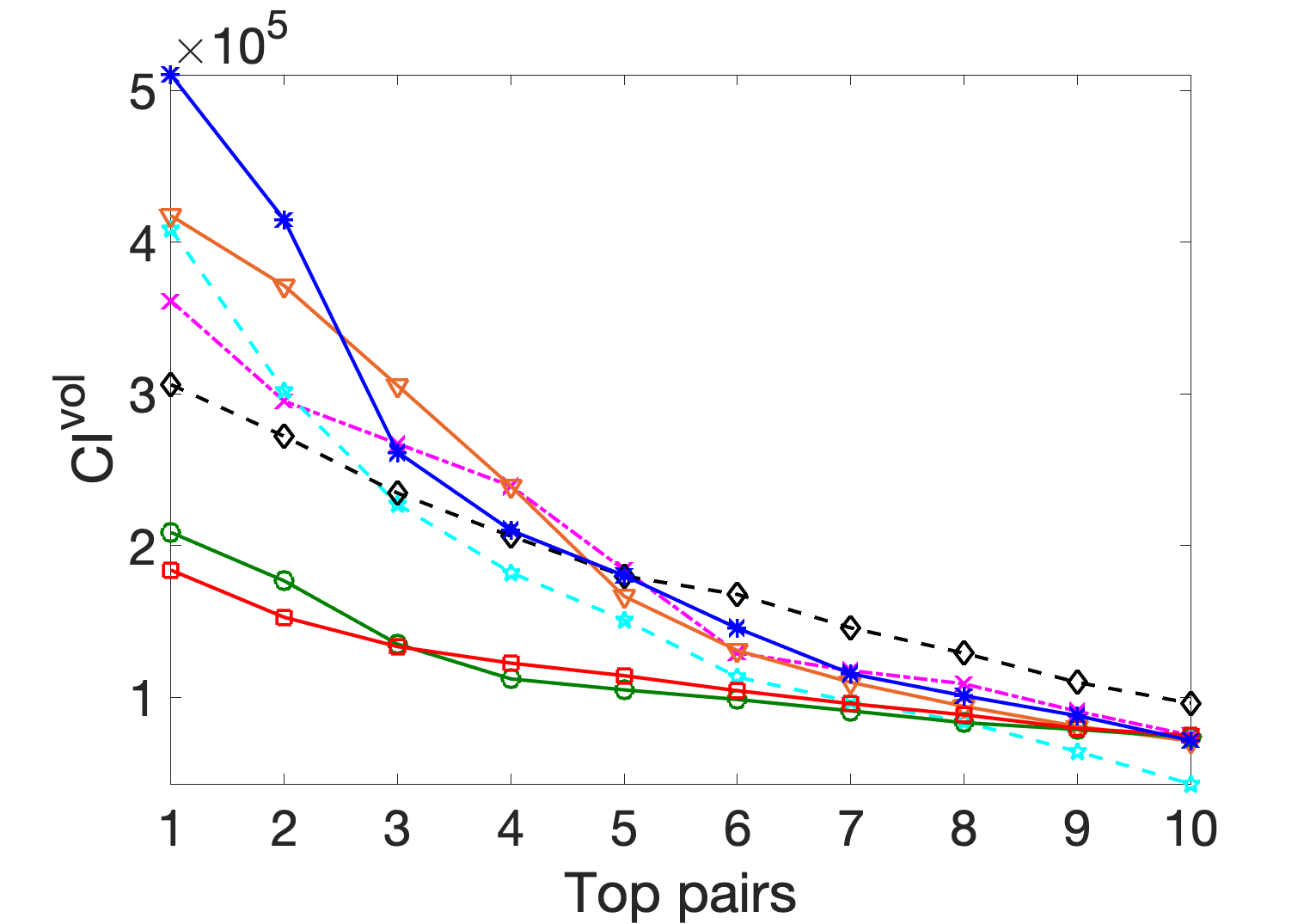} \\ 
\end{tabular}
% \captionsetup{width=0.8\linewidth} 
\vspace{-2mm} 
\captionof{figure}{The $ \cip $ and   $\civ$ scores attained by the top pairs, for the \textsc{UK-Migration}  data set with $N=354$ and $k=\{2,3,5,8\}$ clusters (averaged over 20 runs).}
\label{fig:scanID_10_uk}
\end{figure*}

\vspace{0.1cm}
\begin{wrapfigure}{r}{.25\textwidth}
%\begin{figure}
    \begin{minipage}{\linewidth}
    %\vspace{-0.3cm}
    \centering 
    \includegraphics[width=0.7\columnwidth]{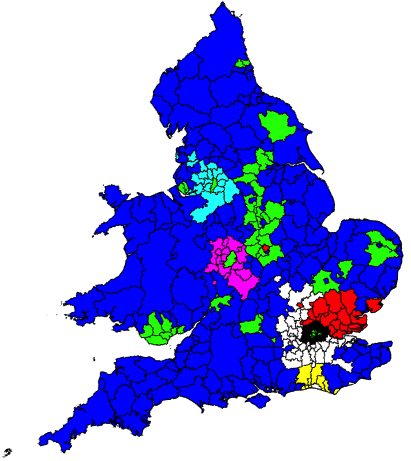}  % \subcaption{$p = 0.01$}  % \label{fig:5a}  % \par\vfill 
    %  \hspace{-10mm} 
   % \includegraphics[width=0.49\columnwidth]{Figures/scanID_1b_Kk/scanID_1b_n100_k50_p0p02_Kk_nrReps10_ARI.png}  % \subcaption{$p = 0.02$} % \label{fig:5b}
% Note: if we turn on the \subcaption{} - the two figures end up one under the other...
\end{minipage}
\caption{\small  Clustering structure recovered by \textsc{Herm-RW} ($k=8$) for \textsc{UK-Migration}.}
\vspace{-3mm}
\label{fig:uk_clist_k8}
\end{wrapfigure}

\noindent  \textbf{\textsc{UK-Migration}}: Another data set we considered is the \textsc{UK-Migration} network with $N=354$, which captures in a directed graph the number of people who migrated between local authority Districts in the UK, aggregated over the interval 2012-2017 \cite{UKmigdata}.  
 Figure \ref{fig:scanID_10_uk} shows the $\cip$ and $\civ$ scores for the top pairs, for varying values of $k$. For $k=2$, \textsc{DISG-L} and \textsc{DISG-LR} are the best performing methods.
For $k=\{3,5,8\}$, a number of methods perform comparably well, with \textsc{Herm-RW} being the best performer in terms of the $\civ$ scores.
% and \textsc{Herm-Sym} ranks second for $k=\{3,8\}$. 
% For $k=10$, the top three methods are   \textsc{DISG-L}, \textsc{DISG-LR} and \textsc{Herm-RW}.
Finally, Figure~\ref{fig:uk_clist_k8} shows the clustering recovered by \textsc{Herm-RW} with $k=8$ clusters, highlighting the Greater London metropolitan area,  as well as counties such as Essex, Surrey, West Sussex and Oxfordshire. 
 
% \clearpage
\bigskip

% (3)   kaiser2006nonoptimal
\noindent  \textbf{\textsc{c-Elegans}}:  The last data set we studied is the \textsc{c-Elegans} neural connectome network, which encodes connection between the neurons in a directed network \cite{white86a}. This popular data set \cite{snapnets}, also considered in \cite{co-clustering}, highlights significant dissimilarities between the sending and receiving patterns in the neural network. Figure \ref{fig:scanID_11_cele} compares the $\cip$ and $\civ$ scores corresponding to the top pairs, across all algorithms and for various values of $k$. For $k=2$, respectively $k=3$, \textsc{Herm-RW}  is the best performer, followed closely by \textsc{Bi-Sym}, resp. \textsc{Herm}, while the rest of the algorithms perform significantly worse. For higher $k=\{5,10\}$, results are mixed, with a number of methods performing similarly well.

\begin{figure*}[h] \sffamily
\hspace{-1cm}
\begin{tabular}{l*4{C}@{}}
 & $k=2$ & $k=3 $ & $k=5$ & $k=10$ \\ 
& \includegraphics[width=\wid]{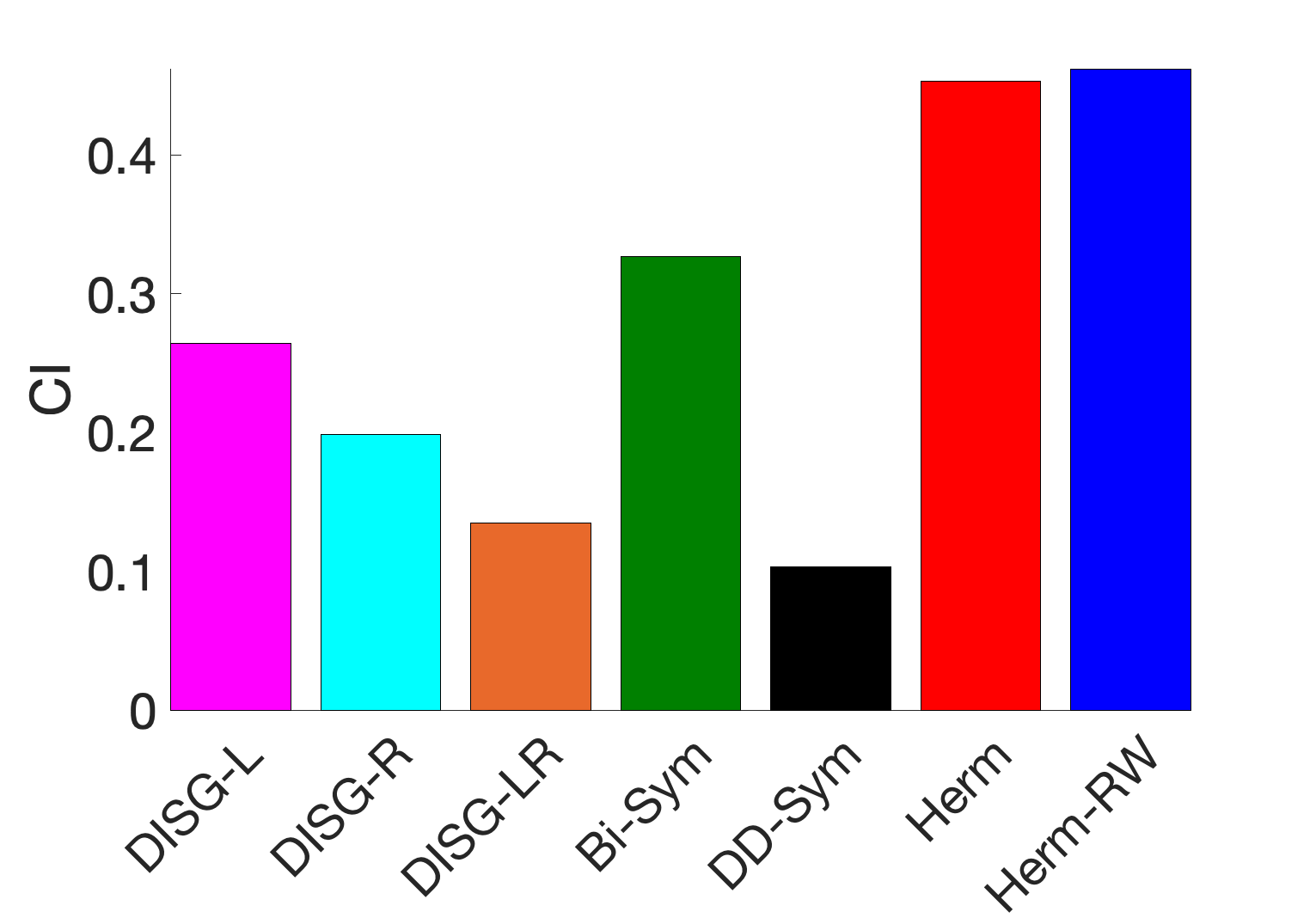} 
& \includegraphics[width=\wid]{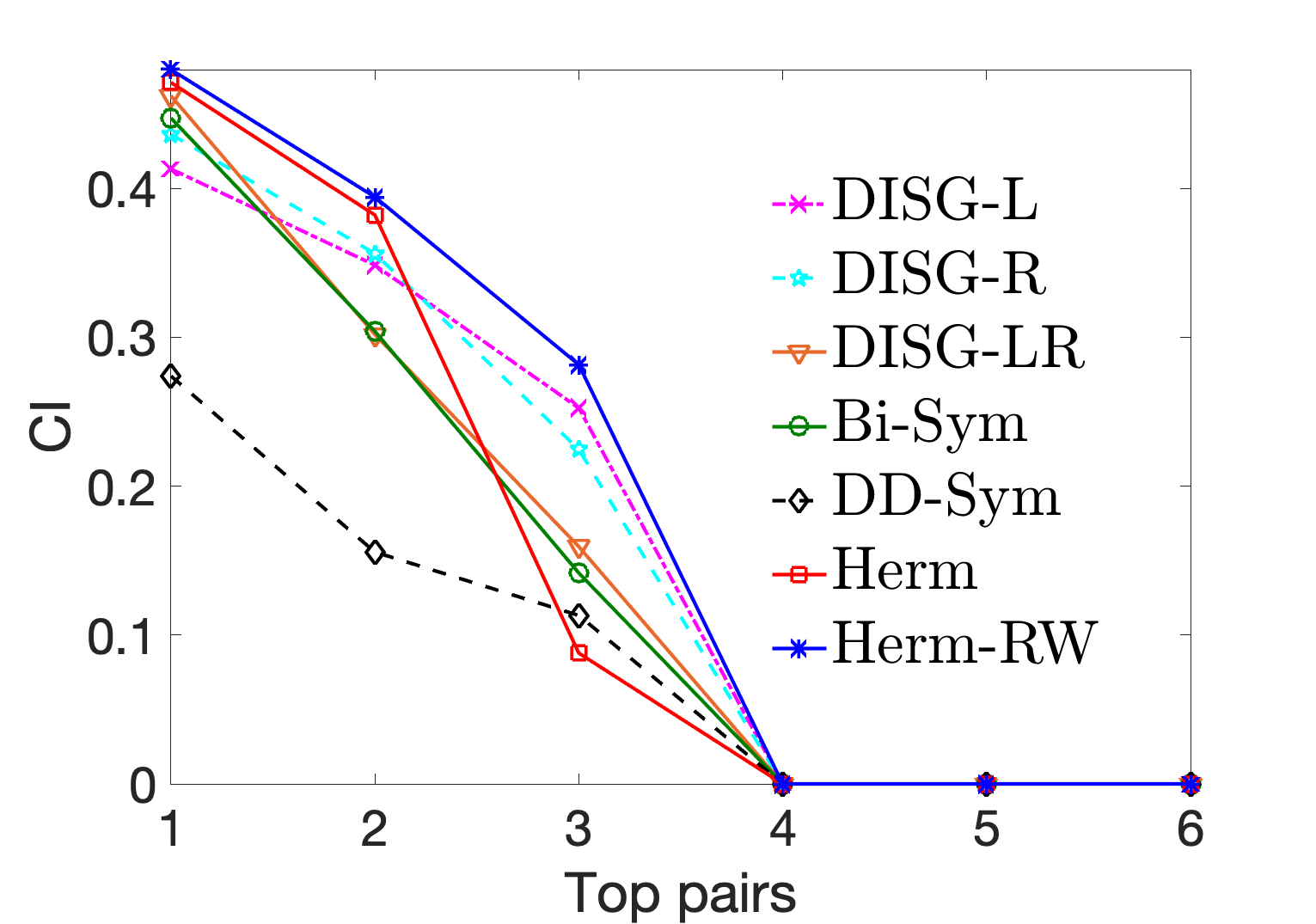} 
& \includegraphics[width=\wid]{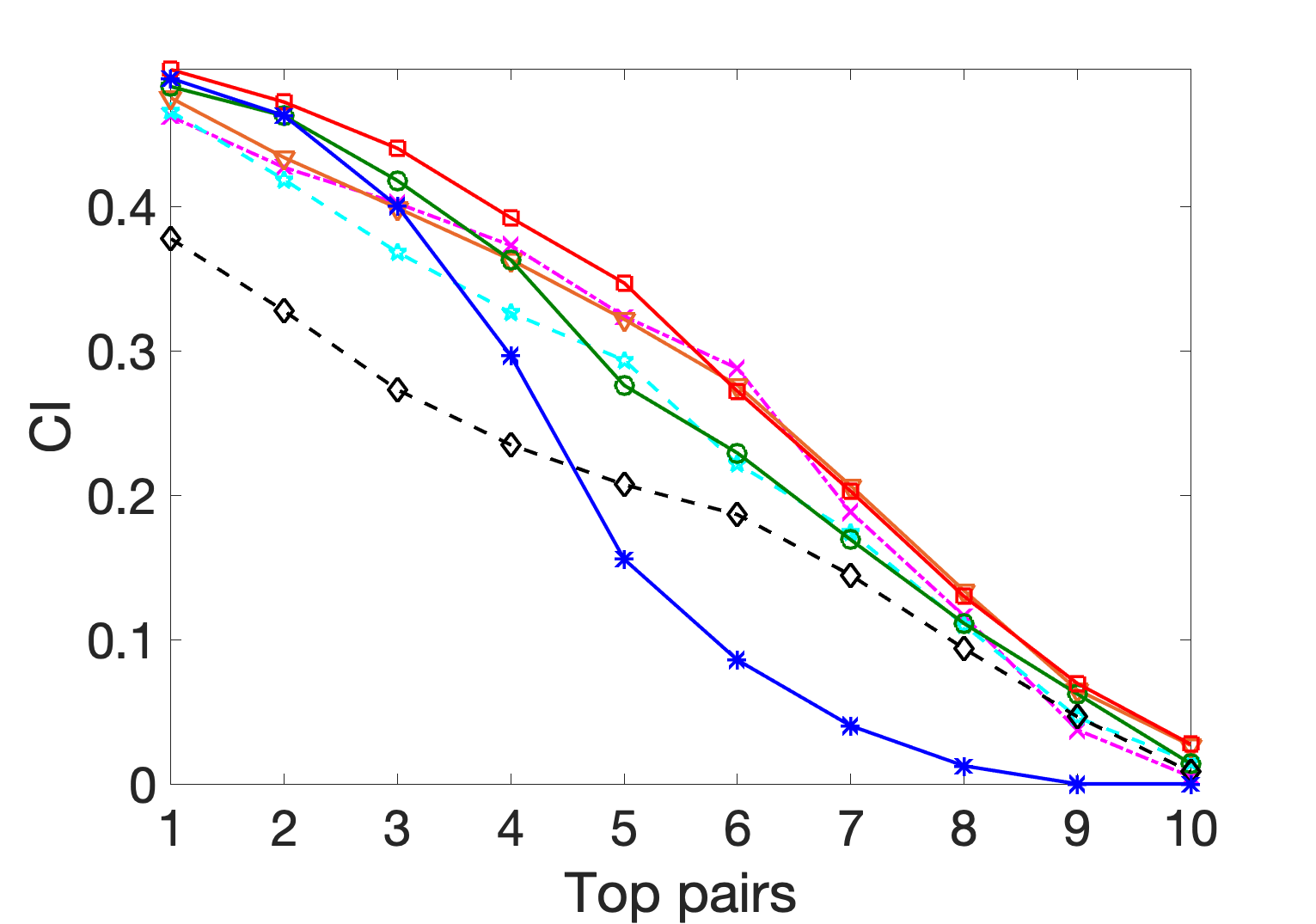}
& \includegraphics[width=\wid]{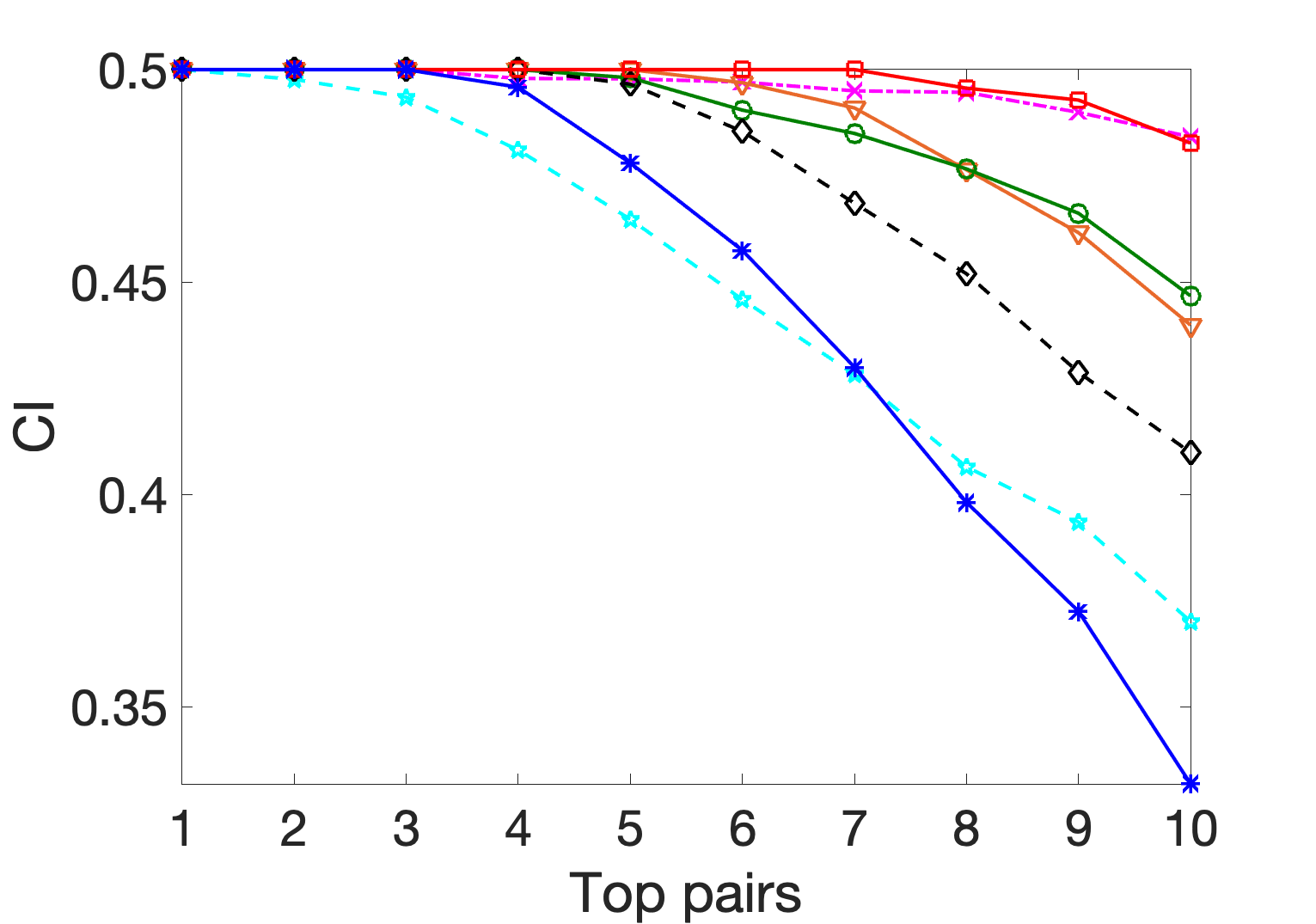} \\ 
& \includegraphics[width=\wid]{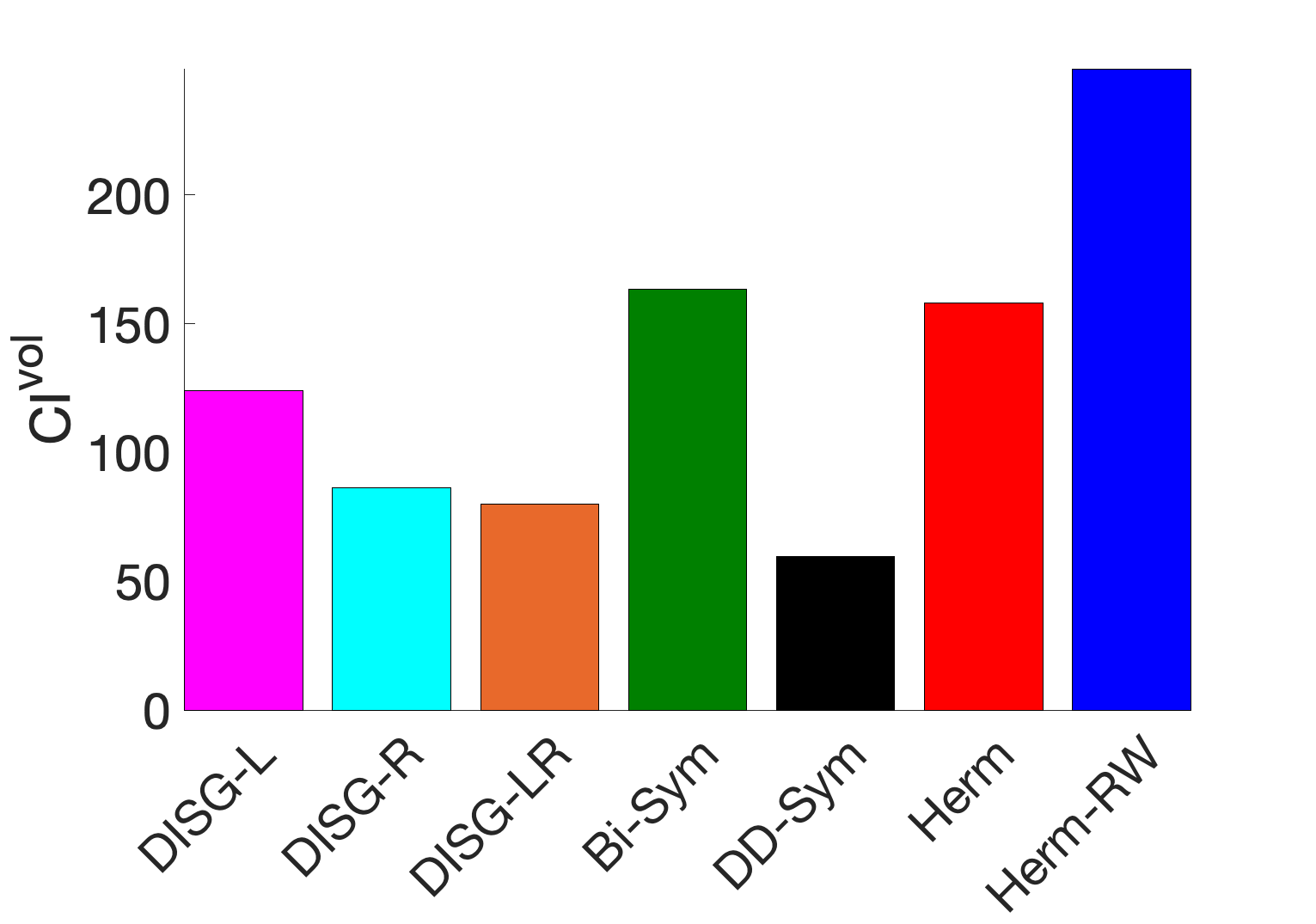}
& \includegraphics[width=\wid]{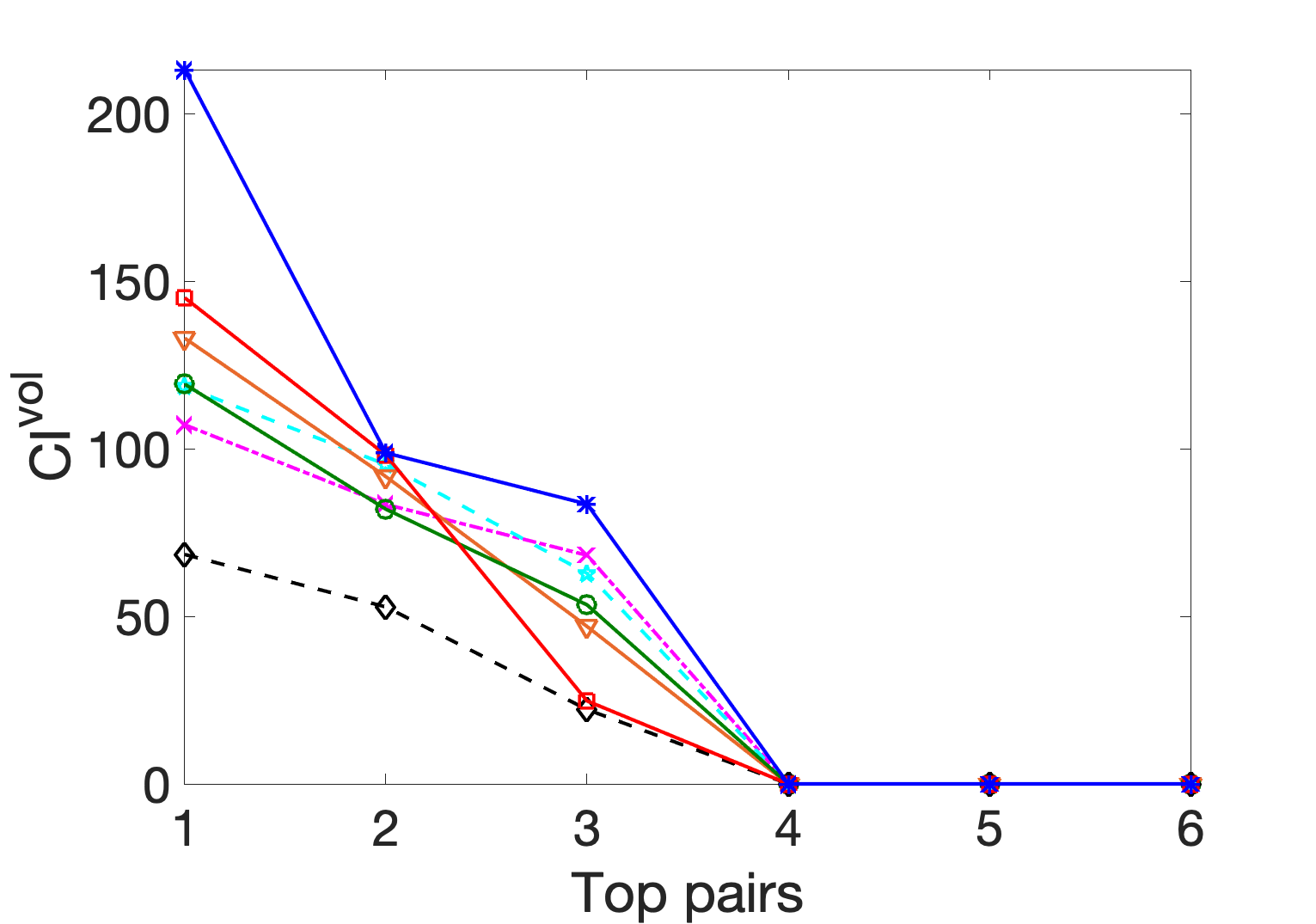}
& \includegraphics[width=\wid]{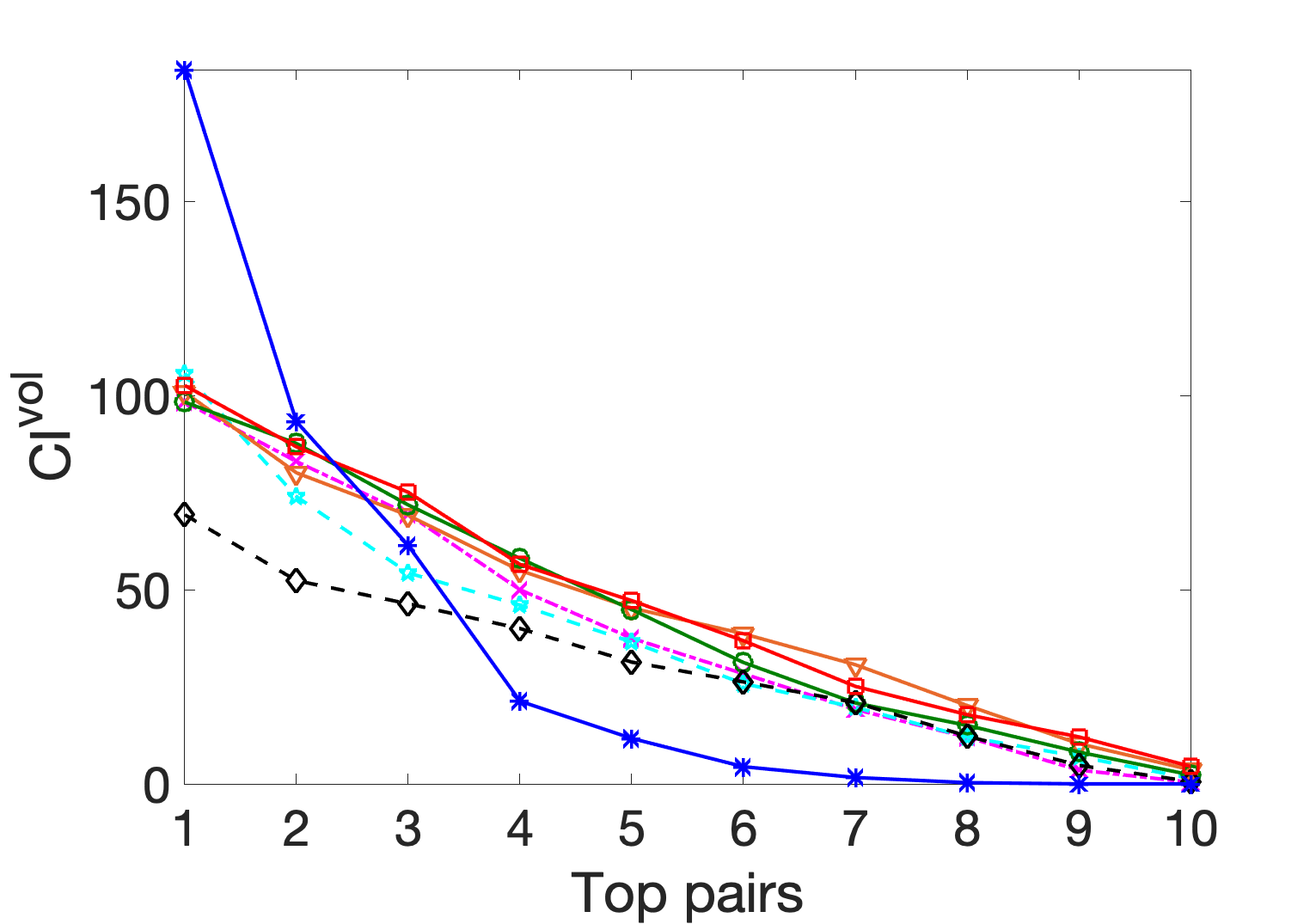}
& \includegraphics[width=\wid]{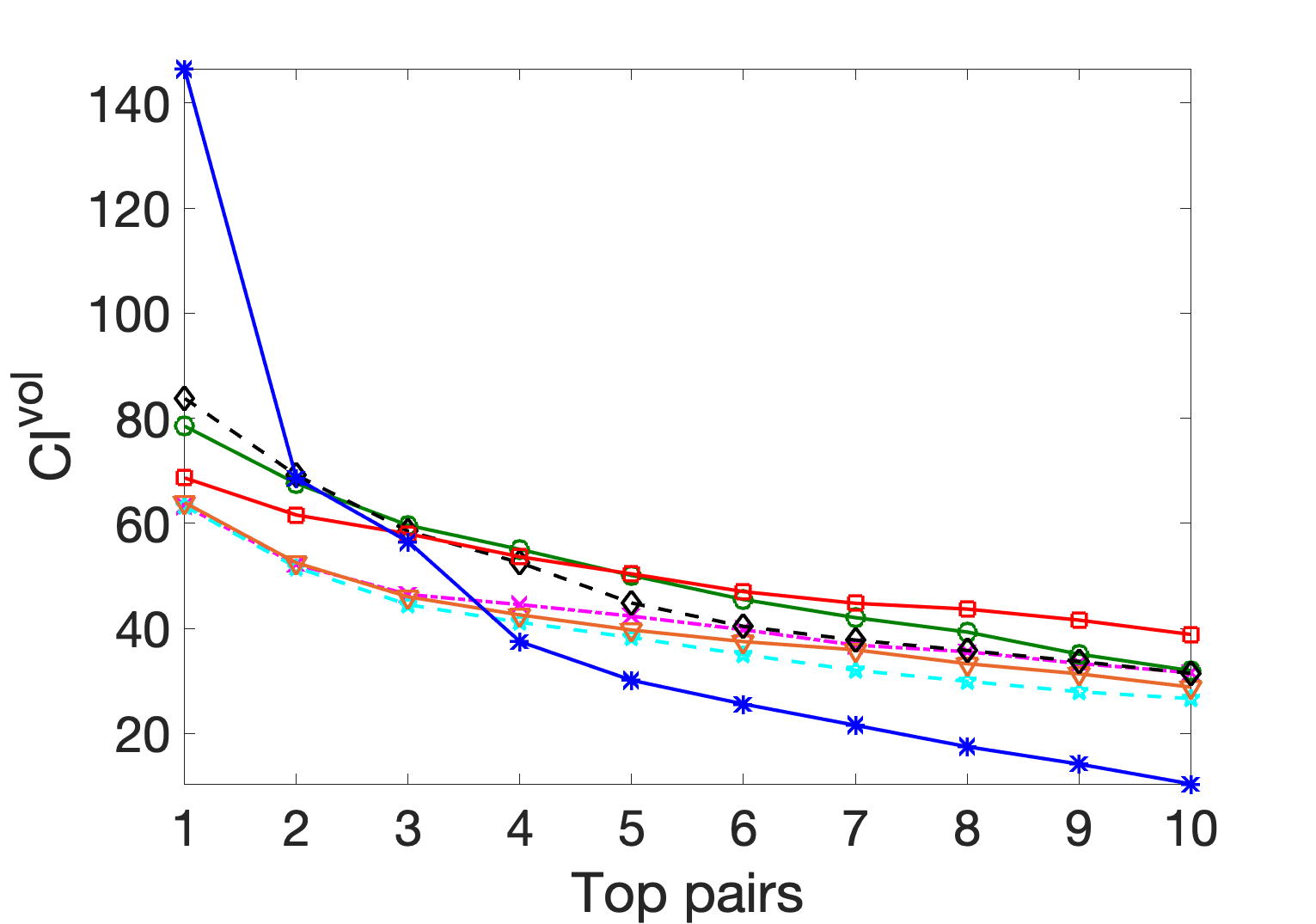} \\ 
\end{tabular}
% \captionsetup{width=0.8\linewidth}
\vspace{-2mm}
\captionof{figure}{The $ \cip $ and   $\civ$ values attained by the top pairs, for the \textsc{c-Elegans} data set with $N=354$ and $k=\{2,3,5,10\}$ clusters (averaged over 20 runs).}
\label{fig:scanID_11_cele}
\end{figure*}

\iffalse
% 11 cele: 2 3 5 7 10 15
\newcommand{\wid}{1.25in}
\newcolumntype{C}{>{\centering\arraybackslash}m{\wid}}
\begin{table*}\sffamily
\hspace{-11mm}
\begin{tabular}{l*2{C}@{}}
 & $k=10$ & $k=15$ \\ 
& \includegraphics[width=0.2470\columnwidth]{Figures/topk_pairs/cele_k10_Avg20_TopCI.png}
& \includegraphics[width=0.2470\columnwidth]{Figures/topk_pairs/cele_k15_Avg20_TopCI_NoLeg.png} \\ 
% 
& \includegraphics[width=0.2470\columnwidth]{Figures/topk_pairs/cele_k10_Avg20_TopCIvol.png}
& \includegraphics[width=0.2470\columnwidth]{Figures/topk_pairs/cele_k15_Avg20_TopCIvol_NoLeg.png} \\ 
\end{tabular}
% \captionsetup{width=0.8\linewidth}
\vspace{-2mm}
\caption{\small The $ \cip $ and  $\civtop$ objective function values for the \textsc{CELE} data set with $N=354$  (averaged over 20 runs).}
\label{fig:scanID_11_cele}
\end{table*} 
\fi

% \iffalse 
% \clearpage 
% \bigskip

\renewcommand{\wid}{1.5in}
%\newcolumntype{C}{>{\centering\arraybackslash}m{\wid}}
\begin{figure}[t]\sffamily
\hspace{-3mm}
 % LOUT_END:     \begin{tabular}{l*5{C}@{}}
   % LOUT_END:   & I & II & III & IV & V \\ 
% \begin{tabular}{l*4{C}@{}}
\begin{tabular}{l*3{C}@{}}
   & I & II & III  \\  
% 
% \iffalse     % LOUT_END: 
\small{\textsc{NAIVE}}
& \includegraphics[width=\wid]{{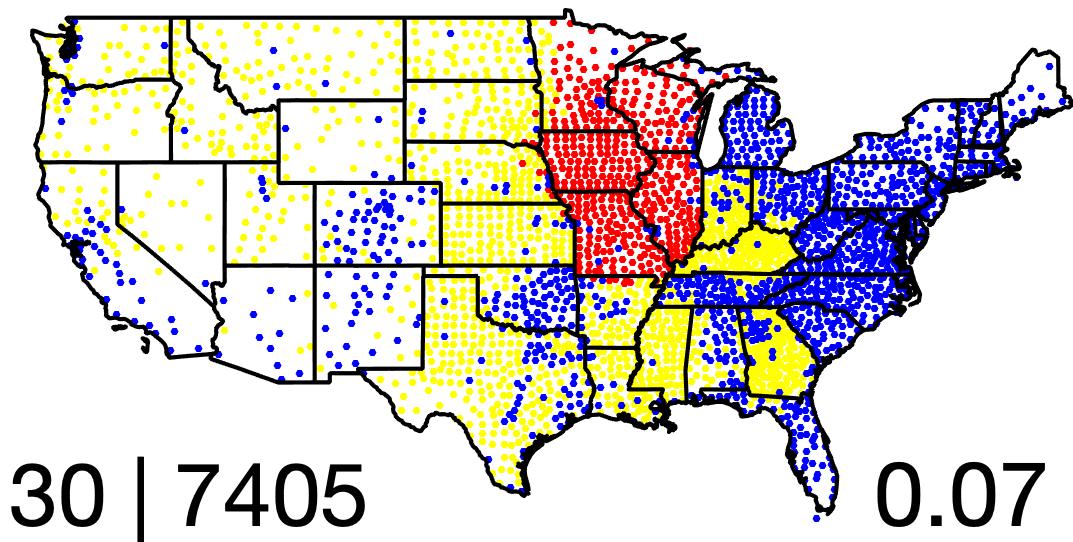}} 
& \includegraphics[width=\wid]{{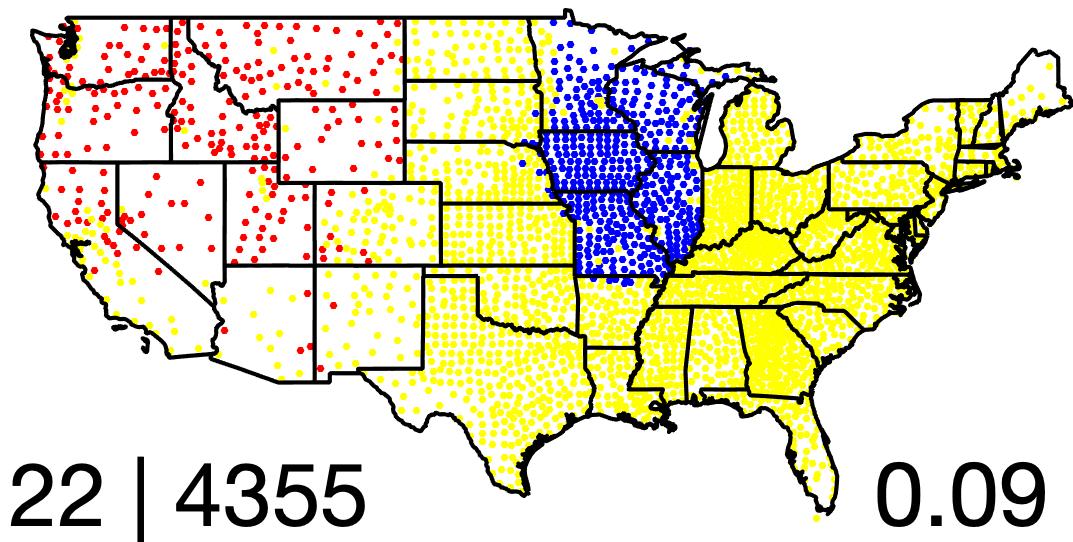}} 
& \includegraphics[width=\wid]{{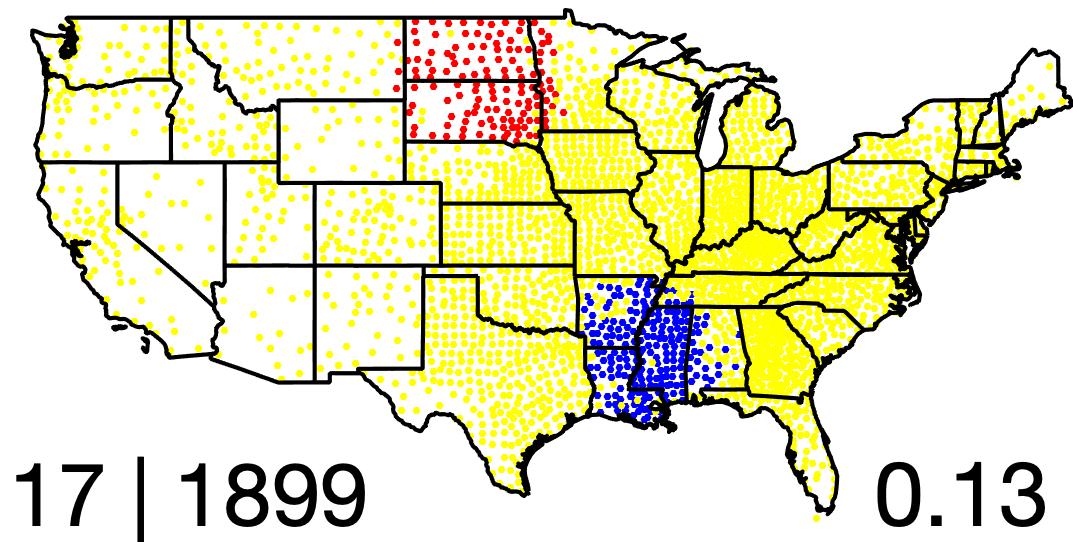}} \\
% & \includegraphics[width=\wid]{{Figures/mig1Top5/mig1_k10_Naive__TopIFMpair4.jpg}} 
% & \includegraphics[width=\wid]{{Figures/mig1Top5/mig1_k10_Naive__TopIFMpair5.jpg}} \\ 
% \fi 
%
% \iffalse 
\small{\textsc{DISGL}}
& \includegraphics[width=\wid]{{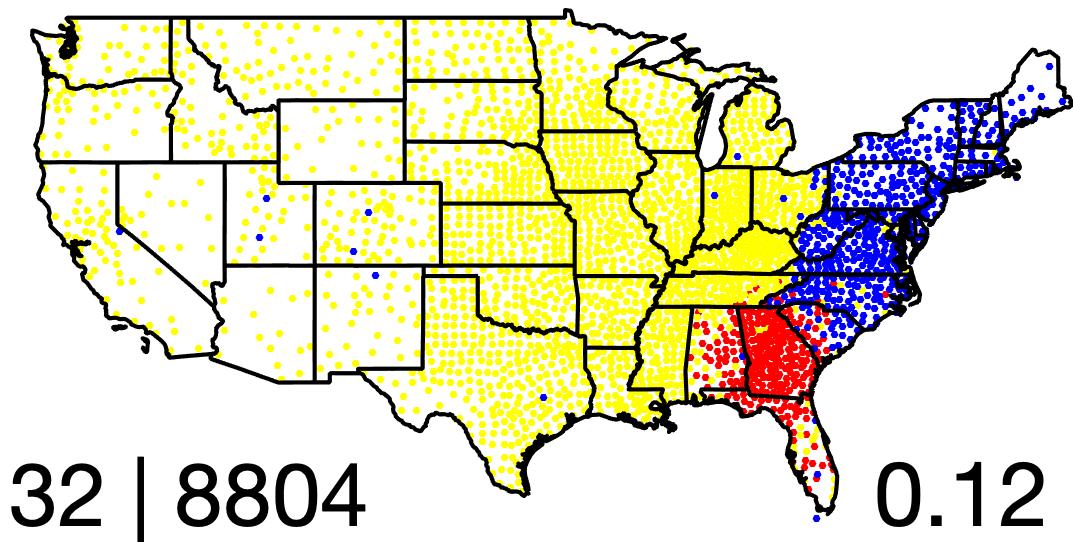}} 
& \includegraphics[width=\wid]{{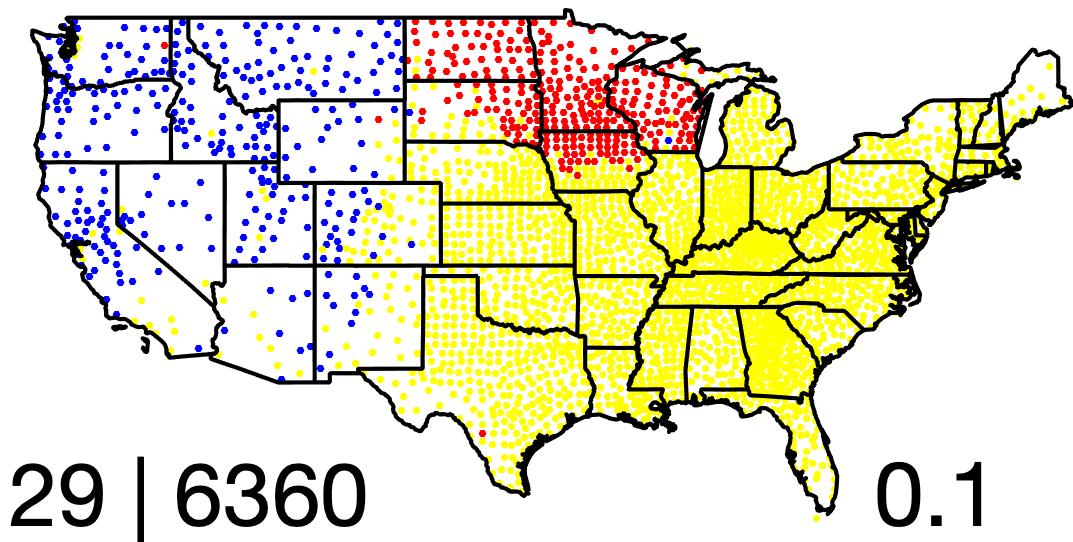}} 
& \includegraphics[width=\wid]{{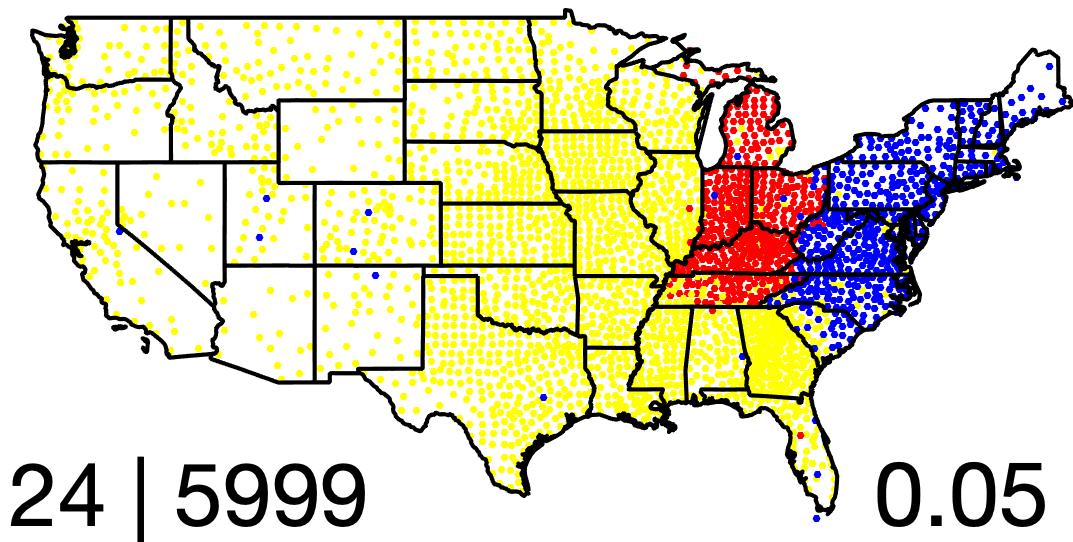}}  \\
% & \includegraphics[width=\wid]{{Figures/mig1Top5/mig1_k10_DISGL__TopIFMpair4.jpg}} 
% & \includegraphics[width=\wid]{{Figures/mig1Top5/mig1_k10_DISGL__TopIFMpair5.jpg}} \\ 
%
\small{\textsc{DISGR}}
& \includegraphics[width=\wid]{{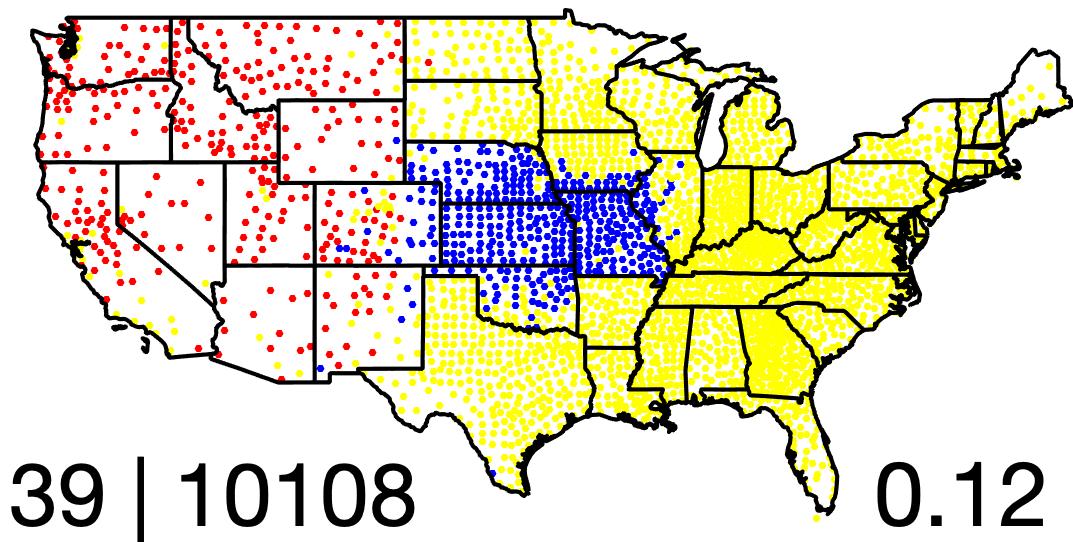}} 
& \includegraphics[width=\wid]{{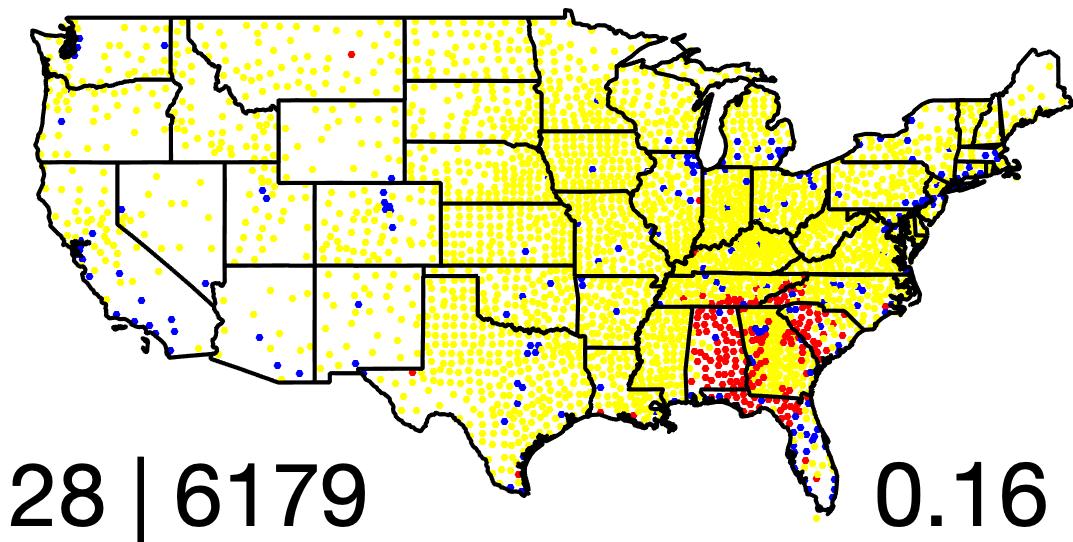}} 
& \includegraphics[width=\wid]{{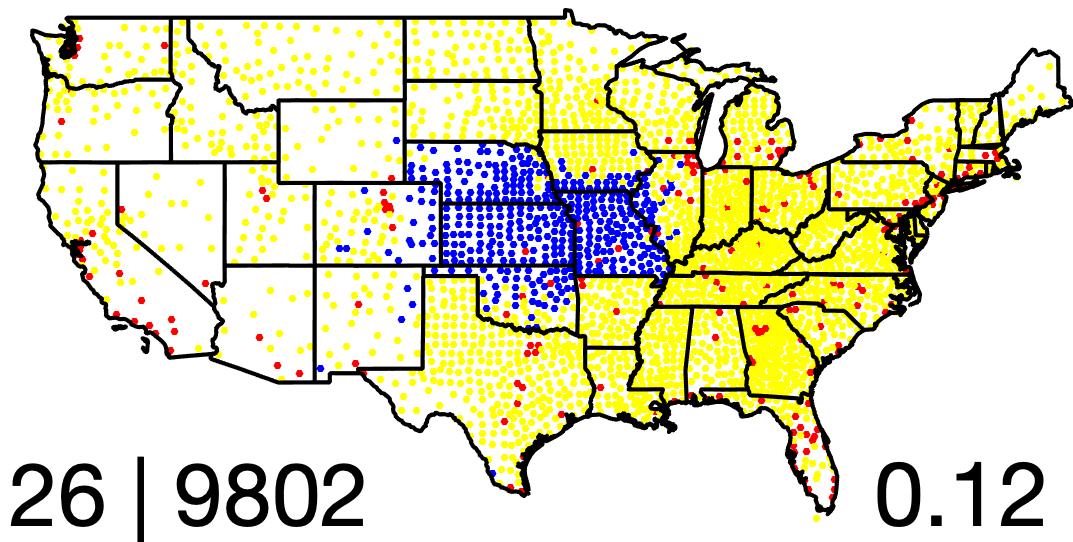}} \\
% & \includegraphics[width=\wid]{{Figures/mig1Top5/mig1_k10_DISGR__TopIFMpair4.jpg}} 
% & \includegraphics[width=\wid]{{Figures/mig1Top5/mig1_k10_DISGR__TopIFMpair5.jpg}} \\ 
% 
% \fi 
% 
\small{\textsc{DISGLR}}
& \includegraphics[width=\wid]{{Figures/mig1Top5/mig1_k10_DISGLR__TopIFMpair1.jpg}} 
& \includegraphics[width=\wid]{{Figures/mig1Top5/mig1_k10_DISGLR__TopIFMpair2.jpg}} 
& \includegraphics[width=\wid]{{Figures/mig1Top5/mig1_k10_DISGLR__TopIFMpair3.jpg}} \\
% & \includegraphics[width=\wid]{{Figures/mig1Top5/mig1_k10_DISGLR__TopIFMpair4.jpg}} 
% 
% LOUT_END: & \includegraphics[width=\wid]{{Figures/mig1Top5/mig1_k10_DISGLR__TopIFMpair5.jpg}} \\ 
%
% \iffalse     % LOUT_END: 
\small{\textsc{Bi-Sym}}
& \includegraphics[width=\wid]{{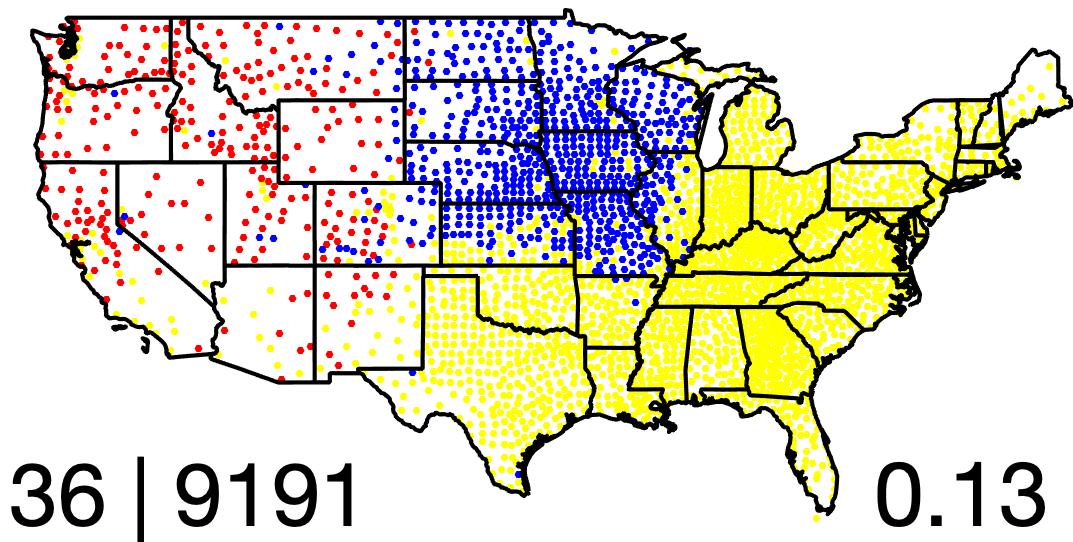}} 
& \includegraphics[width=\wid]{{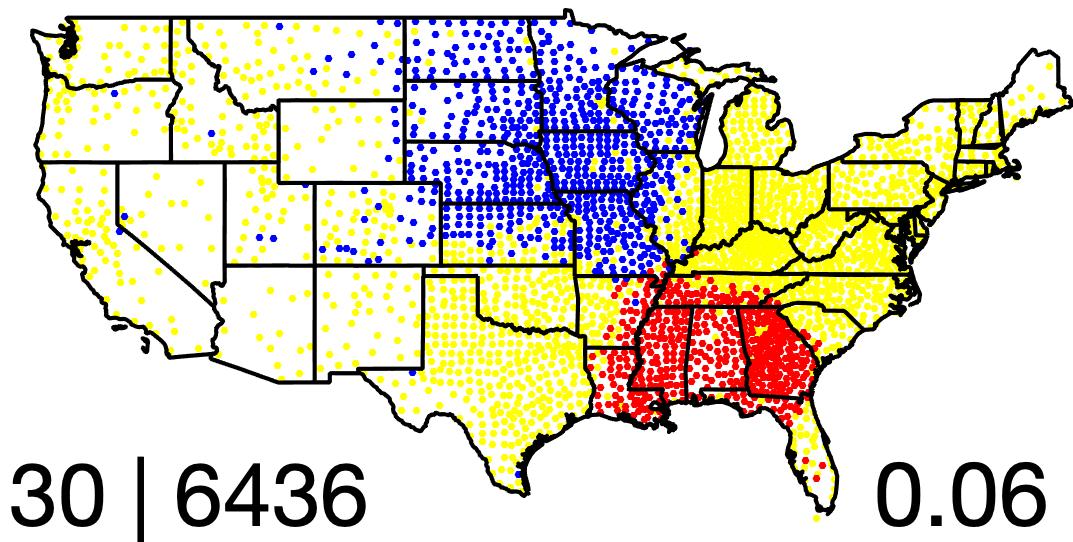}} 
& \includegraphics[width=\wid]{{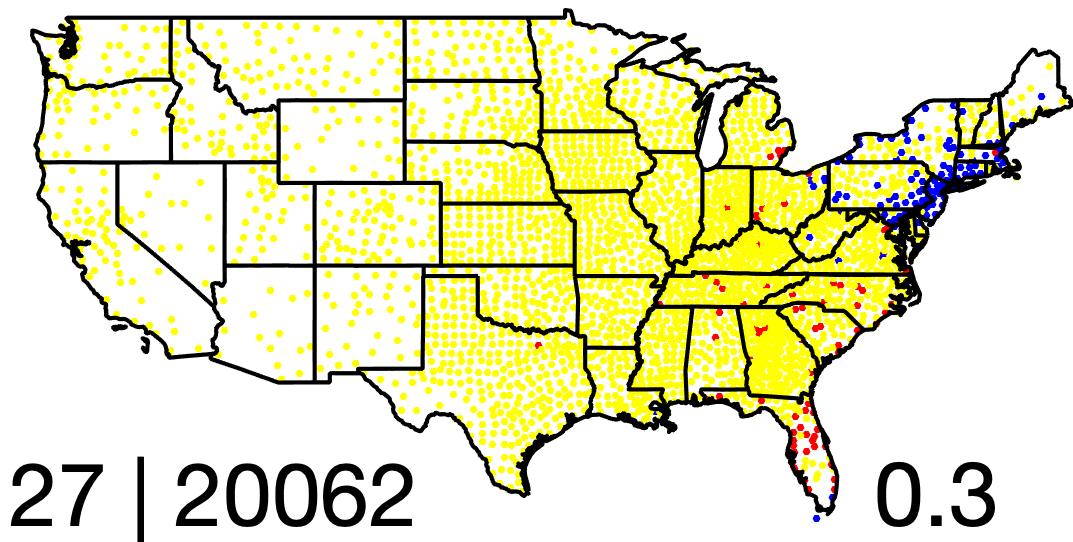}} \\
% & \includegraphics[width=\wid]{{Figures/mig1Top5/mig1_k10_BiSym__TopIFMpair4.jpg}} 
% & \includegraphics[width=\wid]{{Figures/mig1Top5/mig1_k10_BiSym__TopIFMpair5.jpg}} \\ 
% \fi 
%
\small{\textsc{DD-Sym}}
& \includegraphics[width=\wid]{{Figures/mig1Top5/mig1_k10_DDSym__TopIFMpair1.jpg}} 
& \includegraphics[width=\wid]{{Figures/mig1Top5/mig1_k10_DDSym__TopIFMpair2.jpg}} 
& \includegraphics[width=\wid]{{Figures/mig1Top5/mig1_k10_DDSym__TopIFMpair3.jpg}} 
% & \includegraphics[width=\wid]{{Figures/mig1Top5/mig1_k10_DDSym__TopIFMpair4.jpg}}  
\\
 % LOUT_END: & \includegraphics[width=\wid]{{Figures/mig1Top5/mig1_k10_DDSym__TopIFMpair5.jpg}} \\ 
% 
\small{\textsc{Herm}}
& \includegraphics[width=\wid]{{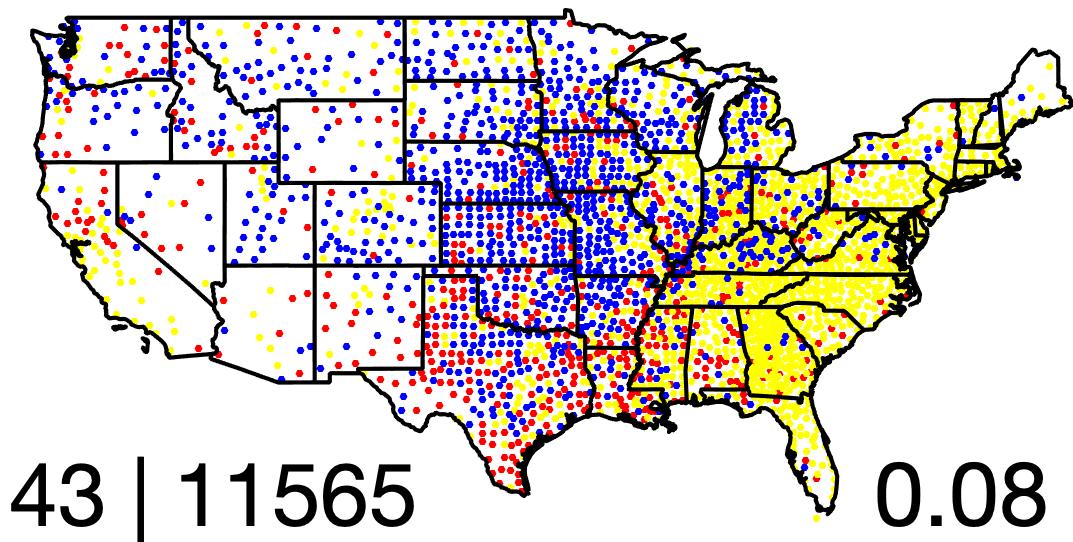}} 
& \includegraphics[width=\wid]{{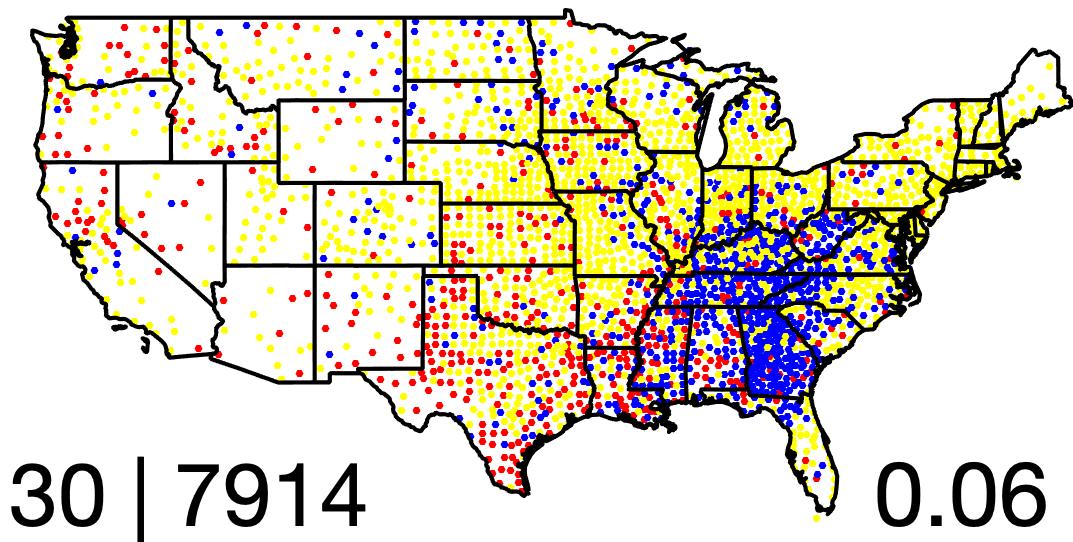}} 
& \includegraphics[width=\wid]{{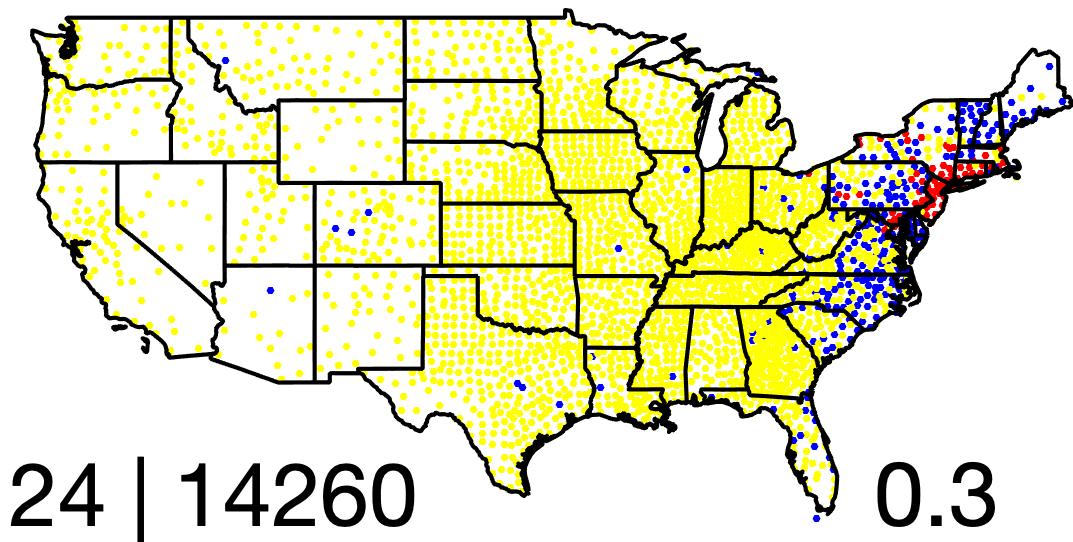}} 
% & \includegraphics[width=\wid]{{Figures/mig1Top5/mig1_k10_Herm__TopIFMpair4.jpg}}
\\
 % LOUT_END: & \includegraphics[width=\wid]{{Figures/mig1Top5/mig1_k10_Herm__TopIFMpair5.jpg}} \\ 
%
\small{\textsc{Herm-RW}}
& \includegraphics[width=\wid]{{Figures/mig1Top5/mig1_k10_HermRW__TopIFMpair1.jpg}} 
& \includegraphics[width=\wid]{{Figures/mig1Top5/mig1_k10_HermRW__TopIFMpair2.jpg}} 
& \includegraphics[width=\wid]{{Figures/mig1Top5/mig1_k10_HermRW__TopIFMpair3.jpg}} 
% & \includegraphics[width=\wid]{{Figures/mig1Top5/mig1_k10_HermRW__TopIFMpair4.jpg}}
\\
 % LOUT_END: & \includegraphics[width=\wid]{{Figures/mig1Top5/mig1_k10_HermRW__TopIFMpair5.jpg}} \\
% 
% \iffalse     % LOUT_END: 
\small{\textsc{Herm-Sym}}
& \includegraphics[width=\wid]{{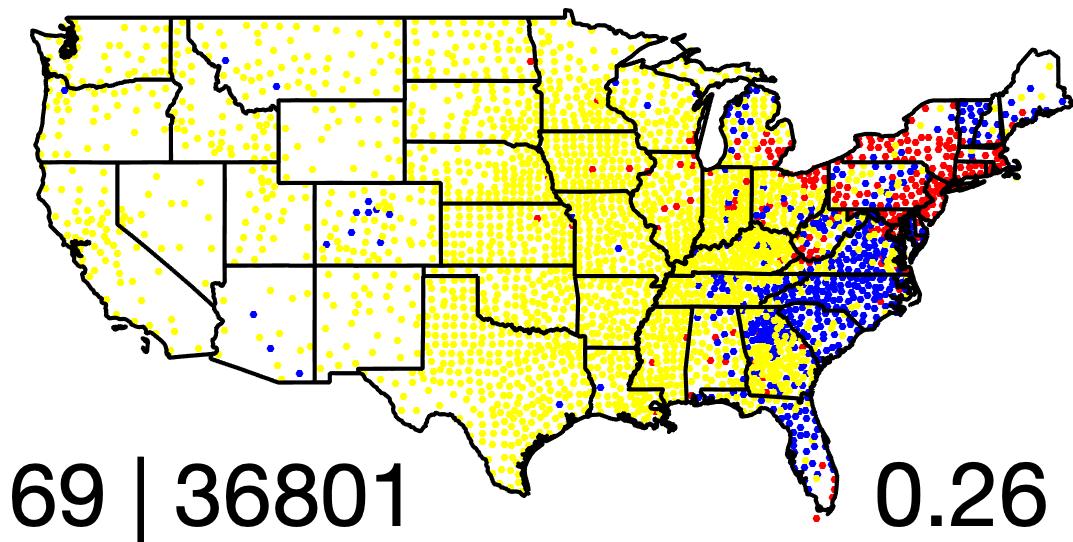}} 
& \includegraphics[width=\wid]{{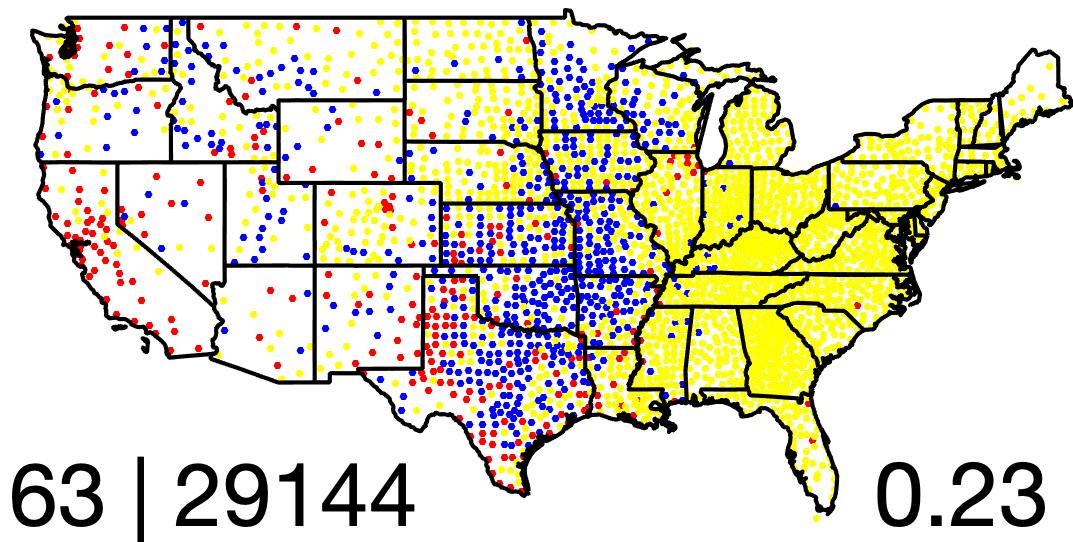}} 
& \includegraphics[width=\wid]{{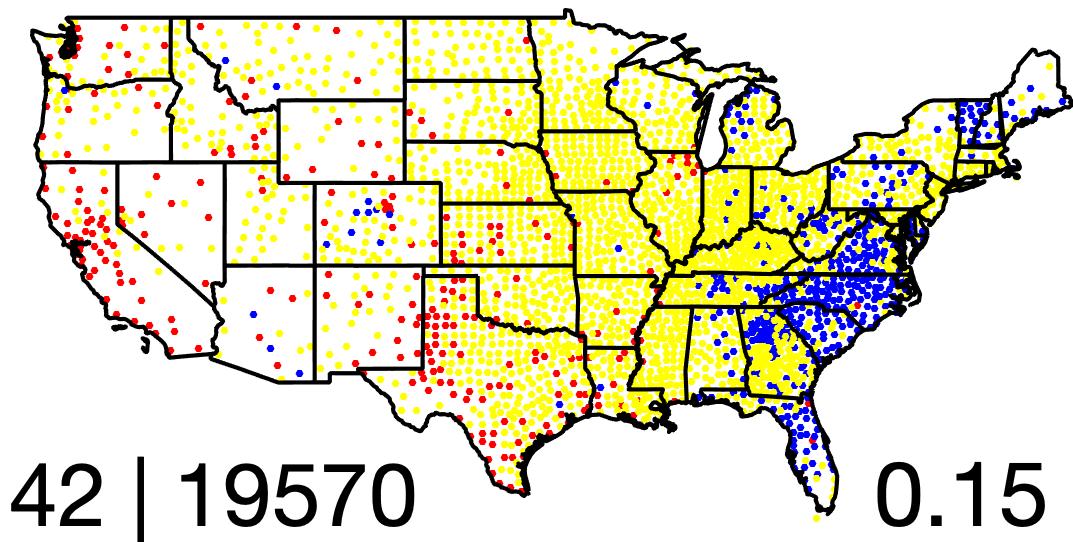}} \\
% & \includegraphics[width=\wid]{{Figures/mig1Top5/mig1_k10_HermSym__TopIFMpair4.jpg}} 
% & \includegraphics[width=\wid]{{Figures/mig1Top5/mig1_k10_HermSym__TopIFMpair5.jpg}} \\
% \fi 
%
\end{tabular}
% \captionsetup{width=0.99\linewidth}
% \vspace{-1mm}
\captionof{figure}{The top three largest size-normalised cut imbalance pairs for the \textsc{US-Migration} data with $k=10$ clusters, for all the methods considered. Red denotes the source cluster, and blue denotes the destination cluster. For each plot, the bottom left text contains the numerical values (rounded to nearest integer) of the two normalised  $\cis$ and $\civ$ pairwise cut imbalance values, and the bottom right text contains the $\cip$ cut imbalance ratio in $ [0,0.5] $.
% \noteMC{todo: drop color bar and insert the CI values [CI,CIM,CIV] in the figure (title or large text)}.
}
\label{fig:mig1Top5_ALL_Methods}	
\end{figure}

\end{document}